\documentclass[11pt]{article}

\usepackage{enumerate, fullpage}
\usepackage{graphicx}
\usepackage{multirow}
\usepackage{color}
\usepackage{amsmath}
\usepackage{amsfonts}
\usepackage{amssymb}
\usepackage{mathrsfs}
\usepackage{mathptmx}
\usepackage{eucal,bm}
\usepackage{algorithm}
\usepackage{smile}
\usepackage{algorithmic}

\def\cA{{\cal A}}
\def\cS{{\cal S}}

\newcommand{\1}{\mathbf 1}

\newcommand{\citep}{\cite}

\newcommand{\lv}{\left \langle}
\newcommand{\rv}{\right \rangle}

\usepackage{tcolorbox}

\usepackage{tocloft}
\usepackage[colorlinks=true,linkcolor=red,urlcolor=blue,citecolor=blue]{hyperref}

\newcommand{\QED}{\ \hfill\rule[-2pt]{6pt}{12pt} \medskip}

\numberwithin{equation}{section}

\def\cO{{\cal O}}

\newcommand{\E}{\mathbb E}

\newcommand{\beq}{\begin{equation}}
\newcommand{\eeq}{\end{equation}}

\newcommand{\beqnr}{\begin{eqnarray}}
\newcommand{\eeqnr}{\end{eqnarray}}

\newcommand{\benum}{\begin{enumerate}}
\newcommand{\eenum}{\end{enumerate}}

\newtheorem{DE}{Definition}[section]
\newtheorem{AS}[DE]{Assumption}

\advance \topmargin by -\headheight
\advance \topmargin by -\headsep

\evensidemargin \oddsidemargin

\begin{document}

\title{\huge  {Federated Multi-Level Optimization \\
over Decentralized Networks}}
\date{}
\author{Shuoguang Yang\thanks{email: sy2614@columbia.edu {\tt  }}
~~~~Xuezhou Zhang\thanks{Princeton University, xz7392@princeton.edu {\tt  }}
~~~~Mengdi Wang\thanks{Princeton University; mengdiw@princeton.edu{\tt  }}
}

\maketitle

\vskip 0.5cm

\begin{abstract}
Multi-level optimization has gained increasing attention in recent years, as it provides a powerful framework for solving complex optimization problems that arise in many fields, such as meta-learning, multi-player games, reinforcement learning, and nested composition optimization.
In this paper, we study the problem of distributed multi-level optimization over a network, where agents can only communicate with their immediate neighbors. This setting is motivated by the need for distributed optimization in large-scale systems, where centralized optimization may not be practical or feasible. 
To address this problem, we propose a novel gossip-based distributed multi-level optimization algorithm that enables networked agents to solve optimization problems at different levels in a single timescale and share information through network propagation. Our algorithm achieves optimal sample complexity, scaling linearly with the network size, and demonstrates state-of-the-art performance on various applications, including hyper-parameter tuning, decentralized reinforcement learning, and risk-averse optimization.
    \end{abstract}

\section{Introduction}\label{sec:intro}
In recent years, stochastic multi-level optimization (SMO) has attracted increasing attention from the machine learning community. It aims at solving the following
\begin{equation}
    \min_{x \in \RR^{d_x}}  \EE_{\zeta}[ f(x,y_M^\star(x),\zeta) ] , \text{ s.t. }   y_j^\star(x)  = \argmin_{y_j \in \RR^{d_j}} \EE_{\xi_j}[ g_j(y_{j-1}^\star(x), y_j , \xi_j) ] ,  \ \ j = 1,2,\cdots, M,
\end{equation}
where $y_0^\star(x) = x \in \RR^{d_x}$ and $y_j^\star(x) $ represents the unique optimal solution to the stochastic optimization problem at level $j$. 
Given any $x \in \RR^{d_x}$, the best response $y_M^\star(x)$ can be computed by recursively solving 
$y_1^\star(x),\cdots, y_{M-1}^\star(x)$.
It has been found to provide favorable solutions to a variety of problems, such as meta learning and hyperparameter optimization \citep{franceschi2018bilevel,snell2017prototypical, bertinetto2018meta}, multilevel composition optimization \citep{wang2016stochastic,yang2019multilevel}, multi-player games \citep{von1952theory}, reinforcement learning and imitation learning \citep{arora2020provable, hong2020two}. 
Despite the importance of SMO, it has not been systematically studied in both the theoretical and numerical perspectives. Existing work mainly focus on one of its special case, stochastic bilevel optimization (SBO), where $M=1$  and $y_M^\star(x)$ can be computed by solving a vanilla stochastic optimization problem
$$
y_M^\star(x) = \argmin_{y } \EE_{\xi} [g(x,y,\xi)],
$$
rather than recursively solving a sequence of stochastic optimization problem. It remains an open problem to design efficient algorithms for solving SMO with strong theoretical guarantees.

In addition, even for SBO, the majority of the above work focuses on the classic centralized setting. However, such problems often comes from distributed/federated applications, where agents are unwilling to share data but rather perform local updates and communicate with neighbors. Theories and algorithms for distributed stochastic bilevel optimization are less developed.

This work aims to answer the following two questions: 
\begin{center}
\emph{
    (i) How to generalize stochastic bilevel optimization to stochastic multilevel optimization?\\
    (ii) How to design efficient algorithms for stochastic multilevel optimization over distributed network with a generic topology? 
}
\end{center}
We consider the \emph{decentralized} learning setting where the data are distributed over $K$ agents $\cK= \{ 1,2,\cdots, K\}$ over a communication network of general topology. The network may not necessarily contain a central server that connects all other agents as in a \textit{star network}, but may preserve a general connected structure illustrated in Figure~\ref{fig:network}, where each agent can only communicate with its neighbors. One example is federated learning which is often concerned with a single-server-multi-user system, where agents communicate with a central server to solve a task cooperatively    \citep{lan2018random,ge2018minimax}. Another example is the sensor network, where sensors are fully decentralized and can only communicate with nearby neighbors \citep{distributed_sensor}.

\begin{figure}[t]
\begin{minipage}{0.5\textwidth}
    \centering
        \includegraphics[width=0.9\linewidth]{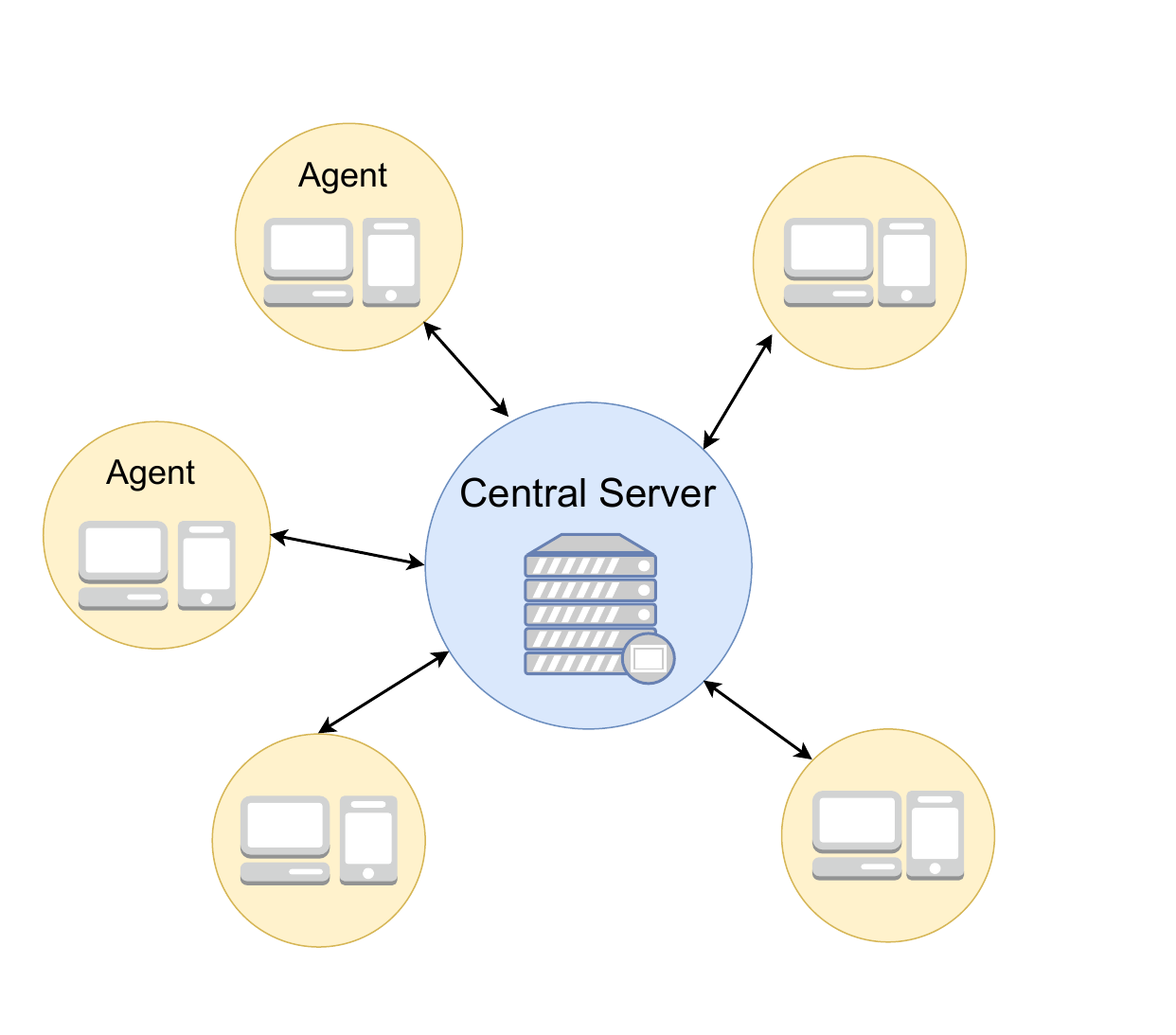}
        (a)~Star Network
        \end{minipage}  
        \begin{minipage}{0.5\textwidth}
        \centering
   \includegraphics[width=0.9\linewidth]{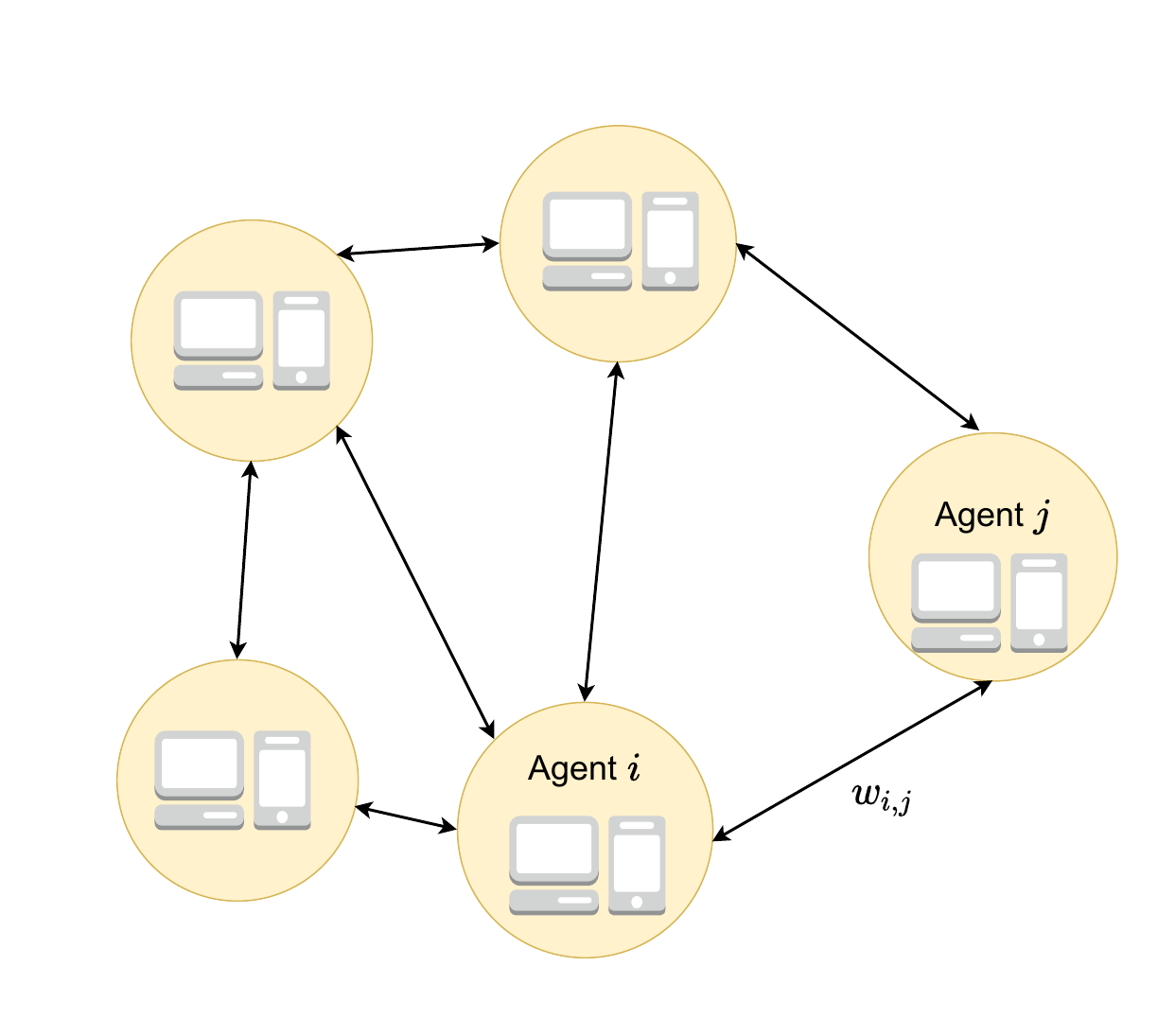} 
   (b)~Decentralized Network
   \end{minipage}
    \caption{Illustrations of Distributed Network Structures. In a star network, the agents may  cooperate to solve a specific task by communicating with the central server, commonly seen in federated learning. In a decentralized network, agents only trust and communicate with neighbors in the network.
    }
    \label{fig:network}
    \vskip -3mm
\end{figure}

We consider the following decentralized stochastic multi-level optimization (DSMO)
\begin{equation}\label{prob:1} 
\begin{split}
\min_{ x \in \RR^{d_x}}F(x) & = \left \{ \frac{1}{K} \sum_{k=1}^K f^k(x, y_M^\star(x)) \right \} , 
\\
\mbox{ s.t. }  
  y_j^\star(x) & = \argmin_{y_j \in \RR^{d_j}} \left  \{  \frac{1}{K} \sum_{k=1}^K g_j^k(y_{j-1}^\star(x), y_j)  \right \} ,  \ \ j = 1,2,\cdots, M.
\end{split}
\end{equation}
with $y_0^\star(x) = x \in \RR^{d_x}$, $y_j \in \RR^{d_j}$, $f^k(x,y_M)  = \EE_{\zeta^k}[f^k(x,y_M, \zeta^k)]$ is the objective for agent $k$, and $g_j^k(y_{j-1}, y_j)  = \EE_{\xi_j^k}[g_j^k(y_{j-1}, y_j,\xi_j^k)]$ represents the $j$-th level decision function for agent $k$. Here,  expectations $\EE_{\zeta^k}[\cdot]$  and $\EE_{\xi_j^k}[\cdot]$ are taken with respect to the random variables $\zeta^k$ and $\xi_j^k$, and each agent has heterogeneous objective and inner-level decision functions so that $f^k$ and $\{ g_j^k\}_{j=1}^M$ may vary across agents. 
We consider the scenario where each $g_j^k(y_{j-1},y_j)$ is strongly convex in $y_j$. We use the notation $F^* = \min_{x \in \RR^{d_x}} F(x)$, $f(x,y) =\frac{1}{K} \sum_{k= 1}^K f^k(x,y)$, and $g_j(y_{j-1},y_j) = \frac{1}{K} \sum_{k= 1}^K  g_j^k(y_{j-1},y_j)$ for  convenience.

\subsection{Example applications of SBO and SMO}
When $M=1$, SMO reduces to SBO. SBO
was first employed to formulate the resource allocation problem~\citep{bracken1973mathematical} and has since found applications in many classic operations research settings \citep{cramer1994problem, sobieszczanski1997multidisciplinary, livne1999integrated, tu2020two, tu2021two}, and more recently in machine learning problems \citep{franceschi2018bilevel,snell2017prototypical, bertinetto2018meta,wang2016stochastic}. In particular, we introduce two applications that have recently attracted a lot of recent attention, namely hyperparameter optimization and compositional optimization.

\paragraph{Hyperparameter Optimization}
The problem of hyper-parameter tuning \citep{okuno2021lp} often takes the following form:
 \begin{equation}\label{prob:hyperopt}
  \min_{x \in \RR^d} \left \{   \sum_{i \in \cD_{\text{val}}}  \ell_i (y^\star(x)) \right \}, \mbox{ s.t. } y^\star(x) \in \argmin_{y \in \RR^{d_y}} \left \{  \sum_{j \in  \cD_{\text{train}}} \ell_j(y) + \cR(x,y) \right \},
 \end{equation}
 where $\cD_{\text{train}}$ and $\cD_{\text{val}}$ are two datasets used for training and validation, respectively, $y \in \RR^{d_y}$  is a vector of unknown parameters to optimize, $\ell_i(y)$ is a convex loss over data $i$, and $x \in \RR^d$ is a vector of hyper-parameters for a strongly convex regularizer $\cR(x,y)$. For any hyper-parameter $x \in \RR^d$, the inner-level problem solves for the best parameter $y^\star(x)$ over the training set $\cD_{\text{train}}$ under the regularized training loss $\ell_j(y)+\cR(x,y)$.
 The goal is to find the hyper-parameter $x^* \in \RR^d$ whose corresponding best response $y^\star(x)$ yields the least loss over the validation set $\cD_{\text{val}}$. 
  In practice,  continuous hyperparameters are often tuned  by the grid search which is exponentially expensive. An efficient SBO algorithm should find the optimal parameters in time increasing polynomially with dimension, rather than exponentially. When the training and validation set are distributed across nodes, the problem becomes a distributed SBO.

\noindent 
\textbf{Distributed Risk-averse Optimization.}
Letting $U^k(x,\xi^k )$ be a random utility function for agent $k \in \cK$, we denote by $U(x) =  \frac{1}{K}\sum_{k \in \cK} \EE_{\xi^k} [U^k(x,\xi^k )] $ the expected utility function averaged over all agents, and consider the following distributed regularized mean-deviation risk-averse optimization problem 
\begin{equation}\label{prob:dist_riskaverse}
\max_{x} \left \{  U(x ) - \kappa  \Big [ \EE_{\xi^1,\cdots, \xi^K}   \big [   U(x) - \frac{1}{K} \sum_{k \in \cK}U^k(x,\xi^k )  \big ]_+^p \Big ]^{1/p} - \frac{\lambda \| x\|^2}{2} \right \}.
\end{equation}
Here, all agents are connected by a decentralized network and cooperate to solve a shared risk-averse mean-deviation optimization problem \citep{ruszczynski2006optimization}. 
This problem can be reformulated as a SMO  optimization problem with $M=2$ such that 
\begin{equation*}
\begin{split}
    \max_x f\big (x,y_1^\star(x), y_2^\star(x) \big ) & =  y_1^\star(x) - \kappa \big ( y_2^\star(x) \big )^{1/p} - \frac{\lambda \| x\|^2 }{2} , 
    \\
    \mbox{s.t. } \  y_1^\star(x)  & = \argmin_{y_1 \in \RR }  \EE_{\xi^1,\cdots, \xi^K} \Big  (y_1 - \frac{1}{K} \sum_{k \in \cK} U^k(x,\xi^k  )  \Big )^2 , 
    \\
    y_2^\star(x) & = \argmin_{y_2  \in \RR } \EE_{\xi^1,\cdots, \xi^K}  \Big (  y_2 - \big ( y_1^\star(x) - \frac{1}{K} \sum_{k \in \cK} U^k (x,\xi^k) \big  )_+^p  \Big )^2,
    \end{split}
\end{equation*}
where $p > 1$ is  a constant. 
The above problem is $\lambda$-strongly concave for any $\kappa \in (0,1]$ and $\lambda >0$~\citep{ruszczynski2006optimization}. 
\\
\noindent 
\textbf{Policy Optimization in Finite-Horizon MDPs.} We consider the policy optimization problem for collaborative multi-agent MDP over a finite horizon. Letting $\cS$ be the state space, $\cA$ be the action space, and $H$ be the horizon length, agent $k$ obtains an instantaneous reward  $r_h^k(s,a)$ if it 
takes action $a \in \cA$ at state $s \in \cS$ and time $h \in[H]$.
For any policy $\pi \in \Pi$, state $a\in \cA$, and any time $h \in [H]$, 
we denote by $V_h^\pi(s)$ the \textit{joint value function}, which can be recursively expressed as $V_h^\pi(s) = \frac{1}{K} \sum_{k=1}^K \EE_{s'\sim P(\cdot|s, a),a\sim\pi_h(\cdot|s)} \left[r_h^{k}(s,a) + V_{h+1}^\pi(s')\right]$. In other words, $V_h^\pi(s)$ measures the sum of expected cumulative rewards each agent receives. The goal is to find a policy $\pi$ which maximizes the joint value function:
\begin{equation}
\min_{\pi \in \Pi} \EE_{s_1\sim\rho} V_1^\pi(s_1). 
\end{equation}
Putting into the multi-level SMO formulation, we have
$x = \pi_h(a|s)\in \RR^{H\times |\cS| \times |\cA|}$, and
$$
y_h^\star(x) = V_h^\pi(s) = \argmin_{y_h \in \RR^{|\cS|} } \sum_{s\in \cS} \Big (\frac{1}{K} \sum_{k=1}^K \EE_{s'\sim P(\cdot|s, a),a\sim\pi_h(\cdot|s)} \left[r_h^{k}(s,a) + V_{h+1}^\pi(s')\right]-y_h \Big )^2.
$$

\subsection{Challenges with Distributed SMO}\label{sec:challenges}
\textbf{From SBO to SMO.}
Despite the recent rapid development of single-level and bilevel optimization, a method appropriate to Stochastic Multilevel Optimization (SMO) remains elusive. The major hurdle to solving SMO lies in the absence of explicit knowledge of $y^\star(x)$, so that an unbiased gradient for $\nabla f(x, y_M^\star(x))$ is not available. Specifically, by using the implicit function theorem, we can write the gradient of SMO as 
\begin{equation}\label{eq:biased_grad}
\nabla F(x) = \nabla_1 f(x,y_M^\star(x)) + \nabla y_{M}^\star(x) \nabla_2 f(x, y_M^\star(x))
\end{equation}
whose expression is not readily known due to the unavailability of $\nabla y_M^\star(x)$. In the SBO regime where $M=1$, \cite{couellan2016convergence,ghadimi2018approximation} write out the full gradient of non-distributed SBO as 
\begin{equation} \label{eq:biased_grad_SBO}
    \nabla F(x) = \nabla_1 f(x,y^\star(x))  - \nabla_{12}^2 g(x,y^\star(x)) [\nabla_{22}^2 g(x,y^\star(x)) ]^{-1}\nabla_2 f(x,y^\star(x)),
\end{equation}
where $y^\star(x) = y_M^\star(x)$, providing a connection between SBO and classical stochastic optimization. As $y^\star(x)$ is the optimal solution to a stochastic program, which cannot be explicitly computed within finite steps but has to be approximated, the majority of bias comes from the inaccurate estimation of $y^\star(x)$.
To overcome the issue, for non-distributed SBO, various algorithms have been proposed to obtain sharp estimators for $y^\star(x)$ and reduce the bias of the constructed gradients \citep{chen2021single,hong2020two,ji2021bilevel,yang2021provably}. These techniques result in tight convergence analysis and give rise to algorithms widely used in modern applications.

Unfortunately,  the SMO problem~\eqref{prob:1} is more challenging compared with SBO, because (i) the explicit expression of $\nabla y_M^\star(x)$ is unknown, and (ii) $y_M^\star(x)$ is harder to estimate. To solve SMO, we have to explicit derive $\nabla y_M^\star(x)$, which requires applying implicit function theorem to each level $j=1,\cdots,M$ and carefully combine them. Meanwhile, to compute $y_M^\star(x)$ for each $x$, it can be seen in \eqref{prob:1} that it requires solving a sequence of $M$ stochastic optimization problems to compute $\{ y_1^\star(x), y_2^\star(x),\cdots, y_M^\star(x)\}$ rather than solving a vanilla strongly convex program in SBO. This is significantly more challenging because the errors induced by different layers may interact and compound with each other, which has to be carefully handled.
As a result, no prior SBO algorithm can be applied to the SMO. It is also unclear  what the best achievable convergence rate is for SMO.

\medskip 
\noindent 
\textbf{From non-distributed to the decentralized regime.} In addition to the challenges from multi-level optimization, the distributed decentralized regime also brings unique challenges that have to be carefully handled, even for the SBO scenario where $M=1$. Let us use decentralized SBO for illustration and consider its full gradient \eqref{eq:biased_grad_SBO}. Computing a sharp estimator of \eqref{eq:biased_grad_SBO} suffers from the following obstacles:
\begin{enumerate}
    \item Even for SBO, $y^\star(x)$ is the shared optimal inner-level solution for $K$ heterogeneous agents, who can only communicate with their neighbors. An estimator of $y^\star(x)$ has to be computed through decentralized algorithms, with carefully designed stochastic approximation, weighting, and communication strategies.
    \item Calculating the outer gradient is highly nontrivial, even when we have an inner solution $y$. Note that 
\begin{eqnarray}
 \frac{1}{K}\sum_{k \in \cK} \Big [ \nabla_{12}^2 g^k(x,y) [ \nabla_{22}^2 g^k(x,y) ]^{-1} \nabla_2 f^k(x,y) \Big ]  \neq \nabla_{12}^2 g(x,y) [\nabla_{22}^2 g(x,y) ]^{-1}\nabla_2 f(x,y).
\end{eqnarray}
In other words, even if the inner problem is solved, the outer gradient requires a new estimation mechanism.
\item Estimating the Hessian inverse $[\nabla_{22}^2 g(x,y)]^{-1}$ is nontrivial in decentralized networks. Note that 
\begin{equation}\label{eq:hessian_inverse}
 \nabla_{22}^2 g(x,y) = \frac{1}{K} \sum_{k=1}^K g^k(x,y) \text{ but }    [ \nabla_{22}^2 g(x,y) ]^{-1}\neq \frac{1}{K} \sum_{k=1}^K [g^k(x,y)]^{-1},
\end{equation}
for non-identical $g^k(x,y)$'s. That is, the shared Hessian inverse is not available even if the explicit value of hessian inverse for each agent is known. Hence, we cannot compute the Hessian inverse $[g^k(x,y)]^{-1}$ for each agent $k$ and then averaging them through gossip communication. Meanwhile, even if $g(x,y)$ is known, it requires $O(d_x^3)$ cost to invert it directly, which is computationally expensive for high dimensional machine learning applications.
A more sophisticated strategy has to be developed to overcome these issues.

\item  In the decentralized network learning, communication among agents can be limited by the network structure and communication protocol, so taking a simple average across agents may require multiple communication rounds.
\end{enumerate}
All of the above challenges in SBO would further exacerbated in the decentralized SMO regime because both estimation and consensus error would interact and compound among different levels, making it particularly challenging to solve.
Because of the above difficulties, it remains unclear how to estimate the outer gradient sharply for a decentralized network.

In an attempt to tackle this problem, this paper studies the convergence theory and sample complexity of gossip-based algorithms. In particular, we ask two theoretical questions: \begin{center}
\emph{
    (i) How does the sample complexity of DSMO scale with the optimality gap and network size?\\
    (ii) How is the efficiency of DSMO affected by the network structure?
}
\end{center}
\noindent 
\textbf{Contributions.} 
In this paper, we develop a gossip-based stochastic approximation scheme where each agent  solves an optimization problem collaboratively by sampling stochastic first- and second-order information using its data and making gossip communications with its neighbors. 
In addition, we develop novel techniques for convergence analysis to characterize the convergence behavior of our algorithm. 
To the best of our knowledge, our work is the first to  formulate DSMO mathematically and propose an algorithm with theoretical convergence guarantees. Specifically, we show that our algorithm enjoys an $\widetilde \cO(\frac{1}{K \epsilon^2})$ sample complexity for finding $\epsilon$-stationary points for nonconvex objectives regardless of the number of levels $M$, where $\widetilde \cO(\cdot)$ hides logarithmic factors, and enjoys an $\widetilde \cO(\frac{1}{K \epsilon})$ sample complexity for Polyak-Łojasiewicz (PL) functions, subsuming strongly convex optimization. 
These results subsume the state-of-the-art results for non-federated  stochastic bilevel optimization~\citep{chen2021tighter} and central-server stochastic bilevel optimization~\citep{tarzanagh2022fednest}, showing that almost no degradation is induced by network consensus. 
Further, the above results suggest that our algorithm exhibits a linear speed-up effect for decentralized settings; that is, the required per-agent sample complexities decrease linearly with the number of agents. It is worthy emphasizing that our convergence rate results are optimal even for SBO, implying its optimality for decentralized SMO.

\section{Related Works}
Bilevel optimization was first formulated by~\cite{bracken1973mathematical} for solving resource allocation   problems. Later, a class of constraint-based algorithms was proposed by \cite{hansen1992new,shi2005extended},  which treats the inner-level optimality condition as constraints to the out-level problem. 
Recently,  \cite{couellan2016convergence} examined the finite-sum case for unconstrained strongly convex lower-lower problems and proposed a gradient-based algorithm that exhibits asymptotic convergence under certain step-sizes. 
For SBO, \cite{ghadimi2018approximation} developed a double-loop algorithm and established the first known complexity results. Subsequently, various methods  have been employed to improve the sample complexity, including two-timescale stochastic approximation~\citep{hong2020two}, acceleration  \citep{chen2021single}, momentum~\citep{khanduri2021near},  and variance reduction
\citep{guo2021randomized,ji2021bilevel,yang2021provably}.

Distributed optimization was developed to  handle real-world large-scale datasets~\citep{dekel2012optimal,feyzmahdavian2016asynchronous} and graph estimation \citep{wang2015distributed}. Centralized and decentralized systems are two important problems that have drawn significant attention. A centralized system considers the network topology where there is a central agent that communicates with the remaining agents \citep{lan2018random}, while in a decentralized system \citep{gao2020periodic,koloskova2019decentralized,lan2020communication,mcmahan2017communication}, each agent can only communicate with its neighbors by using gossip \citep{lian2017can} or  gradient tracking \citep{pu2021distributed} communication strategies, with applications in multi-agent reinforcement learning \citep{xu2020voting}. Variance reduction approaches \citep{xin2020variance,xin2021improved,lian2017finite} have also been applied to improve the convergence rate of decentralized optimization. Random projection schemes have been studied to handle large sets of constraints \citep{wang2015incremental,wang2016stochastica,7178639}.
All of the above trials were made on vanilla stochastic optimization problems.

Notably, existing studies only focus on SBO in the nondistributed or star-network setting. Stochastic multilevel optimization has never been studied even in the single-server setting. Our work subsumes prior work along both directions: we generalize the communication protocol to decentralized gossip-based network, and extend the underlying problem from stochastic bilevel optimization to stochastic multi-level optimization.

\section{Preliminaries}

\noindent\textbf{Expression of $\bf{\nabla F(x)}$:} 
We start by writing the first-order derivative to problem~\eqref{prob:1}. 
Letting $f(x,y_M) = \frac{1}{K} \sum_{k \in \cK} f^k(x,y_M)$ and $g_j(x_j,y_j) =   \frac{1}{K} \sum_{k=1}^K g_j^k(x_j, y_j)   $ for each $j  =1,\cdots, M$. With a slight abuse of notation, we denote by $\nabla_1 f(x,y_M) = \frac{\partial }{\partial x} f(x,y)$, $\nabla_2 f(x,y_M) = \frac{\partial }{\partial y_M} f(x,y)$, $\nabla_1 g_j(y_{j-1},y_j) = \frac{\partial }{ \partial y_{j-1}} g_j(y_{j-1},y_j)$,  and  $\nabla_2 g_j(y_{j-1},y_j) = \frac{\partial }{ \partial y_j} g_j(y_{j-1},y_j)$ for $j=1,\cdots, M$. 
For any fixed $y_{j-1}^\star(x)$,  by using the optimality of $y_{j}^\star(x)$, we have 
$$
\nabla_2 g_j(y_{j-1}^\star(x) ,  y_{j}^\star(x)) = 0.
$$
Taking derivative with respect to $x$, we further have 
$$
\nabla y_{j-1}^\star(x)  \nabla_{12}^2 g_j(y_{j-1}^\star(x) ,  y_{j}^\star(x))  + \nabla y_j^\star(x)  \nabla_{22}^2 g_j(y_{j-1}^\star(x) ,  y_{j}^\star(x)) = 0,
$$
which implies that 
\begin{equation}\label{eq:grad_y_multilevel}
\nabla y_j^\star(x)  =  - \nabla y_{j-1}^\star(x)  \nabla_{12}^2 g_j(y_{j-1}^\star(x) ,  y_{j}^\star(x)) [ \nabla_{22}^2 g_j(y_{j-1}^\star(x) ,  y_{j}^\star(x))]^{-1}.
\end{equation}
By recursively applying the above relationship, we express the gradient as 
\begin{equation}\label{eq:gradient_multilevel}
\nabla F(x) = \nabla_1 f(x,y_M^\star(x)) + \nabla y_{M}^\star(x) \nabla_2 f(x, y_M^\star(x))
\end{equation}
where 
\begin{equation*}
\nabla y_M^\star(x)  =(-1)^M   \prod_{j=1}^{M}  \nabla_{12}^2 g_j(y_{j-1}^\star(x) ,  y_{j}^\star(x)) [ \nabla_{22}^2 g_j(y_{j-1}^\star(x) ,  y_{j}^\star(x))]^{-1}. 
\end{equation*}

We assume each agent has access to the following sampling oracle.
\begin{AS}[Sampling Oracle $\mathcal{SO}$]\label{assumption:SO}
Agent $k$ may query the sampler, receive an independent locally sampled unbiased first-order information $\nabla_1 f^k(x,y_M; \zeta^k)$ and $\nabla_2 f^k(x,y_M; \xi^k)$ for the objective, 
and receive unbiased first- and second-order information $\nabla_2 g_j^k(x,y;\xi^k)$, $\nabla_{22}^2 g_j^k(x,y;\xi^k) $, and $ \nabla_{12}^2 g_j^k(x,y; \xi^k)$  for each inner level $j=1,\cdots, M$.
\end{AS}
\begin{AS}[Gossip Protocol] \label{assumption:W}
The network gossip protocol is specified by a $K\times K$ symmetric matrix $W$ with nonnegative entries. Each agent $k$ may receive information from its neighbors, e.g., $z_j, j\in\mathcal{N}_k$, and aggregate them by a weighted sum $\sum_{j\in\mathcal{N}_k} w_{k,j} z_j$. Further, 
matrix $W$ satisfies
\begin{enumerate}
\item[(i)] $W$ is doubly stochastic such that $\sum_{i } w_{i,j} = 1$ and $\sum_{j} w_{i,j} = 1$ for all $i,j \in [K]$.
\item[(ii)] There exists a constant $\rho \in (0,1)$ such that $\| W - \frac{1}{K} \1 \1^\top \|_2^2 = \rho $, where $\| A\|_2$ denotes the spectral norm of $A \in \RR^{K \times K}$.
\end{enumerate}
\end{AS}
These assumptions on the adjacency matrix are crucial to ensure the convergence of decentralized algorithms and are commonly made in the literature of decentralized optimization  \citep{lian2017can}.

We also impose the following smoothness, boundedness, and convexity assumptions for $\{ f^k \}_{k \in \cK}$ and $\{ g_m^k \}_{k \in \cK}$ throughout this paper.

\begin{AS}\label{assumption:1}
Let $C_f, L_f$ be positive scalars. The outer level functions $\{ f^k \}_{k \in \cK} $ satisfy the following.  
\begin{enumerate}
\item[(i)] There exists at least one optimal solution to problem~\eqref{prob:1}.

{\item[(ii)] Both $\nabla_x f^k(x,y)$ and $\nabla_y f^k(x, y)$ are $L_f$-Lipschitz continuous in $(x,y)$ such that for all $x,x' \in \RR^{d_x}$ and $y_M,y_M' \in \RR^{d_y}$, 
\begin{align*}
   & \|  \nabla_1 f^k(x,y_M) - \nabla_1 f^k(x',y_M') \| \leq L_f (\| x - x' \| + \| y_M - y_M' \|),
    \\
    \text{ and }  &  \|  \nabla_2 f^k(x,y_M) - \nabla_2 f^k(x',y_M')  \| \leq L_f (\| x - x'\| + \| y_M - y_M' \|).
\end{align*}
}
 \item[(iii)] For all $x \in \RR^{d_x}$ and $y \in \RR^{d_y}$, 
 $$\EE [ \| \nabla_1 f^k(x,y_M; \zeta^k )\|^2] \leq~C_f^2 \text{ and } \EE [ \| \nabla_2 f^k(x,y_M; \zeta^k )\|^2] \leq C_f^2 .$$
\end{enumerate}
\end{AS}

Before proceeding, we also assume the following smoothness and boundedness conditions to facilitate our analysis. 
\begin{AS}\label{assumption:2}
For $m=1,\cdots,M$, let $C_{g,m}, L_{g,m}, \widetilde L_{g,m}, \mu_{g,m},\kappa_{g,m} $ be positive scalars, the inner level functions $\{ g_m^k\}_{ k \in \cK}$ satisfy the following. 
\begin{enumerate}
\item[(i)] For all $y_{m-1} \in \RR^{d_{m-1}}$,  $g_m(y_{m-1},y_{m})$ is $\mu_{g,m}$-strongly convex in $y_m$ 
\item[(ii)] For all $y_{m-1}  \in \RR^{d_{m-1}}$ and $ y_{m} \in \RR^{d_m}$, $g_m^k(y_{m-1},y_{m})$ is twice continuously differentiable in $(y_{m-1},y_{m})$. 
\item[(iii)] 
$\nabla_2 g_m^k(y_{m-1},y_{m})$, $ \nabla_{12}^2 g_m^k(y_{m-1},y_{m})$, and $\nabla_{22}^2 g_m^k(y_{m-1},y_{m})$ are Lipschitz continuous in $(y_{m-1},y_{m})$ such that for all $y_{m-1},y_{m-1}' \in \RR^{d_{m-1}}$ and $y_{m},y_{m}' \in \RR^{d_m}$, 
\begin{align*}
\| \nabla_{2} g_m^k(y_{m-1},y_{m}) - \nabla_{2} g_m^k(y_{m-1}',y_{m}')\| & \leq L_{g,m} (\| y_{m-1} - y_{m-1}' \| + \| y_{m} - y_{m}'\| ),
\\
\| \nabla_{12}^2 g_m^k(y_{m-1},y_{m}) - \nabla_{12}^2 g_m^k(y_{m-1}',y_{m}')\|_F & \leq \widetilde L_{g,m} (\| y_{m-1} - y_{m-1}' \| + \| y_{m} - y_{m}'\| ),\\
\| \nabla_{22}^2 g_m^k(y_{m-1},y_{m}) - \nabla_{22}^2 g_m^k(y_{m-1}',y_{m}')\|_F & \leq \widetilde L_{g,m} (\| y_{m-1} - y_{m-1}' \| + \| y_{m} - y_{m}'\| ).
\end{align*}

\item[(iv)] For all  $y_{m-1} \in \RR^{d_{m-1}}, y_{m} \in \RR^{d_{m}} $,  $\nabla_{2} g_m^k( y_{m-1}, y_{m}; \xi_m^k )$,  $\nabla_{22}^2 g_m^k( y_{m-1}, y_{m}; \xi_m^k )$,  and $\nabla_{12}^2 g_m^k( y_{m-1}, y_{m}; \xi_m^k)$ have bounded second-order moments such that 
\begin{align*}
\EE[ \| \nabla_2 g^k( y_{m-1}, y_{m}; \xi_m^k) \|_2^2 ] \leq C_{g,m}^2
, \EE[ \|  \nabla_{22}^2 g_m^k( y_{m-1}, y_{m}; \xi_m^k ) \|_2^2] \leq L_{g,m}^2, 
\\ \text{ and }\EE[ \| \nabla_{12}^2 g_m^k( y_{m-1}, y_{m}; \xi_m^k) \|_2^2 ] \leq L_{g,m}^2.
\end{align*}
\item[(v)] For all $y_{m-1} \in \RR^{d_{m-1}}, y_{m} \in \RR^{d_{m}} $,  $(\mathbf{I} - \frac{1}{L_{g,m}}\nabla_{22}^2 g_m^k( y_{m-1}, y_{m}; \xi_m^k ) )$ has bounded second moment such that 
$
\EE[ \| \mathbf{I}   - \frac{1}{L_{g,m}}\nabla_{22}^2 g_m^k( y_{m-1}, y_{m}; \xi_m^k )\|_2^2] \leq (1-\kappa_{g,m} )^2
$, where $0< \kappa_{g,m} \leq   \frac{\mu_{g,m} }{L_{g,m}} \leq 1$.
\end{enumerate}
\end{AS}

Note the we denote by $\| A\| = \| A\|_2 = \sigma_{\max}(A)$ the induced $2$-norm for any matrix $A$. 
Here we point out that the above assumptions allow  heterogeneity between functions $f^k$'s and $g_m^k$'s over the agents.

\begin{algorithm}[t!]
\caption{Gossip-Based Decentralized Stochastic Multi-level Optimization}\label{alg:1}
\begin{algorithmic}[1]
\REQUIRE Step-sizes $\{ \alpha_t \}$,  $ \{ \beta_t \}$, $\{ \gamma_t\}$, total iterations $T$, sampling oracle $\mathcal{SO}$, adjacency matrix $W$, smoothness constants~$\{ L_{g,m} \}_{m=1}^M$, Hessian sampling parameter $b$, number of levels $M$
\\
$x_0^k = {\bf 0}_{d_x}$,  $s_0^k = \bf{0}$, $h_0 = \bf{0}$,  $y_{m,0}^k = {\bf 0}_{d_m}$, $u_{m,0}^k = {\bf 0}_{d_{m-1} \times d_m}$, $q_{m,0}^k = {\bf 0}_{d_m \times d_m}$,  
\\
$v_{m,0,i}^k = \mu_g {\bf I}_{d_m \times d_m}$ for $i=1,2,\cdots,b$, for $m=1,\cdots,M$
\FOR{$t=0, 1, \cdots, T-1$}
	\FOR{$k=1,\cdots, K$}
	\STATE \textcolor{blue}{Local sampling}: 
 \\
 \hspace{1cm} Query $\mathcal{SO}$ at $(x_t^k,y_{M,t}^k)$ to obtain $\nabla_1 f^k(x_t^k,y_{M,t}^k;\zeta_t^k)$,
	$\nabla_2 f^k (x_t^k,y_{M,t}^k;\zeta_t^k)$. 
\STATE \textcolor{blue}{Outer loop update}: $x_{t+1}^k = \sum_{j \in \cN_k}w_{k,j} x_t^j  - \alpha_t  \left ( s_t^k + (-1)^M u_{1,t}^k q_{1,t}^k  \cdots u_{M,t}^k q_{M,t}^k  h_t^k \right ) $.

\STATE 
\textcolor{blue}{Estimate $\nabla_1 f(x_t,y_{M,t})$:}
$
s_{t+1}^k   =(1-\beta_t) \sum_{j \in \cN_k} w_{k,j} s_{ t}^j + \beta_t  \nabla_1 f^k(x_t^k, y_{M, t}^k ; \zeta_t^k )$.
\STATE 
\textcolor{blue}{Estimate $\nabla_2 f(x_t,y_{M, t})$:}
$
h_{t+1}^k  =(1-\beta_t) \sum_{j \in \cN_k} w_{k,j} h_{ t}^j + \beta_t  \nabla_2 f^k(x_t^k, y_{M, t}^k ;  \zeta_t^k).  
$
	\STATE \textcolor{blue}{Level $m$ Update:} 
	\FOR{$m=1,\cdots, M$}  
 \STATE \textcolor{blue}{Local sampling}: Query $\cS\cO$ at $(y_{m-1,t}^k, y_{m,t}^k)$ to obtain $\nabla_2 g_m^k (y_{m-1,t}^k, y_{m,t}^k;\xi_{m,t}^k )$, 
 \\
 \hspace{1cm} $\nabla_{12}^2  g_m^k(y_{m-1,t}^k, y_{m,t}^k;\xi_{m,t}^k)$,   and  $\{ \nabla_{22}^2 g_m^k(y_{m-1,t}^k, y_{m,t}^k;\xi_{m,t,i}^k)  \}_{i=1}^b$. 
		\STATE 
	\textcolor{blue}{Inner loop update}:  
$
	y_{m, t+1}^k  =  \sum_{j \in \cN_k} w_{k,j} y_{m, t}^j   - \gamma_t \nabla_2 g_m^k ( y_{m-1, t}^k,y_{m,t}^k;\xi_{m,t}^k ).
$
\STATE \textcolor{blue}{Estimate $\nabla_{12}^2 g_m(y_{m-1,t}^k, y_{m,t}^k)$:} 
\\
\hspace{2cm}
$
u_{m, t+1}^k   =(1-\beta_t) \sum_{j \in \cN_k} w_{k,j} u_{m, t}^j + \beta_t \nabla_{12}^2 g_m^k( y_{m-1, t}^k,  y_{m, t}^k; \xi_{m,t}^k).
$
\STATE 
\textcolor{blue}{Estimate $[\nabla_{22}^2  g_m(y_{m-1,t}^k, y_{m,t}^k)]^{-1}$:} 
Set 
$Q_{m,t+1,0}^k = \mathbf{I}$
\FOR{ $i= 1,\cdots,b$}
\STATE $v_{m,t+1,i}^k  =(1-\beta_t) \sum_{j \in \cN_k} w_{k,j} v_{t,i}^j + \beta_t  \nabla_{22}^2 g_m^k(y_{m-1,t}^k, y_{m,t}^k;\xi_{m,t,i}^k),$ 
\STATE $Q_{m,t+1,i}^k = \mathbf{I} + ( \mathbf{I} - \frac{1}{L_{g,m}} v_{m,t+1,i}^k)Q_{m,t+1,i-1}^k$
\ENDFOR 
\STATE Set $q_{m,t+1}^k = \frac{1}{L_{g,m} } Q_{m,t+1,b}^k$
\ENDFOR
 	\ENDFOR
 \ENDFOR
 \ENSURE $\bar{x}_t=\frac{1}{K} \sum_{k \in\cK} x_t^k$
\end{algorithmic}
\end{algorithm}

\section{Algorithm}

As discussed in Section~\ref{sec:challenges}, the key challenge to solving DSMO is that each agent only has  access to its own data but is required to construct estimators for the gradients and Hessian averaged across all agents. It is particularly challenging to construct such estimators when limited by the network's communication protocol. 

Now we propose a gossip-based DSMO to tackle problem~\eqref{prob:1}. In our algorithm,  each agent $k \in \cK$ iteratively updates a sequence of solutions $(x_t^k,y_{1,t}^k,\cdots, y_{M,t}^k)$ by using the combination of gossip communications and weighted-average stochastic approximation, where $y_{m,t}^k$ is agent $k$'s estimator of the best response to the $m$-th level solution $y_{m,t}^\star := y_m^\star(\bar x_t)$ for all levels $m=1,\cdots, M$, with  $\bar x_t : = \frac{1}{K}\sum_{k\in \cK} x_t^k$ representing the solution averaged over all agents.

We provide the details of our DSMO algorithm in Algorithm~\ref{alg:1} and explain the concept here. 
Briefly speaking, under the decentralized SMO setting, the key to solving problem~\eqref{prob:1} is to provide good approximation for each component within the gradient expression of $\nabla F(\bar x_t)$ \eqref{eq:gradient_multilevel}. To do so, we first update our estimators of  $\nabla_1 f(x_t,y_{M, t})$.
Suppose agent $k$ would like to estimate $\nabla_1 f(x_t^k,y_{M,t}^k)$ by  $s_t^k$, under Algorithm~\ref{alg:1} Step 5, it would query the stochastic first-order information using its own data,  make gossip communications with neighbors, and update its estimators by taking the weighted average of its previous estimate $ s_{t-1}^{k} $, neighbors $j$'s estimate $s_{t-1}^{j}$ and the newly sampled gradient $\nabla_1 f^k(x_t^k,y_{M,t}^k;\zeta_t^k)$ as 
\begin{equation}\label{eq:update_st}
    s_{t+1}^k   =(1-\beta_t) \sum_{j \in \cN_k} w_{k,j} s_{ t}^j + \beta_t  \nabla_1 f^k(x_t^k, y_{M, t}^k ; \zeta_t^k ).
\end{equation}
Roughly speaking, this procedure can be viewed as taking the weighted average of gradients sampled by all agents over the network, 
except that the effect of consensus should also been taken into account. The  outer-level gradient  $\nabla_2 f(x_t,y_{M, t})$ can be estimated similarly (Step 6).

Next, we move to estimate the inner-level gradients and Hessians for each level $m=1,\cdots, M$. We start from the level $m=1$ for the purpose of illustration. Let $y_{1,t}^k$ be the estimator of $y_{1,t}^\star$ maintained by agent $k$, we first perform an inner-loop update as  
$$
	y_{1, t+1}^k  =  \sum_{j \in \cN_k} w_{k,j} y_{1, t}^j   - \gamma_t \nabla_2 g_1^k (x_t^k,y_{M,t}^k;\xi_{1,t}^k ),
$$
which  communicates the inner-level solutions $y_{1,t}^j$'s over the neighbors $j \in \cN_k$ and conducts a stochastic gradient descent step with pre-fixed stepsize $\gamma_t$. 
We then update 
$\nabla_{12}^2 g_1$ by using a similar manner as \eqref{eq:update_st} and summarize it in Algorithm~\ref{alg:1} Step 10.  

However, it requires extra effort to evaluate $[\nabla_{22}^2  g_1 ]^{-1}$, because it has no unbiased estimator. 
To be specific, we note that $\frac{1}{K}\sum_{k\in \cK} [\nabla_{22}^2  g_1^k]^{-1} \neq [\frac{1}{K} \sum_{k \in \cK} \nabla_{22}^2 g_1^k]^{-1}$, making the unbiased estimator of the desired term unavailable even if each agent has an unbiased estimator for $[\nabla_{22}^2  g_1^k]^{-1}$. This is a unique challenge for decentralized multilevel optimization, as discussed in Section~\ref{sec:challenges}. 
To overcome this issue, we propose a \emph{novel} approach that each agent $k$ constructs $b$ independent estimators $\{ v_{1,t,j}^k \}_{j=1}^b$ for $\nabla_{22}^2 g(x_t^k,y_{M,t}^k)$ using consensus and stochastic approximation. We then estimate $ \nabla_{22}^2 g(x_t^k,y_{M,t}^k) $ by utilizing the following approximation.
$$
   \Big  (\frac{1}{L_g} \nabla_{22}^2 g_1(x_t^k,y_{M,t}^k)  \Big )^{-1} \approx \mathbf{I} +  \sum_{j=1}^b  \Big (\mathbf{I} - \frac{1}{L_g} \nabla_{22}^2 g_1(x_t^k,y_{M,t}^k) \Big )^j \approx \mathbf{I} +  \sum_{i=1}^b \prod_{j=1}^i \Big (\mathbf{I}  - \frac{1}{L_g} v_{1,t,j}^k   \Big ). 
$$
We provide details in Algorithm~\ref{alg:1} Steps 11 - 16. 

After conducting the above updates for the first level $m=1$, we employ this updating scheme and recursively estimates all essential components $y_{m,t}$, $\nabla_{12}^2 g_m(  y_{m-1,t},  y_{m,t})$, and $\nabla_{22}^2 g_m( y_{m-1,t},  y_{m,t}))$
for all rest levels $m=2,\cdots, M$. We summarize the details in Steps 8-17. 

Finally, each agent computes the full gradient  $\nabla F(\bar x_t)$ \eqref{eq:gradient_multilevel}  using   the estimators obtained in the above procedure and updates the outer solution $x_t^k$ by using the combination of gossip communication and stochastic gradient descent (Step 4).

\noindent 
\textbf{Key features.}
We highlight the following key features of our DSMO algorithm:
(1) each agent only communicates their current iterates, as well as gradient and Hessian estimates instead of the raw data in the gossip-communication process,  preserving data privacy. 
(2) agent $k$  makes $\cO(|\cN_k|)$ communications with its neighbors in each round, which is much smaller than the total number of agents in a naive approach.
(3) the algorithm is robust to contingencies in the network. If a communication channel fails, the agents can still jointly learn provided that the network is still connected. By contrast, a single-center-multi-user network would fail completely in case of a center failure. 
(4) the algorithm estimates the Hessian inverse $[\nabla_{22}^2g_m(x,y)]^{-1}$ by using $b$ independent estimators $\{ v_{m,t,j}^k\}_{j=1}^b$, which avoids the expensive $O(d_m^3)$ computational cost of directly inverting the Hessian when the dimension $d_m$ is large.

\medskip 

\section{Theoretical Guarantee}
In this section, we analyze the performance of our DSMO algorithm for  both nonconvex and $\mu$-PL objectives and derive the convergence rates in both cases.

\subsection{Nonconvex Objectives}\label{sec:nonconvex}
We first consider the scenario where the overall objective function $F(x)$ is nonconvex. For nonconvex objectives, given the total number of iterations $T$, we employ the step-sizes in a constant form such that 
\begin{equation} \label{eq:stepsize_multilevel}
\alpha_t = C_0 \sqrt{\tfrac{K}{T}}, \   \beta_t = \gamma_t = \sqrt{\tfrac{K}{T}},  \ \text{ and } b_m = 3 \lceil \log_{\frac{1}{1-\kappa_{g,m}}} T \rceil ,    \text{ for all }t=0,1,\cdots, T,
\end{equation} 
where $C_0>0$ is a small constant and the number of iterations $T $ is large   such that $\beta_t ,\gamma_t \leq 1$. 
\medskip 
\\
\textbf{Compounded effect of consensus and SMO:}
As discussed earlier, to derive the convergence rate of SMO under a decentralized federated setting, the key step is to quantify the compounded effect between the consensus errors induced by the network structure and the biases induced by estimating gradients within~\eqref{eq:gradient_multilevel}. Unlike the central-server or non-federated regimes, the consensus errors induced by the decentralized network structure must be handled carefully. 
We conduct a thorough  analysis to derive the contraction of consensus errors, and further show that both bias and variance of the averaged estimator diminish to zero, establishing a nontrivial convergence argument for the desired gradient and Hessian. In particular, the estimators preserve a concentration property so that their variances decrease proportionally to $1/K$, suggesting that the network consensus effect does not degrade the concentration of the generated stochastic samples. To achieve the best possible convergence rate, we carefully set the algorithm parameters, including the step-sizes $\alpha_t, \gamma_t$ and averaging weights $\beta_t$, to control the above consensus errors and biases.


We derive the convergence rate as follows and provide the detailed proof in Appendix Section~\ref{app:nonconvex}.

\begin{theorem}\label{thm:nonconvex_multilevel}
Suppose Assumptions
\ref{assumption:SO}, \ref{assumption:W}, \ref{assumption:1},  and \ref{assumption:2} hold. Letting $\bar x_t = \frac{1}{K}\sum_{k \in \cK} x_t^k$, then
    \begin{equation*}
\begin{split}
&\frac{1}{T}\sum_{t=0}^{T-1} \EE[ \| \nabla F(\bar x_t)\|^2]  \leq  \cO \left (\frac{1}{ \sqrt{ KT} } \right )  + \cO \left ( \frac{K}{T(1- \rho)^2 } \right ) . 
  \end{split}
\end{equation*}
\end{theorem}
\textbf{Effect of consensus:}
In this result, the $\cO(\tfrac{K}{T(1- \rho)^2})$ term represents the errors induced by the consensus of the network. Despite depending on  the network structure, this term diminishes to zero in the order of $\cO(1/T)$, becoming a small order term when $T$ is large. Consequently, our result indicate that the asymptotic convergence behavior of DSMO is independent of the network structure, answering question (ii) raised in Section~\ref{sec:intro}. 
\medskip 
\\
\textbf{Linear speedup:} 
Because each agent queries $\cO(b)$ stochastic samples per round, clearly the required iteration and per-agent sample complexities for finding an $\epsilon$-stationary point such that 
$
\frac{1}{T}\sum_{t=0}^{T-1} \EE[ \| \nabla F(\bar x_t)\|^2]  \leq  \epsilon
$
are  $\cO(\frac{1}{K \epsilon^2})$ and $\widetilde\cO(\frac{1}{K \epsilon^2})$, respectively. This result implies that our algorithm achieves a linear speed-up effect proportionate to the number of agents $K$, regardless of the number of levels $M$. In other words, in the presence of more agents, each agent needs to obtain fewer stochastic samples to achieve a specified accuracy. 
Meanwhile, our rate also matches the best-known $\cO(1/K\epsilon^2)$ iteration and per-agent sample complexities under the decentralized vanilla stochastic gradient descent settings~\citep{lian2018asynchronous}.  This is the first time  such a   result has been established for DSBO and general DSMO problems. 
\medskip 
\\
\textbf{Single-center-multi-user-federated SMO}: We point out that a simplified version of our algorithm solves SMO in star networks, where the central server  collects information directly from each agent and calculates the gradient by employing the weighted-average stochastic approximation scheme for the collected information. In such a scenario, the agents no longer  communicate by gossip  with neighbors but synchronously receive a common solution from the central server, so that the consensus effect disappears.

\subsection{Proof Sketch of Theorem~\ref{thm:nonconvex_multilevel}.}
First, we characterize the smoothness properties of the best response function $y_m^\star(x)$ for each decision level $j=1,\cdots, M$. By using Lemma~\ref{lemma:Lip_y_multilevel}, we show that under Assumption~\ref{assumption:2}, the best response for the $m$-th decision level $y_{m}^\star(x)$ is $L_{y,m}$-Lipschitz continuous such that 
 $$
 \|  y_m^\star(x) -  y_m^\star(x') \| \leq L_{y,m} \| x - x'\|,  \ \ \,\,\forall x,x'\in \RR^{d_{x}},
 $$
 where $L_{y,m}= \prod_{j=1}^m\frac{L_{g,j}}{\mu_{g,j}}$. By using this property, we are able to analyze the compounding effect of the consensus error $\|x_t^k - \bar x_t\|^2$ when each agent  keeps its own solution $x_t^k$. For the $m$-th decision level, Lemma~\ref{lemma:y_multilevel} suggests that 
 \begin{equation}
\begin{split}
& \EE [ \| \bar y_{m,t+1} - y_{m,t+1}^\star \|^2 ] 
\\
&
\leq ( 1 - \gamma_t \mu_{g,m})  \EE [ \| \bar y_{m,t} -  y_{m,t}^\star \|^2 ]  + \cO\left (\frac{ \gamma_t \alpha_t^2  + \gamma_t^3 }{(1- \rho )^2 } \right ) + \frac{ \gamma_{t} L_{g,m}^2 }{ \mu_{g,m}} \EE[ \| \bar y_{m-1,t} - y_{m-1,t}^\star \|^2] 
\\
& \quad + \frac{2\gamma_t^2 \sigma_{g,m}^2 }{K}  + \frac{3L_{y,m}^2\alpha_t^2 }{\gamma_t \mu_{g,m}  } \EE [\| \bar z_t \|^2 ]  .
\end{split}
\end{equation}
which recursively bounds the error $\| \bar y_{m,t+1} - y_{m,t+1}^\star \|^2$ incurred in iteration $t+1$ by the error $\| \bar y_{m,t} -  y_{m,t}^\star \|^2$ incurred in iteration $t$  and the error $\| \bar y_{m-1,t} - y_{m-1,t}^\star \|^2$ incurred in the previous level $m-1$. This answers how the estimation and consensus errors are compounded across different levels in DSMO under a general network topology.

Second, by using the $L_F$-smoothness property of $F(x)$, 
Lemma~\ref{lemma:nonconvex_main_multilevel} suggests that 
   \begin{equation*}
\begin{split}  
& \EE[ \| \nabla F(\bar x_t)\|^2] 
\\
& \leq  \frac{2}{\alpha_t } \Big ( \EE[ F(\bar{x}_t)] -  \E[F(\bar{x}_{t+1}) ] \Big ) 
  -  (1 - \alpha_t L_F ) \EE[ \|  \bar z_t  
\|^2  ]  
 + 4 \underbrace{ \EE [ \|  \nabla_1 f( \bar x_t ,   y_{M,t}^\star   ) - \bar s_t \|^2 ]  }_{\Delta_1}
 \\
 & \quad +  C_{2} \underbrace{ \EE[  \| \nabla_2 f(  \bar x_t  ,   y^\star_{M,t}    ) -  \bar h_t \|^2 ]  }_{\Delta_2}  + \sum_{m=1}^M  C_{m,3} \underbrace{ \EE[ \|  \nabla_{12}^2 g_m( y_{m-1,t}^\star, y_{m,t}^\star    ) - \bar u_{m,t} \|_F^2 ]   }_{\Delta_{3,m}}
 \\
 & \quad 
+    \sum_{m=1}^M C_{m,4,j} \underbrace{ \EE[ \| \nabla_{22}^2 g_m(y_{m-1,t}^\star, y_{m,t}^\star  ) -  \bar v_{m,t,i} \|_F^2]  }_{\Delta_{4,m}} + \cO \left ( \frac{M \beta_t^2}{(1-  \rho  )^2}\right ),
  \end{split}
\end{equation*}
for some constants $C_2,C_{m,3}, C_{m,4,j} >0$. Here $\Delta_1$ quantifies the error of estimating the partial gradient $\nabla_1 f(\bar x_t, y_{M,t}^\star)$ by $\bar s_t = \frac{1}{K} \sum_{k=1}^K s_{t}^k$ averaged over all agents, $\Delta_2$ quantifies the approximation error for the partial gradient $\nabla_2 f(\bar x_t, y_{M,t}^\star)$, $\Delta_{3,m}$ is the estimation error of the second-order information $\nabla_{12}^2g_m( y_{m-1,t}^\star, y_{m,t}^\star    )  $ for the $m$-level function $g_m$, and $\Delta_{4,m}$ represents the estimation error of the Hessian $\nabla_{22}^2 g_m( y_{m-1,t}^\star, y_{m,t}^\star    )$ for the $m$-th level. Meanwhile, the term $-(1-\alpha_t L_F)\EE[\| \bar z_t\|^2]$ becomes negative when $\alpha_t$ is properly chosen to be small. 
It remains to bound each of the approximation error within the above inequality to establish the overall convergence result.

To analyze the estimation error $\| \bar s_t - \nabla_1 f( \bar x_{t}, y_{M,t}^\star ) \|^2 $ of the averaged estimator $\bar s_t = \frac{1}{K}\sum_{k=1}^K s_t^k$, Lemma~\ref{lemma:error_multilevel} (a) shows that the estimation error incurred in iteration $t+1$ can be bounded by the error in iteration $t$ as
\begin{equation}\label{eq:sketch_1}
    \begin{split}
    &    \EE[ \|  \bar s_{t+1} - \nabla_1 f( \bar x_{t+1}, y_{M,t+1}^\star ) \|^2]  
        \\
              & \leq (1-\beta_t)\EE [ \| \bar s_t - \nabla_1 f (   \bar x_{t}, y_{M}^\star(\bar x_t) ) \|^2] + \frac{2\beta_t^2   C_f^2}{K }    + \frac{4\alpha_t^2  L_f^2 (1 + L_{y,M}^2 )}{\beta_t }  \EE[  \| \bar z_t \|^2  ] 
              \\
              & \quad + \frac{2\beta_t^2  \sigma_f^2}{K } +    \cO \left ( \frac{ \beta_t ( \alpha_t^2 +  \beta_t^2) }{(1- \rho)^2 } \right )  + 6\beta_t L_f^2  \EE[ \| \bar y_{M,t} - y_{M,t}^\star \|^2 ]. 
        \end{split}
\end{equation}
The error $\|  \bar h_{t+1} - \nabla_2 f( \bar x_{t+1}, y_M^\star(\bar x_{t+1}) ) \|^2$ can be analyzed analogously in Lemma~\ref{lemma:error_multilevel}~(b). In addition, Lemma~\ref{lemma:ut} quantifies the estimation errors $ \|  \bar u_{m,t} - \nabla_{12}^2 g_m( y_{m-1,t}^\star, y_{m,t}^\star ) \|_F^2 $ and $\| \nabla_{22}^2 g_m(y_{m-1,t}^\star, y_{m,t}^\star  ) -   v_{m,t,i}^k \|_F^2$.

After establishing the error bound in the above terms, by combining the above blocks and  setting the step-sizes as 
$\alpha_t = C_0 \sqrt{\tfrac{K}{T}}$ and $\beta_t = \gamma_t = \sqrt{\tfrac{K}{T}}$ for some  properly chosen  small $C_0 >0$, we conclude that 
  \begin{equation*}
\begin{split}
\frac{1}{T}\sum_{t=0}^{T-1} \EE[ \| \nabla F(\bar x_t)\|^2] 
 \leq  \cO \Big ( \frac{M}{\sqrt{ TK} } \Big )   + \cO \left ( \frac{KM}{T (1- \rho)^2 }\right ),
  \end{split}
\end{equation*}
completing the proof. We defer the detailed proofs to Appendix Section~\ref{app:nonconvex}. 
\QED 

\subsection{$\mu$-PL Objectives}\label{sec:PL}

Next we study the case where the objective function satisfies the following $\mu$-PL condition. 
\begin{AS}\label{assumption:PL}
There exists a constant $\mu >0$ such that the objective satisfies the PL condition: 
\begin{equation*}\label{def:PL}
    2 \mu (F(x ) - F^*) \leq \| \nabla F(x)\|^2.
\end{equation*}
\end{AS}

Note that the class of 
strongly convex functions is a special case of $\mu$-PL functions. 
To utilize the $\mu$-PL property and achieve fast convergence, unlike the nonconvex case, where the step-sizes \eqref{eq:stepsize_multilevel} are set to constants depending on the total number of iterations $T$,  we employ  step-sizes in a diminishing form such that 
\begin{equation}\label{eq:stepsize_PL}
\alpha_t = \frac{2}{\mu(C_1 + t)} ,   \beta_t =  \gamma_t = \frac{C_1}{C_1 + t}, \text{ and }   b = \Theta( \log( T )) 
\text{ for } t \geq 1,
\end{equation}
where $C_1 >0$ is  a large constant. 
By following an analytical process similar to that of the nonconvex scenario, in the next result, we derive the convergence rate of Algorithm~\ref{alg:1}  for $\mu$-PL objectives. 

\begin{theorem}\label{thm:PL}
Suppose Assumptions \ref{assumption:SO}, \ref{assumption:W}, \ref{assumption:1}, and \ref{assumption:2} hold and the function satisfies the $\mu$-PL Assumption~\ref{assumption:PL}.  Letting $\bar x_T = \frac{1}{K}\sum_{k \in \cK} x_T^k$, then
\begin{equation*}
\EE[ F(\bar x_T) ] - F^* \leq \cO \left (\frac{1}{KT} \right ) + \cO \left (\frac{\ln T}{T^2(1-\rho)^2} \right ).
\end{equation*}
The iteration and per-agent sample complexities for finding an $\epsilon$-optimal point $\EE[F(\bar x_T)] - F^* \leq \epsilon$ are $\cO(\frac{1}{K\epsilon})$ and $\widetilde \cO(\frac{1}{K\epsilon})$, respectively.
\end{theorem}

Details and proof are deferred to  Section \ref{app:proof_of_thm_PL} of the supplement.
This result shows that our algorithm achieves a faster convergence rate for functions satisfying the $\mu$-PL condition in terms of both iteration and sample complexities. 
First, as in the nonconvex scenario, the consensus error decays in the order of $\widetilde \cO(\frac{1}{T^2(1-\rho)^2})$. Dominated by $\cO(\frac{1}{KT})$, such consensus decaying order indicates that the network structure would not affect Algorithm~\ref{alg:1}'s asymptotic convergence behavior under $\mu$-PL objectives. 
Meanwhile, the above result implies that Algorithm~\ref{alg:1} speeds up  linearly  with the number of agents and matches the optimal $\cO(1/\epsilon)$ sample complexity for single-server vanilla strongly-convex stochastic optimization~\citep{rakhlin2011making}. As a result, our algorithm achieves the \emph{optimal} sample complexity for decentralized stochastic bilevel optimization, establishing the benchmark.


\section{Numerical Experiments}\label{sec:numerics}
In this section, we validate the practical performance of our algorithm in three applications: hyper-parameter optimization, policy evaluation in Markov Decision Processes (MDP), and risk-averse optimization, on artificially constructed decentralized ring networks.

\subsection{Hyper-parameter Optimization}
We consider federated hyper-parameter optimization~\eqref{prob:hyperopt} for a handwriting recognition problem over the Australia handwriting dataset \citep{chang2011libsvm} consisting of data points $(w_i,z_i)$, where $w_i \in \RR^{14}$ is the feature and $z_i \in \{ 0,1\}$ indicates whether this data point belongs to category ``1'' or not.  
In our experiment, we consider the  sigmoid loss function that $l_i(z) = 1/(1+ \exp(-z))$ and a strongly convex regularizer 
 $\cR(x,y)  = \sum_{i=1}^d \frac{x_i }{2} \| y_i \|^2 $.
 We consider a ring network of $K$ agents where each agent $i$ preserves two neighbors $(i-1)$ and $(i+1)$ and conducts a gossip communication strategy with adjacency matrix $w_{i,j} = \frac{1}{3}$ for $ j  \in \{ i-1,i,i+1 \}$. We tackle this problem by our DSMO Algorithm~\ref{alg:1} with $M=2$ levels. 
 
 Before testing Algorithm~\ref{alg:1}, we first randomly split the dataset for training and validation, and then allocates 
both training and validation dataset over $K$ agents. We then run  Algorithm~\ref{alg:1} for $T=20000$ iterations, with $b = 200$, $\alpha_t = 0.1\sqrt{K/T}$, and $\beta_t = \gamma_t = 10\sqrt{K/T}$.

To provide a benchmark for comparison, we implement a baseline  Decentralized Bilevel Stochastic Approximation (DBSA) algorithm, a naive extension of the double-loop BSA algorithm \citep{ghadimi2018approximation} in the decentralized setting,  formally stated in Section~\ref{app:hyper_opt} of the supplementary materials.

We first consider $K = 5$, test Algorithm~\ref{alg:1} for $5 \times 10^4$ iterations, and compare its performance with DBSA.  We report the validation loss against total samples  in Figure~\ref{fig:1} and observe that DSMO exhibits better performance than DBSA. In particular, Further, we observe that our algorithm outperforms the baseline algorithm DBSA in that it requires 
fewer samples for DSMO to achieve the same accuracy.

To investigate the efficiency of Algorithm~\ref{alg:1} to the network structure, we test Algorithm~\ref{alg:1} 
over $K = 5,10,20$, and report the details of training and validation loss in Figure~\ref{fig:1}. Further, comparing the performances of Algorithm~\ref{alg:1} over different agents $K=5,10,20$, we observe that Algorithm~\ref{alg:1} converges faster when using more  agents. 
This observation suggests that Algorithm~\ref{alg:1} exhibits a speed-up effect when using more agents. 
We provide additional experiments on networks of larger size ($K=100$) and various topologies (fully-connected and randomly-connected) in Section~\ref{app:hyper_opt} of the supplement.

\begin{figure}[t]
\begin{minipage}{0.5\textwidth}\label{fig:source_v_trans}
    \centering
    \includegraphics[width=0.65\linewidth]{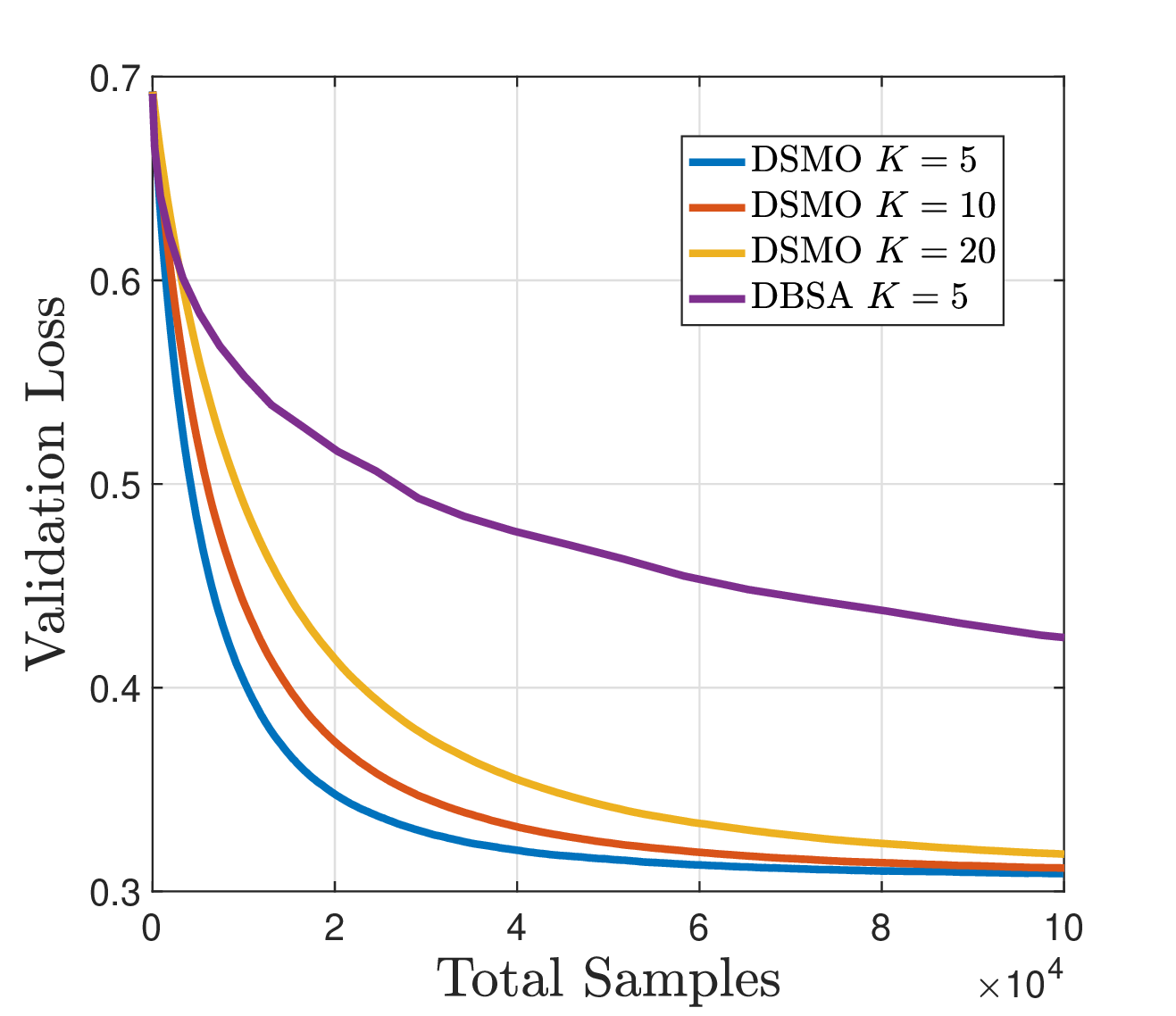}\\
    (a)
\end{minipage}
\begin{minipage}{0.5\textwidth}
    \centering
    \includegraphics[width=0.65\linewidth]{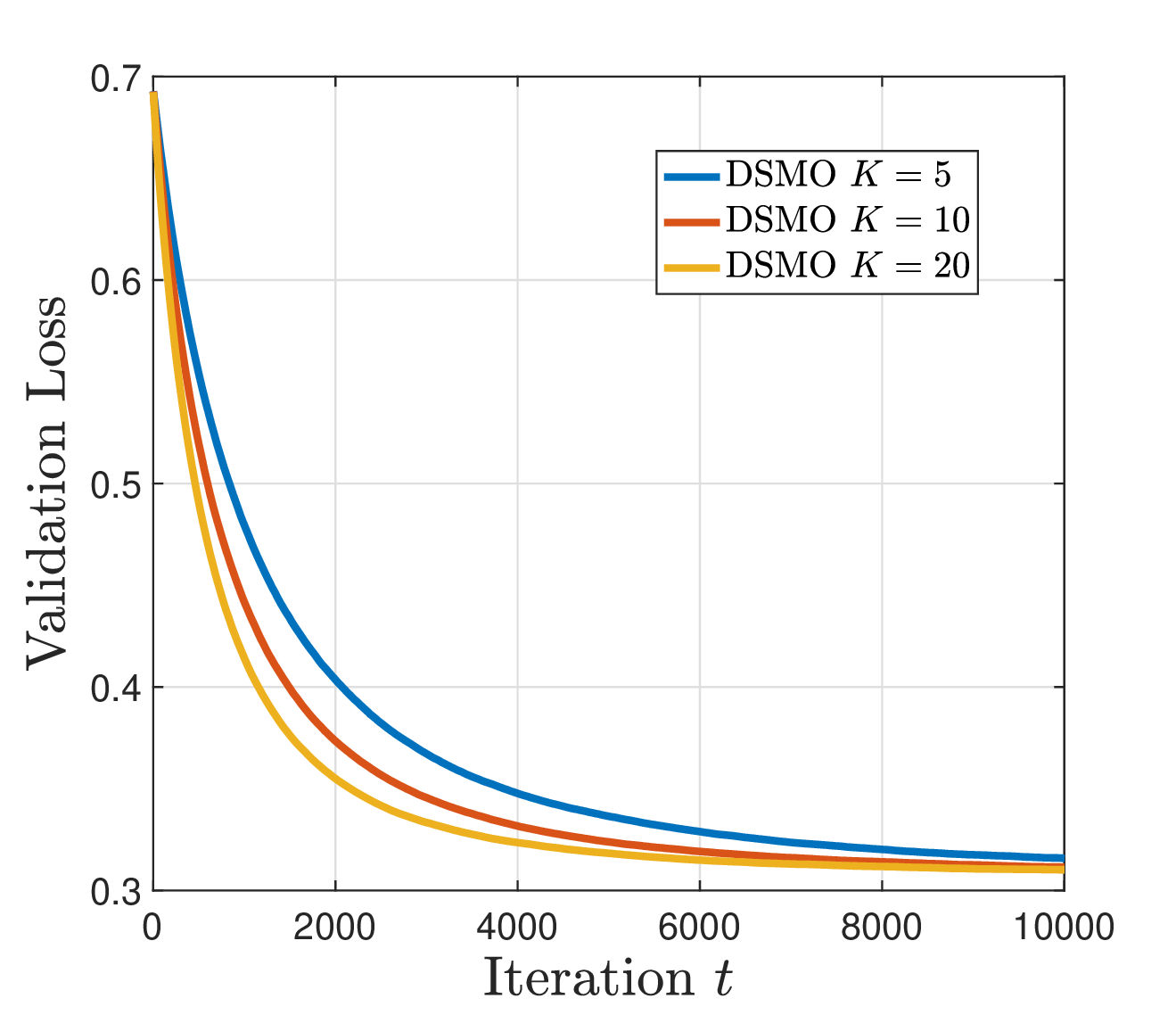}\\
    (b)
\end{minipage}
\caption{(a) Empirical averaged training loss against total samples  for DSMO $K=5,10,20$ and DBSA $K = 5$. 
(b) Empirical averaged validation loss against iteration for DSMO $K=5,10,20$. All figures are generated through 10 independent simulations over the Australia handwriting dataset.
}
\label{fig:1}
\vskip -3mm
\end{figure}

\subsection{Distributed Policy Evaluation for Reinforcement Learning}
We consider a multi-agent MDP problem that arises in reinforcement learning. Let $\cS$ be the state space. For any state $s \in \cS$, we denote by $V(s)$ the value function. 
we consider the scenario where the value function can be approximated by a linear function such that $V(s) = \phi_s^\top x^*$, where $\phi_s \in \RR^m$ is a feature and $x^* \in \RR^m$ is an unknown parameter.   
To obtain the optimal $x^*$, we consider the following regularized Bellman minimization problem
\begin{equation*}
    \min_{x} \ F(x) = \tfrac{1}{2|\cS|} \sum_{s \in \cS} \big (\phi_s^\top x - \EE_{s'}[r(s,s') + \gamma \phi_{s'}^\top x \mid s ] \big )^2 + \tfrac{\lambda \|x \|^2}{2},
\end{equation*}
where $r(s,s')$ is the random reward incurred from a transition $s$ to $s'$, $\gamma \in (0, 1)$ is the discount factor, $\lambda$ is the coefficient for the $\ell_2$-regularizer, and the expectation is taken over all random transitions from $s$ to $s'$. 

In the federated learning setting, we consider a ring network of $K$ agents. Here each agent $k$ has access to its own data with a heterogeneous random reward function $r^k$ and can only communicate with its two neighbors $k+1$ and $k-1$. 
We denote by 
\begin{equation*}
\begin{split}
y_s^\star(x) & = \argmin_y \EE_{s' } \Big ( \phi_s^\top x -  \frac{1}{K} \sum_{k\in\cK}r^k(s,s') + \gamma \phi_{s'}^\top x   - y \  \Big |  \ s\Big )^2 
\\
& 
= \phi_s^\top x -  \EE_{s'} \Big [\frac{1}{K} \sum_{k\in\cK}r^k(s,s') + \gamma \phi_{s'}^\top x \ \Big | \  s \Big ] 
\end{split}
\end{equation*}
where $r^k(s,s')$ is the random reward function for agent $k$. 
The above problem can be recast as a bilevel optimization problem 
\begin{equation*}
    \min_{x \in \RR^d }f(x,y^\star(x)) = \tfrac{1}{2|\cS|} \sum_{ s \in \cS}( y_s^\star(x) )^2 + \tfrac{\lambda \|x \|^2}{2} . 
\end{equation*}
As pointed out by \citep{wang2016accelerating}, the above problem is $\lambda$-strongly convex. 

In our experiments, we simulate an environment with state space $|\cS| = 100$ and set the regularizer parameter $\lambda = 1$. 
We test the performance of Algorithm\ref{alg:1} over three scenarios with $K= 5,10,20$ and conduct 10 independent simulations for each $K$. We implement a baseline double-loop algorithm DSGD that first estimates $y_s^\star(x_t)$ with $t$ samples in iteration $t$ and then optimizes the solution $x_t$. We defer the implementation details of the environment and  above algorithms to Section \ref{app:MDP} of the supplement.

We first consider $K = 5$, run Algorithm~\ref{alg:1} for $10^4$ iterations and compare its performance with DSGD. We plot the empirical averaged mean square error $\|\bar  x_t - x^*\|^2$ against total samples  generated by \emph{all} agents in Figure \ref{fig:1}. This empirical result suggests that Algorithm~\ref{alg:1} outperforms DSGD. 
To investigate the convergence rate of DSMO, we compare the performance of DSMO over all three setups $K=5,10,20$ and plot the trajectory of the averaged log-error $ \log( \| \bar x_t - x^* \|^2)$ averaged, with a straight line of slope -1 provided for comparison. We observe that for all three scenarios, the slopes of $\log(\| \bar x_t - x^*\|^2)$ are close to -1,  matching our theoretical claim in Theorem~\ref{thm:PL} that Algorithm~\ref{alg:1} converges at a rate of $\cO(1/t)$ for strongly convex objectives. 

In the above experiment, we also note that Algorithm~\ref{alg:1} converges faster when using more agents.
To further demonstrate the linear speedup effect, we compute the total  samples generated to find an $\epsilon$-optimal solution $\|\bar x_t - x^* \|^2 \leq \epsilon $ and plot the 75\% confidence region of the log-sample against the number of agents $K = 5,10,20$ in Figure~\ref{fig:MDP1}. We observe that it takes a roughly same amount of samples to find a $10^{-6}$-optimal solution despite different number of agents are involved. This suggests that the per-node sample complexity decreases linearly with $K$, validating the linear speedup claim in Theorem~\ref{thm:PL}. We provide additional numerical results for other optimality level $\epsilon$ in Section~\ref{app:MDP} of the supplementary material to further demonstrate the linear speedup effect.

\begin{figure}[t]
\begin{minipage}{0.32\textwidth}
    \centering
        \includegraphics[width=1\linewidth]{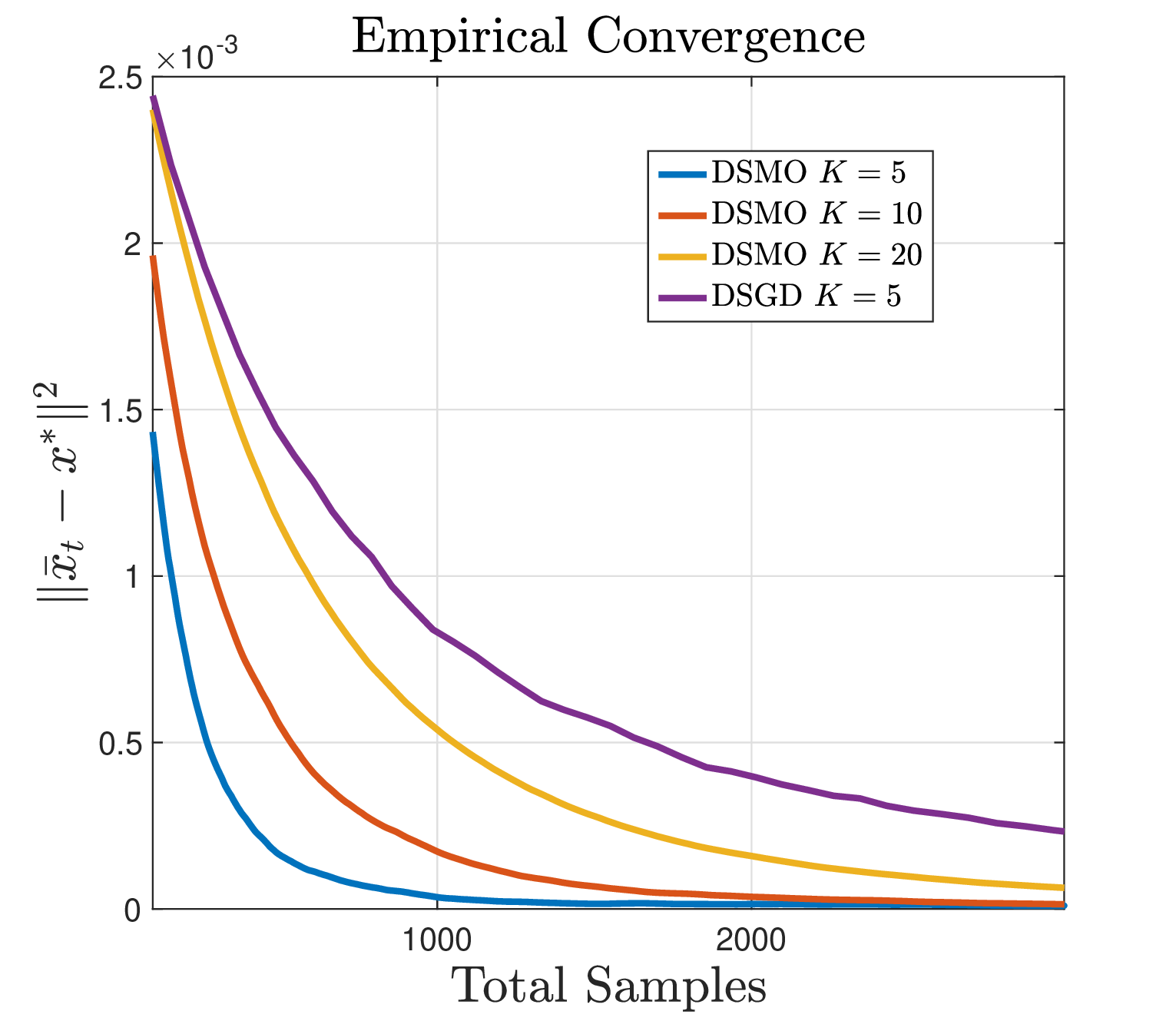}
        (a)
        \end{minipage}  
        \begin{minipage}{0.32\textwidth}
        \centering
   \includegraphics[width=1\linewidth]{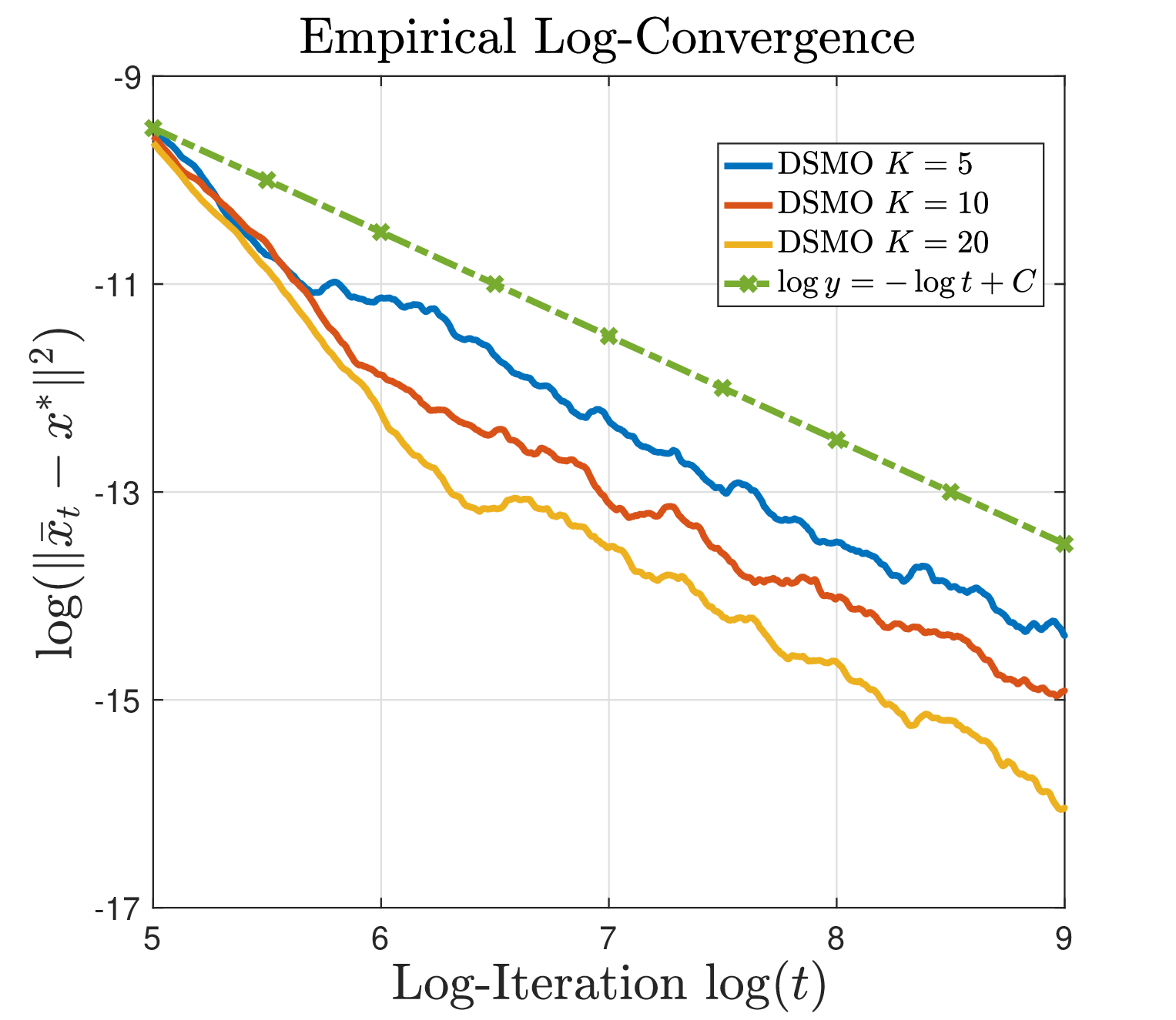} 
   (b)
   \end{minipage}
                  \begin{minipage}{0.32\textwidth}
        \centering
   \includegraphics[width=1\linewidth]{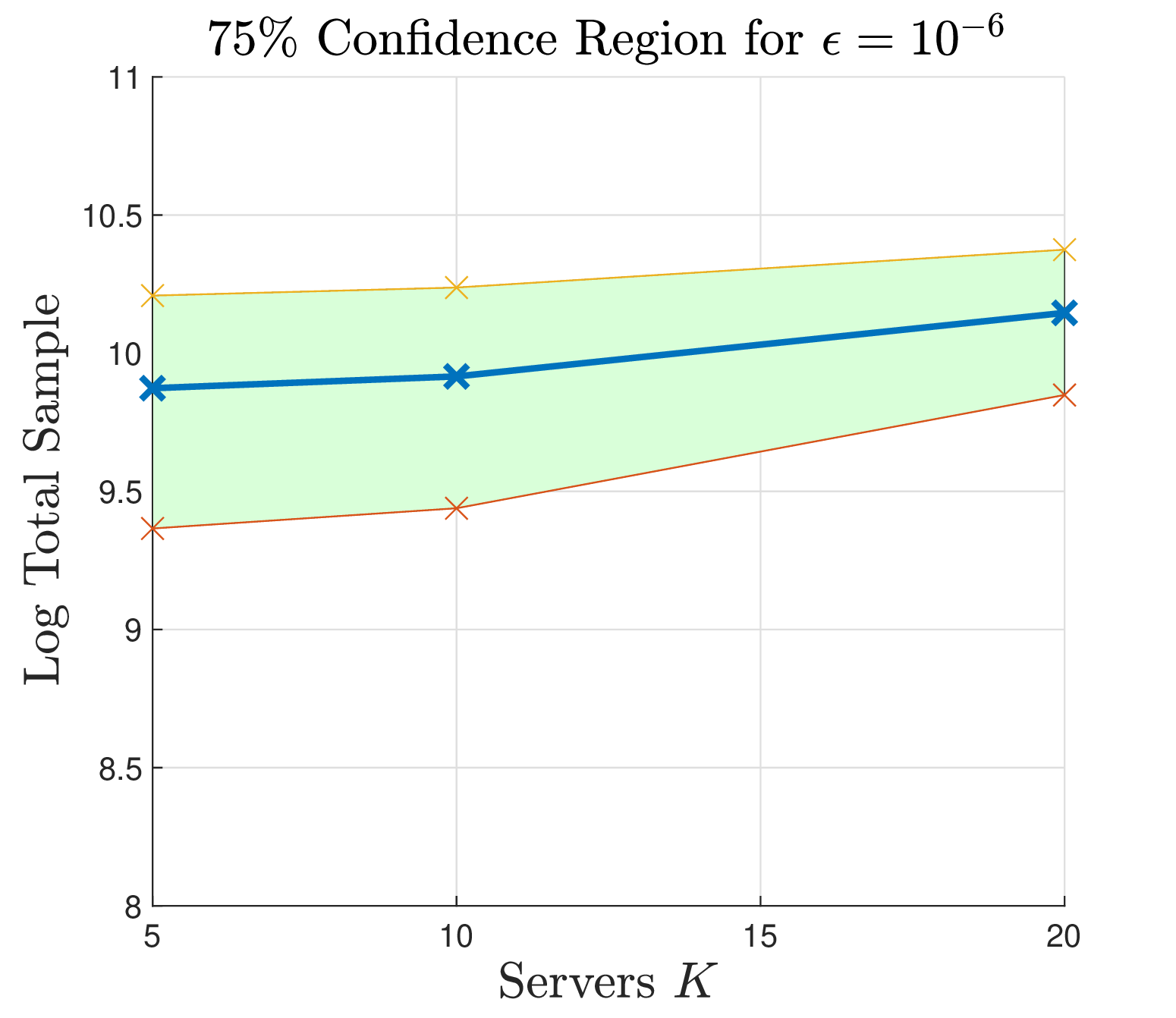}
   (c)
   \end{minipage}
    \caption{(a) Empirical averaged MSE$\| \bar x_t - x^*\|^2$ against total samples  for DSMO $K=5,10,20$ and DSGD $K = 5$.
    (b) Empirical averaged log-MSE $\log(\| \bar x_t - x^*\|^2)$ against log-iteration $\log(t)$ for DSMO $K=5,10,20$.
    (c) 75\% confidence region of log- total samples for achieving $\|\bar x_t - x^* \|^2 \leq \epsilon$, with varying network sizes $K = 5,10,20$. All figures are generated through 10 independent simulations. 
    }
    \label{fig:MDP1}
    \vskip -3mm
\end{figure}

\subsection{Federated Risk-averse Optimization} We consider the federated risk-averse optimization problem \eqref{prob:dist_riskaverse} over a decentralized network, which can be viewed as a decentralized stochastic three-level optimization problem. 
We consider the case where $K$ agents are connected over a ring network. Letting $\xi_i^k = (w_i^k, y_i^k)$ be a random feature-label pair accessible to agent $k$, we assume a linear model that $y_i = w_i^\top \tilde x + \epsilon_i$ where $\epsilon_i \sim \cN(0,0.2)$ and each entry in $\tilde x$ is independently generated such that $\tilde x_i \sim \text{Unif}[0,1]$. We consider a least-squared utility function $U^k(x,\xi_i^k)$ that 
$$
U^k(x,\xi_i^k) = -(y_i^k - x^\top w_i^k)^2. 
$$
Here, problem \eqref{prob:dist_riskaverse} is $\lambda$-strongly concave and we consider the case where  $p=2$ and $\lambda = 1$. 
We employ the DSMO algorithm~\ref{alg:1} ($M=3$) to solve this problem and conduct 10 independent simulations. In each simulation, we run our algorithm for $T = 25000$ rounds and adopt adaptive stepsizes such that $\alpha_t = \frac{2}{1+t}$, $\beta_t = \gamma_t = \frac{50}{50 +t }$ for all $t\leq T$. 
We generate a batch of data $\{ (w_i^k,y_i^k)\}$ of size $10^4$ and split them to each agent so that each agent only has  access to its own data in simulation. 
We test our algorithm over ring networks of different sizes that $K \in \{ 5,10,20\}$. For a benchmark comparison,  we  derive the optimal solution $x^*$ by solving the batch version of problem \eqref{prob:dist_riskaverse}. For each simulation, letting $\bar x_t = \frac{1}{K} \sum_{k=1}^K x_t^k$, we plot the averaged MSE  $\| \bar x_t - x^*\|^2$ against the iteration number $t$ in Figure~\ref{fig:risk_averse}~(a). We also report $\| \bar x_t - x^*\|^2$ against the number of total samples over all agents in Figure~\ref{fig:risk_averse}~(b). In addition, to study the empirical convergence rate of our algorithm, we plot the log-error $\log( \| \bar x_t - x^*\|^2 )$ against the log-iteration $\log t$ in Figure~\ref{fig:risk_averse}~(c), and provide a straight line of slope $-1$ for comparison with the benchmark.

From Figure~\ref{fig:risk_averse} (a), we observe that our algorithm generates a sequence converging to $x^*$ in all simulations, and it accelerates with the number of agents $K$ increasing. From Figure~\ref{fig:risk_averse} (b), it can be seen that to  obtain a solution of a certain level of accuracy, the required total number of samples are roughly the same among tested networks of different sizes. Further, Figure~\ref{fig:risk_averse} (c) suggests that slopes of $\log(\| \bar x_t - x^* \|^2)$ against $\log t$ are around $-1$ in all tested networks, which further implies that Algorithm~\ref{alg:1} enjoys a convergence rate of $\cO(1/T)$ for strongly convex SMO problems and matches Theorem~\ref{thm:PL} that Algorithm~\ref{alg:1} converges to the optimal solution at the rate of $\cO(\frac{1}{KT})$. The above numerical results demonstrate the practical efficiency of our SMO algorithm over networks of different sizes and validates our theoretical convergent rate results. 

\begin{figure}[t]
\begin{minipage}{0.32\textwidth}
    \centering
        \includegraphics[width=1\linewidth]{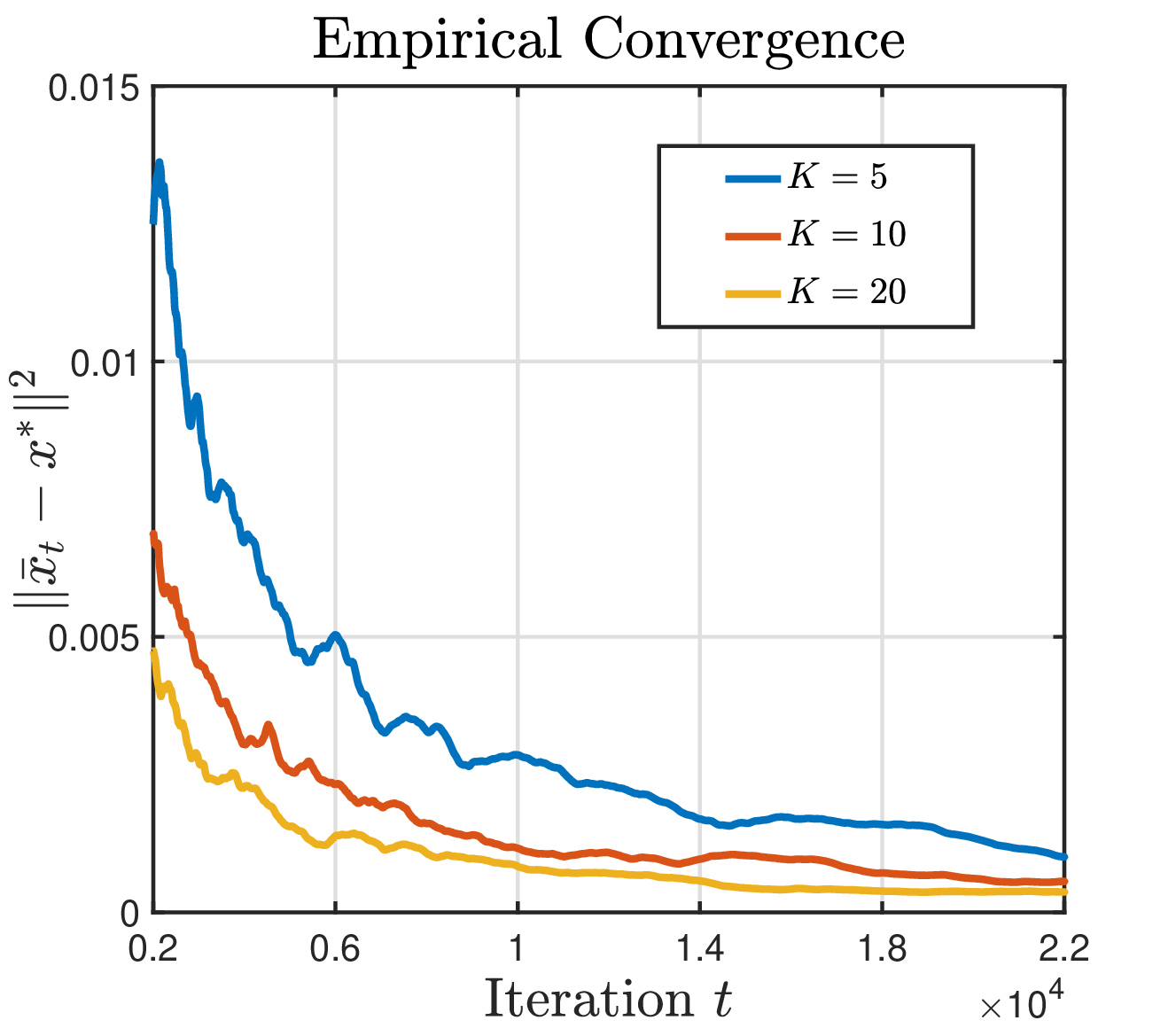}
        (a)
        \end{minipage}  
        \begin{minipage}{0.32\textwidth}
        \centering
   \includegraphics[width=1\linewidth]{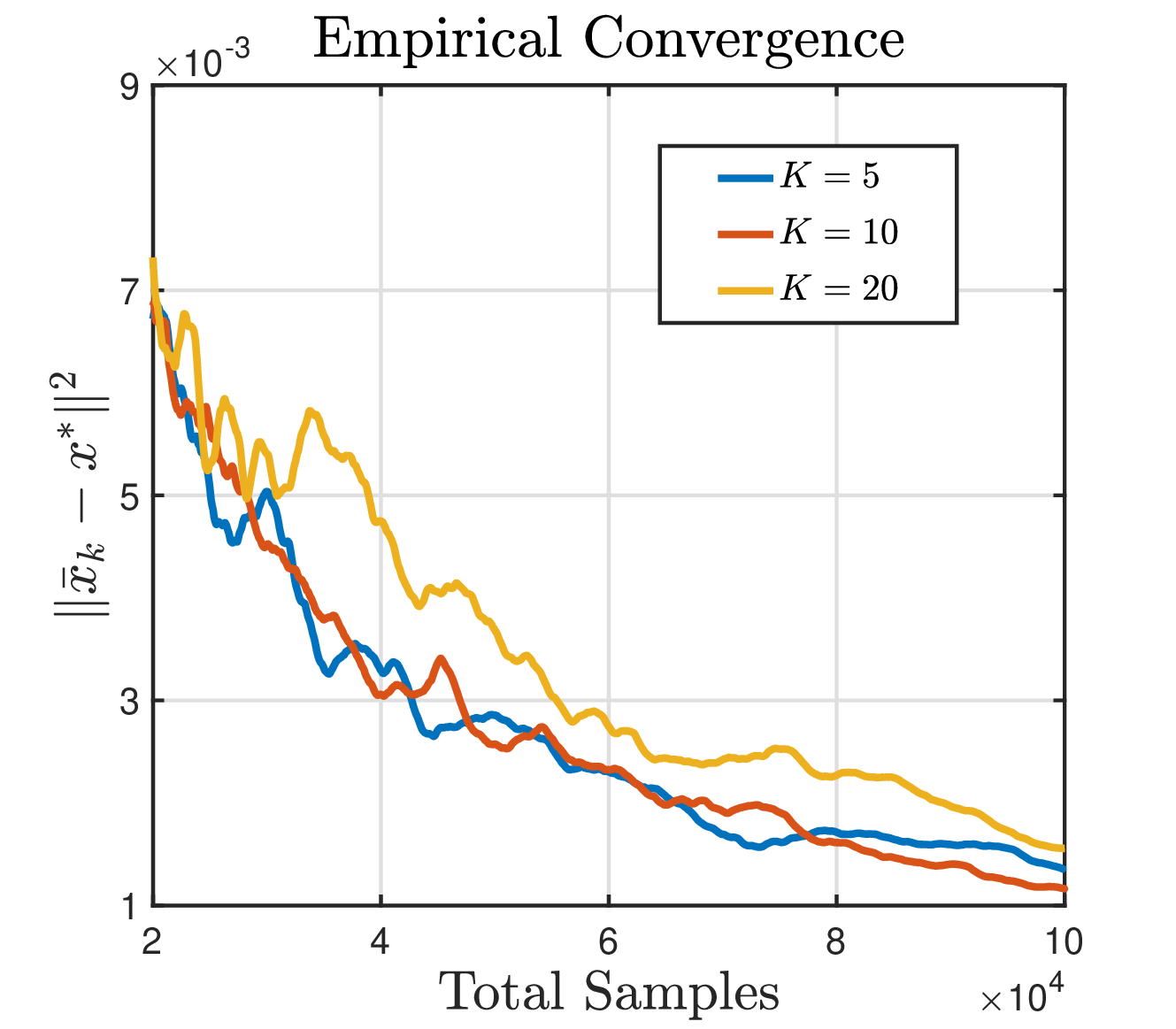} 
   (b)
   \end{minipage}
                  \begin{minipage}{0.32\textwidth}
        \centering
   \includegraphics[width=1\linewidth]{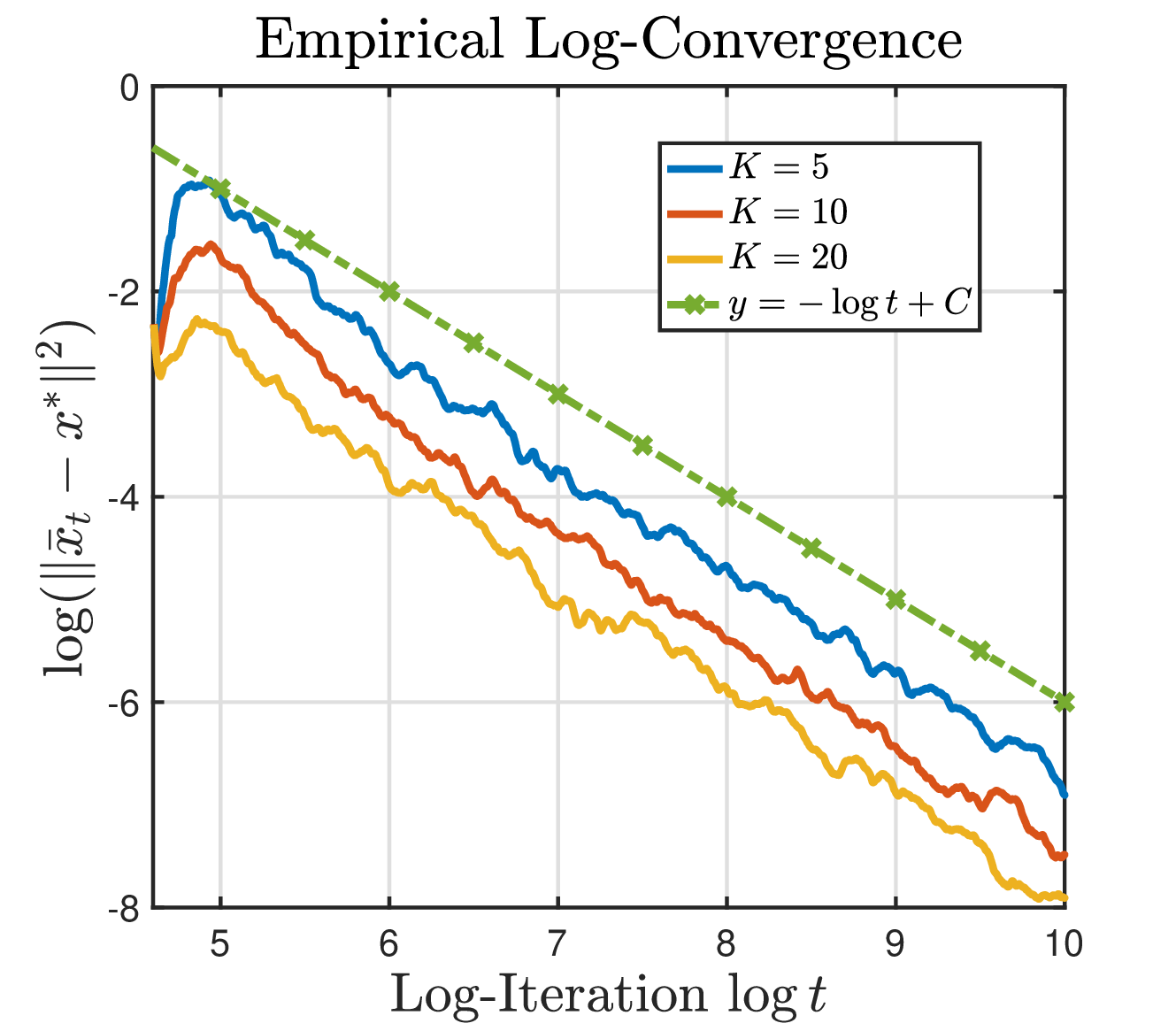}
   (c)
   \end{minipage}
    \caption{(a) Empirical averaged MSE$\| \bar x_t - x^*\|^2$ against the total number of samples  for DSMO $K=5,10,20$.
    (b) Empirical averaged MSE $\| \bar x_t - x^*\|^2$ against iteration  for DSMO $K=5,10,20$.
    (c) Empirical averaged log-MSE $\log(\| \bar x_t - x^*\|^2)$ against log-iteration $\log(t)$ for DSMO $K=5,10,20$.
     All figures are generated from 10 independent simulations. 
    }
    \label{fig:risk_averse}
    \vskip -3mm
\end{figure}

\section{Conclusion}
In this paper, we propose a novel formulation for decentralized stochastic multilevel optimization. We develop a gossip-based stochastic approximation scheme to solve this problem in various settings. We show that our proposed algorithm finds a stationary point at a  rate of $\cO(\tfrac{1}{\sqrt{KT}})$ for nonconvex objectives, and converges to the optimal solution at a  rate of $\cO(\tfrac{1}{KT})$ for PL objectives, regardless of the number of levels $M$ and network topology. Numerical experiments on hyper-parameter optimization, multi-agent federated MDP, and federated risk-averse optimization demonstrate the practical efficiency of our algorithm, exhibit the effect of speed-up in a decentralized setting, and validate our theoretical claims.
In future work, we wish to develop algorithms that achieve lower iteration complexities and enjoy lower communication costs.

\section*{Acknowledgement}
Mengdi Wang acknowledges support by NSF grants DMS-1953686, IIS-2107304, CMMI-1653435, and ONR grant 1006977.
\bibliographystyle{plain}
\bibliography{SCGD,bib}
\newpage
\appendix
\section*{Appendix}

\section*{Outline}

\begin{itemize} 
    \item Section \ref{app:A}:  Notation, assumptions, and supporting lemmas
    \begin{itemize}
        \item Subsection~\ref{app:detailed_assumption} Detailed assumptions
        \item Subsection~\ref{app:lemmas} Technical lemmas for Lipschitz properties and Hessian inverse estimation
    \end{itemize}
    \item Section \ref{app:nonconvex}: Proof of results for nonconvex objectives 
    \begin{itemize} 
        \item Subsection~\ref{app:multilevel_supporting}: Technical Lemmas for  consensus and estimation errors
    \item Subsection~\ref{app:proof_nonconvex_multilevel} Proof of Theorem \ref{thm:nonconvex_multilevel}
    \end{itemize}
    \item Section~\ref{app:proof_PL_multilevel}: Proof of results for $\mu$-PL objectives
    \begin{itemize}
    \item Subsection~\ref{app:main_PL}: Technical Lemma for convergence properties of $\mu$-PL functions
        \item Subsection~\ref{app:proof_of_thm_PL}: Proof of Theorem~\ref{thm:PL}
    \end{itemize}

    \item Section~\ref{app:numerics}: Additional numerical details. 

\end{itemize}

\section{Notation, Detailed Assumptions, and Technical Lemmas} \label{app:A}
For notational convenience, we denote by $\bar u_{m,t} = \frac{1}{K} \sum_{k=1}^K u_{m,t}^k$ the averaged estimates of $\nabla_{12}^2 g_m(y_{m-1,t}^\star, y_{m,t}^\star)$   over $u_{m,t}^k$'s within the network.  We denote by $\bar v_{m,t,j} = \frac{1}{K} \sum_{k=1}^K v_{m,t,j}^k \in \RR^{d_y \times d_y}$, $ \bar s_t = \frac{1}{K} \sum_{k=1}^K s_t^k$,  and $ \bar h_t = \frac{1}{K} \sum_{k=1}^K h_t^k$.   We denote by $\bar z_k = \sum_{k =1}^K z_t^k$. We denote by $\| a\| = \| a\|_2$ for a vector $a$ and denote by $\| A\| = \| A\|_2 = \sigma_{\max}(A)$ for a matrix $A$. We denote by $\|A \|_F $ the Frobenius norm for a matrix $A$ and  denote by $\lv A, B \rv_F = \sum_{i,j} A_{ij} B_{ij}$ the Frobenius inner product for two matrices $A$ and $B$. 
For any $\bar x_t \in \RR^{d_x}$, we denote by $y_{m,t}^\star = y_m^\star(\bar x_t)$. 
For notational convenience, we drop the sub-scripts $\xi^k,\zeta^k$ within the expectations $\EE_{\xi^k}[\cdot]$ and $\EE_{\zeta^k}[\cdot]$.

\subsection{Detailed Assumptions}\label{app:detailed_assumption}
We denote by $L_{q,m}  = \frac{1}{\kappa_{g,m} L_{g,m}} \geq  \mu_{g,m}^{-1}$ and observe that $\|[\nabla_{22}^2g_m(y_{m-1},y_m) ]^{-1}\|_2^2 \leq \mu_{g,m}^{-2} \leq L_{q,m}^2$  for all  $y_{m-1} \in \RR^{d_{m-1}}$ and $y_m \in \RR^{d_m} $. For notational convenience, we use $\sigma_{g,m}, \sigma_f >0$ to represent the upper bounds of standard deviations such that 
\begin{align*}
\EE [ \| \nabla_1 f^k(x,y_M;\zeta^k) - \nabla_1 f^k(x,y_M) \|^2 ] \leq \sigma_f^2, \EE [ \| \nabla_2 f^k(x,y_M;\zeta^k) - \nabla_2 f^k(x,y_M) \|^2] \leq \sigma_f^2,
\end{align*}
and 
\begin{align*}
    & \EE [ \| \nabla_2 g_m^k(y_{m-1},y_m;\xi_m^k) - \nabla_y g^k(y_{m-1},y_m) \|^2 ]  \leq \sigma_{g,m}^2,  \\
     & \EE [ \| \nabla_{12}^2 g_m^k(y_{m-1},y_m ;\xi_m^k) - \nabla_{12}^2 g_m^k(y_{m-1},y_m) \|_F^2 ]\leq \sigma_{g,m}^2, 
     \\
     &  \EE [ \| \nabla_{22}^2 g_m^k(y_{m-1},y_m;\xi_m^k) - \nabla_{22}^2 g_m^k(y_{m-1},y_m ) \|_F^2 ] \leq \sigma_{g,m}^2. 
\end{align*}
We also 
adopt constants $L_F, L_y >0$ to quantify the Lipschitz properties, specified in Section~\ref{app:lemmas}. Given $(x,y_M)$, we use  $\nabla_1 f^k(x,y_M; \zeta^k )$, $\nabla_2 f^k(x,y_M; \zeta^k )$, $\nabla_2 g_m^k(y_{m-1}, y_m; \xi^k )$, \\
$\nabla_{12}^2 g_m^k( y_{m-1}, y_m; \xi_m^k)$, and $\nabla_{22}^2 g_m^k( y_{m-1}, y_m; \xi_m^k) $ to represent the independent stochastic information sampled in round $t$ by agent $k$.  Such independent samples can be obtained by independently querying the $\cS \cO$  three times.

\subsection{Technical Lemmas for Lipschitz Properties and Hessian Inverse Estimation}\label{app:lemmas}

\begin{lemma}\label{lemma:inverse_diff}
 Let $A $ be a positive definite matrix such that $ \delta {\bf I}  \succeq A \succ \bf{0}$ for some $0 < \delta < 1$, and $A_1,\cdots, A_k$ be $k$ matrices such that 
 $\EE[ \| \prod_{i=j}^k A_i\|_2^2] \leq \delta^{2(k-j+1)}$ for $1\leq j \leq k$. Let $   Q_k = {\bf I} + A_k +A_{k-1}  A_k  + \cdots + A_1 A_2 \cdots A_k$, then the following holds. 
 \begin{equation*}
    \begin{split}
     \EE [ \|  ({\bf I}-A)^{-1}   - Q_k \|_2^2 ] 
  \leq \frac{1}{1-\delta} \Big ( \sum_{1 \leq j \leq k+1 }  \delta^{k-j} \EE[ \| A_j - A \|_2^2]  \Big ) + \frac{ \delta^{k+1} }{(1-\delta)^3}.
    \end{split} 
\end{equation*}
\end{lemma}
\begin{proof}
Recall that for any positive definite matrix $A$ such that $\delta I  \succ A \succeq 0 $ for some $0<\delta <1$, we have 
\begin{equation*}
    ({\bf I}- A)^{-1} = \sum_{i=0}^\infty A^i = {\bf I} + A + A^2 + \cdots  .
\end{equation*}
Letting $Q = ({\bf I}- A)^{-1}$, we have 
\begin{equation*}
    Q = {\bf I} + AQ \text{ and } \|Q \|_2^2  = \| ({\bf I}- A)^{-1} \|_2^2 \leq \frac{1}{(1-\delta)^2}.
\end{equation*}
We define an auxiliary sequence $Q_1 = {\bf I} + A_1$, $Q_2 = {\bf I} + A_2  Q_1,\cdots$, and  $Q_{k-1} ={\bf I} + A_{k-1} Q_{k-2}$,  and note that $Q_k = {\bf I} + A_k Q_{k-1}$. Consider $ ({\bf I}-A)^{-1} - Q_k$, by utilizing the above  sequence, we obtain 
\begin{equation*}
\begin{split}
   ({\bf I}-A)^{-1} - Q_k & = A Q  - A_k Q_{k-1} = (A- A_k) Q + A_k (Q - Q_{k-1})  + A^{k+1}Q.
   \end{split}
\end{equation*}
We note that $Q - Q_{k-1} = AQ -A_{k-1}Q_{k-2} = (A- A_{k-1})Q + A_{k-1}(Q - Q_{k-2})$. 
By using such an induction relationship, we can quantify the estimation error by 
\begin{equation*}
\begin{split}
  \|  ({\bf I}-A)^{-1} - Q_k \|_2  
   & \leq \| A- A_k \|_2 \| Q \|_2 + \| A_k \|_2  \| A - A_{k-1} \|_2 \| Q\|_2  + \cdots 
   \\
   & \quad + \| A_k A_{k-1} \cdots A_2 \|_2 \| A - A_1 \|_2  \| Q \|_2  + \|  A^{k+1}Q \|_2.
   \end{split}
\end{equation*}
 Letting $a_i = \|A - A_i \|_2$ and $b_i = \|  A_{i+1} \cdots A_k\|_2 \| Q\|_2$, and taking expectations on both sides of the above inequality, we obtain $\EE[ \|b_i \|_2^2] \leq \delta^{2(k-i)} \| Q \|_2$ and $\EE[\| A^{k+1}\|_2^2] \leq \delta^{2(k+1)}$. By using the fact that $\|AB \|_2 \leq \| A\|_2 \| B\|_2 \leq \frac{\|A \|_2^2}{2} + \frac{\|B \|_2^2}{2}$, we  further have that 
\begin{equation*}
    \begin{split}
      & \EE [ \|  ({\bf I}-A)^{-1}   - Z_k \|_2^2 ] 
       \\
     & \leq \sum_{i=1}^k \EE[ \|  b_i a_i \|_2^2 ] +  \sum_{1 \leq i < j\leq k} 2\EE[ \| b_i  b_j a_i a_j \|_2]
 +   \sum_{1\leq i \leq k} 2 \EE [ \| a_ib_i A^{k+1}Q \|_2]  +  \EE [ \|  A^{k+1}Q  \|_2^2 ]
      \\
      & \leq  \sum_{i=1}^k \EE[ \|  b_i \|_2^2] \EE[ \| a_i \|_2^2 ] +  \sum_{1 \leq i < j\leq k} \EE[ \| b_i  b_j \|_2 ] \EE [ \| a_i \|_2^2 + \| a_j \|^2  ] 
      \\
      & \quad +   \sum_{1\leq i \leq k} 2 \delta^{2k+1-i} \EE[ \|  Q \|_2  \| a_i\|_2 ]
      +  \delta^{2(k+1)} \| Q  \|_2^2 
      \\
      & \leq \sum_{i=1}^k  \delta^{2(k - i)} \EE[\| a_i\|_2^2] +  \sum_{1 \leq i < j\leq k} \delta^{2k - i-j} \EE [ \| a_i \|_2^2 + \| a_j \|_2^2  ] 
      \\
      & \quad +   \sum_{1\leq i \leq k}  \delta^{2k+1-i} \EE[ \|  Q \|_2^2 +  \| a_i\|_2^2 ] + \delta^{2(k+1)} \| Q\|_2^2
      \\
      & =( \sum_{1  \leq  i\leq k } \delta^{k-i} + \delta^{k+1})  \Big ( \sum_{1 \leq j \leq k }  \delta^{k-j} \EE[ \| a_j\|_2^2]  \Big ) + ( \delta^{2(k+1)} +  \sum_{1\leq i \leq k}  \delta^{2k+1-i} )  \| Q \|_2^2 
      \\
      & \leq \frac{1}{1-\delta}\Big ( \sum_{1 \leq j \leq k }  \delta^{k-j} \EE[ \| a_j\|_2^2]  \Big ) + \frac{ \delta^{k+1}}{1-\delta}\| Q \|_2^2,
    \end{split} 
\end{equation*}
where the last inequality uses the fact that  $\sum_{i=0}^\infty \delta^i =\frac{1}{1-\delta}$. 
The desired inequality can be acquired by using $\| Q \|_2^2 \leq \frac{1}{(1-\delta)^2}$.
\end{proof}

We provide the following result to characterize the estimation error $\| [\nabla_{22}^2 g_m( y_{m-1,t}^\star,  y_{m,t}^\star) ]^{-1} - q_{m,t}^k\|_2^2$ induced by Algorithm~\ref{alg:1}.
\begin{lemma} \label{lemma:inverse_Hessian}
Suppose Assumptions~\ref{assumption:SO}, \ref{assumption:W}, \ref{assumption:1}, and \ref{assumption:2} hold, for each $m=1,\cdots, M$, we have 
  \begin{equation*}
  \begin{split}
   & \EE[ \| [\nabla_{22}^2 g_m(y_{m-1,t}^\star, y_{m,t}^\star ) ]^{-1} - q_{m,t}^k\|_2^2 ] 
    \\
    &\leq \frac{1}{L_{g,m}^4 \kappa_{g,m}  }   \Big ( \sum_{1 \leq j \leq  b}  (1-\kappa_{g,m})^{b - j } \EE[ \| \nabla_{22}^2 g_m(y_{m-1,t}^\star, y_{m,t}^\star  ) -  v_{m,t,i}^k \|_F^2]  \Big ) + \frac{(1-\kappa_{g,m})^{b+1}}{ L_{g,m}^2 \kappa_{g,m}^3}.
    \end{split}
\end{equation*}  
\end{lemma}
\textit{Proof:}
Recall that each $v_{m,t,i}^k $ is the convex combination of $\mu_{g,m}\bf{I}$ and sampled Hessian $\nabla_{22}^2 g_m^k(x,y;\xi^k)$, under Assumption~\ref{assumption:2} (v) that $\EE[\| {\bf I}  - \frac{1}{L_g} \nabla_{22}^2 g^k(x,y;\xi^k) \|_2^2]\leq (1- \kappa_{g,m})^2$, we have  $ \EE[ \| {\bf I}  - \frac{1}{L_{g,m}}  v_{m,t,i}^k \|_2^2] \leq (1- \kappa_{g,m} )^2$.
By applying Lemma~\ref{lemma:inverse_diff} with $A = {\bf I} - \frac{1}{L_{g,m}}  \nabla_{22}^2 g_m(y_{m-1,t}^\star, y_{m,t}^\star )$, $A_i = {\bf I} - \frac{v_{m,t,i}^k }{L_{g,m} } $, and $\delta = 1-\kappa_{g,m}$, we obtain that 
 \begin{equation*}
    \begin{split}
   &   \EE [ \|  L_{g,m} [\nabla_{22}^2 g_m(y_{m-1,t}^\star, y_{m,t}^\star)  ]^{-1} - Q_{m,t,b}^k \|_2^2 ] 
     \\
  & \leq \frac{1}{\kappa_{g,m} } \Big ( \sum_{1 \leq j \leq b}  \frac{ (1- \kappa_{g,m} )^{b-j}  }{L_{g,m}^2}\EE[ \|  \nabla_{22}^2 g_m(y_{m-1,t}^\star, y_{m,t}^\star )  - v_{m,t,j}^k  \|_2^2]  \Big ) + \frac{(1- \kappa_g)^{b+1} }{\kappa_g^3}
  \\
  & \leq \frac{1}{\kappa_{g,m} } \Big ( \sum_{1 \leq j \leq b }  \frac{ (1- \kappa_{g,m})^{b-j}  }{L_{g,m}^2}\EE[ \|  \nabla_{22}^2 g_m(y_{m-1,t}^\star, y_{m,t}^\star )  - v_{m,t,j}^k  \|_F^2]  \Big ) + \frac{ (1- \kappa_{g,m})^{b+1} }{\kappa_{g,m}^3}.
    \end{split} 
\end{equation*}
We obtain the desired result by dividing both sides of the above inequality  by $L_g^2$ and using the fact that $q_{m,t}^k = Q_{m,t,b}^k/L_{g,m}$.

\QED


\section{Proof of Results for Nonconvex Objectives}\label{app:nonconvex}
For notational convenience, we drop the sub-scripts $\xi_m^k,\zeta^k$ within the expectations $\EE_{\xi_m^k}[\cdot]$ and $\EE_{\zeta^k}[\cdot]$. 
We denote by $L_{q,m}  = \frac{1}{\kappa_{g,m} L_{g,m}} \geq  \mu_{g,m}^{-1}$ and observe that $\|[\nabla_{22}^2g_m(x_m,y_m) ]^{-1}\|_2^2 \leq \mu_{g,m}^{-2} \leq L_{q,m}^2$  for all  $x_m \in \RR^{d_{m-1}}$ and $y_m \in \RR^{d_m} $. We use $\sigma_{g,m}, \sigma_f >0$ to represent the upper bounds of standard deviations such that 
\begin{align*}
\EE [ \| \nabla_1 f^k(x,y_M;\zeta^k) - \nabla_1 f^k(x,y_M) \|^2 ] \leq \sigma_f^2, \EE [ \| \nabla_2 f^k(x,y_M;\zeta^k) - \nabla_2 f^k(x,y_M) \|^2] \leq \sigma_f^2,
\end{align*}
and for each $m=1,\cdots,M$, 
\begin{align*}
   &  \EE [ \| \nabla_2 g_m^k(y_{m-1},y_m;\xi_m^k) - \nabla_2 g_m^k(y_{m-1},y_m) \|^2 ]  \leq \sigma_{g,m}^2,  \\
     & \EE [ \| \nabla_{12}^2 g^k(y_{m-1},y_m;\xi_m^k) - \nabla_{12}^2 g^k(y_{m-1},y_m) \|_F^2 ]\leq \sigma_{g,m}^2, 
     \\
 \text{ and }    &   \EE [ \| \nabla_{22}^2 g_m^k(y_{m-1},y_m;\xi^k) - \nabla_{22}^2 g^k(y_{m-1},y_m) \|_F^2 ] \leq \sigma_{g,m}^2. 
\end{align*}
\subsection{Supporting Lemmas}\label{app:multilevel_supporting}
We first show that the optimal solution $y_m^\star(x)$ for level $m$ is Lipschitz continuous in $x$ as follows. 
\begin{lemma}\label{lemma:Lip_y_multilevel}
 Suppose Assumptions~\ref{assumption:SO}, \ref{assumption:W}, \ref{assumption:1}, and \ref{assumption:2} hold. Then for all $x,x'\in \RR^{d_{x}}$, we have 
 $$
 \| \nabla F(x) - \nabla  F(x') \| \leq  L_F \| x - x'\|  \text{and } \| y_m^\star(x) -  y_m^\star(x') \| \leq L_{y,m} \| x - x'\|, \forall m \in [M], 
 $$
 where $L_{y,m} = \prod_{j=1}^m \frac{L_{g,j}}{\mu_{g,j}}$ and $L_F>0$ is a positive constant. 
\end{lemma}
\begin{proof}
Recall \eqref{eq:grad_y_multilevel}, for any $m=1,\cdots, M$, we have that 
\begin{equation*}
\begin{split}
\| \nabla y_m^\star(x) \|  &  \leq   \|  \nabla y_{m-1}^\star(x) \| \|  \nabla_{12}^2 g_m(y_{m-1}^\star(x) ,  y_{m}^\star(x)) \|  \| [ \nabla_{22}^2 g_m(y_{m-1}^\star(x) ,  y_{m}^\star(x))]^{-1} \| 
\\
& \leq \frac{L_{g,m}}{\mu_{g,m}} \|  \nabla y_{m-1}^\star(x) \|. 
\end{split}
\end{equation*}
By recursively applying this inequality, we conclude  that $y_m^\star(x)$ is $L_{y,m}$-Lipschitz continuous in $x$ with  $L_{y,m} = \prod_{j=1}^m \frac{L_{g,j}}{\mu_{g,j}}$. The Lipschitz continuity of $\nabla F(x)$ can be obtained by combining the Lipschitz continuity of $y_m^\star(x)$ and the definition of $\nabla F(x)$ \eqref{eq:gradient_multilevel}.
\end{proof}

Next, we provide a few fundamental Lemmas \ref{lemma:inverse_Hessian}, \ref{lemma:boundedness_multilevel}, and \ref{lemma:consensus_multilevel} for the decentralized stochastic multilevel optimization problem \eqref{prob:1} . 
\begin{lemma} \label{lemma:inverse_Hessian}
 Suppose Assumptions~\ref{assumption:SO}, \ref{assumption:W}, \ref{assumption:1}, and \ref{assumption:2} hold, then we have for all $m=1,\cdots,M$,
  \begin{equation*}
  \begin{split}
  &  \EE[ \| [\nabla_{22}^2 g_m(y_{m-1,t}^\star,  y_{m,t}^\star) ]^{-1} - q_{m,t}^k\|_2^2 ] 
    \\
    & \leq \frac{1}{L_{g,m}^4 \kappa_{g,m}  }   \Big ( \sum_{1 \leq j \leq  b}  (1-\kappa_{g,m})^{b - j } \EE[ \| \nabla_{22}^2 g_m(y_{m-1,t}^\star,  y_{m,t}^\star ) -  v_{m,t,i}^k \|_F^2]  \Big ) + \frac{(1-\kappa_{g,m})^{b+1}}{ L_{g,m}^2 \kappa_{g,m}^3}.
    \end{split}
\end{equation*}  
\end{lemma}
 \begin{lemma}\label{lemma:boundedness_multilevel}
Suppose Assumptions~\ref{assumption:SO}, \ref{assumption:W}, \ref{assumption:1}, , and \ref{assumption:2} hold, then we have
\begin{equation*}
\begin{split}
& \EE[ \| s_t^k\|^2] \leq C_f^2, \ \EE[\| h_{t}^k\|^2] \leq C_f^2, \EE[\| z_t^k\|^2]  \leq 2 C_f^2 + 2C_f^2 \prod_{m=1}^M \Big [ L_{q,m}^2 L_{g,m}^2 \Big ], 
 \end{split}
\end{equation*}
and for $m=1,\cdots, M$, 
\begin{equation*}
\begin{split}
\EE[ \| u_{m,t}^k \|^2] \leq  L_{g,m}^2,  \EE[ \| q_{m,t}^k \|^2] \leq L_{q,m}^2, \EE[ \| v_{m,t,j}^k \|_F^2] \leq  L_{g,m}^2, \forall 1\leq j \leq b.
 \end{split}
\end{equation*}
\end{lemma}
\begin{proof}
We first observe that  $s_{t}^k, h_{t}^k,  u_{t}^k$, $v_{t,j}^k$ are convex combinations of past sampled stochastic information $ \nabla_1 f^k(x_t^k, y_t^k ; \zeta_t^k )$, $ \nabla_2 f^k(x_t^k, y_t^k ; \zeta_t^k )$, $ \nabla_{12}^2 g_m^k(y_{m-1,t}^k, y_{m,t}^k ; \xi_t^k )$,  $ \nabla_{22}^2 g_{m}^k(y_{m-1,t}^k, y_{m,t}^k ; \xi_t^k )$, respectively. Therefore, under Assumption \ref{assumption:1}, for all $t \leq T$, for all $ 1\leq j \leq b$, we have 
$$
\EE[ \| u_{m,t}^k \|^2] \leq  L_{g,m}^2, \ \EE[ \|  v_{m,t,j}^k \|^2] \leq  L_{g,m}^2,  \EE[ \| s_t^k\|^2] \leq C_f^2, \text{ and } \EE[\| h_{t}^k\|^2] \leq C_f^2.
$$
Recall that $q_{m,t}^k = \frac{1}{L_{g,m}}\sum_{i=0}^b \prod_{j=1}^i (  I - \tfrac{v_{t,j}^k}{L_{g,m} }  ) $, we further obtain that 
\begin{equation}
\begin{split}
 \EE[ \| q_{m,t}^k \|^2] & =  \frac{1}{L_{g,m}^2} \sum_{ 0 \leq i \leq b} \EE \left (  \prod_{j=1}^i(  I - \tfrac{v_{m,t,j}^k}{L_{g,m} }  )  \cdot  \sum_{ 0 \leq s \leq b} \prod_{j=1}^s(  I - \tfrac{v_{m,t,j}^k}{L_{g,m} }  )  \right ) 
\\
& \leq   \frac{1}{L_{g,m}^2} \sum_{ 0 \leq i \leq b}  \frac{ (1-\kappa_{g,m})^{i} }{\kappa_{g,m}}   \leq \frac{1}{\kappa_{g,m}^2L_{g,m}^2} = L_{g,m}^2. 
\end{split}
\end{equation}
By using the conditional independence of the sampled stochastic information, we have $\EE[ \| u_{1,t}^k q_{1,t}^k  \cdots u_{M,t}^k q_{M,t}^k  h_t^k   \|^2] \leq  C_f^2 \prod_{m=1}^M \Big [ L_{q,m}^2 L_{g,m}^2 \Big ]  $, further implying that 
$$
\EE[\| z_t^k\|^2] = \EE[ \| s_t^k + (-1)^M u_{1,t}^k q_{1,t}^k  \cdots u_{M,t}^k q_{M,t}^k  h_t^k   \|^2 ] \leq 2 C_f^2 + 2C_f^2 \prod_{m=1}^M \Big [ L_{q,m}^2 L_{g,m}^2 \Big ] . 
$$
This completes the proof. 
\end{proof}

\begin{lemma}\label{lemma:consensus_multilevel}
Suppose Assumptions~\ref{assumption:SO}, \ref{assumption:W}, \ref{assumption:1}, , and \ref{assumption:2}  hold and   the step-sizes satisfy $\beta_t\leq 1$ and one of the followings:
\begin{enumerate}
    \item[(i)] $\alpha_t = \alpha_0$, $\beta_t = \beta_0$, and $\gamma_t = \gamma_0$, for $0 \leq t \leq T$.
    \item[(ii)] $\lim_{t\to \infty} (\alpha_t + \beta_t + \gamma_t ) = 0$, $\lim_{t \to \infty} \frac{\alpha_{t-1}}{\alpha_t} = 1$, $\lim_{t \to \infty} \frac{\beta_{t-1}}{\beta_t} = 1$, and  $\lim_{t \to \infty} \frac{\gamma_{t-1}}{\gamma_t} = 1$. 
\end{enumerate}
Then we have for  all $1 \leq j \leq b$, 
\begin{equation*}
\begin{split}
 & \sum_{ k \in \cK }  \EE [ \| x_t^k - \bar x_t    \|^2  ] \leq \cO \left ( \frac{ K \alpha_t^2}{(1-\rho)^2} \right ),  \ 
  \sum_{ k \in \cK }\EE[ \| y_{m,t}^k - \bar y_{m,t}\|^2] \leq \cO \left ( \frac{ K \gamma_t^2}{(1-\rho)^2} \right ), 
  \\
&\sum_{  k \in \cK }  \EE \left [ \| s_t^k - \bar s_t\|^2    +   \| h_t^k - \bar h_t\|_F^2   \right ]\leq \cO \left ( \frac{ K  \beta_t^2}{(1-\rho)^2} \right ),
 \\
  \text{ and }   &\sum_{  k \in \cK }  \EE \left [    \| u_{m,t}^k - \bar u_{m,t} \|^2  +  \| v_{m,t,j}^k - \bar v_{m,t,j}\|_F^2  \right ]\leq \cO \left ( \frac{ K  \beta_t^2}{(1-\rho)^2} \right ), \ \  \forall m=1,\cdots, M. 
\end{split}
\end{equation*}
\end{lemma}
\begin{proof}
Recall the update rule that 
 \begin{equation*}
 X_{t+1} = X_{t}W - \alpha_{t} Z_{t} \text{ and } \bar X_{t+1} =\bar X_{t} - \alpha_{t} \bar Z_{t},
 \end{equation*}
 by using the fact that $\bar X_t = X_t W^\infty$, 
 we have 
 \begin{equation*}
 X_{t+1} - \bar X_{t+1} =  X_{t}( W -W^\infty )  - \alpha_t (Z_t - \bar Z_t). 
 \end{equation*}
Under Assumption~\ref{assumption:W} that  $\| W - W^\infty \|_2 =   \sqrt{\rho}$, we have 
 \begin{equation*}
 \begin{split}
\| X_t  W- \bar X_t \|_F  & = \| (X_t - \bar X_t)(W - W^\infty)\|_F =   \| (W - W^\infty)^\top(X_t - \bar X_t)^\top\|_F 
\\
& \leq \| W - W^\infty\|_2 \| X_t - \bar X_t\|_F \leq \sqrt{\rho} \| X_t - \bar X_t\|_F,
\end{split}
\end{equation*}
where the first equality uses the fact that $\bar X_t W = \bar X_t = X_t W^\infty$. 
Consequently, by using the fact that $\| A + B\|_F^2 \leq (1+\eta) \| A\|_F^2 + (1 + \frac{1}{\eta}) \| B\|_F^2$ for $\eta > 0$, we have 
 \begin{equation*}
 \begin{split}
 \|  X_{t+1} - \bar X_{t+1}  \|_F^2 & \leq (1+ \eta)  \rho \| X_{t} - \bar X_{t} \|_F^2 + (1 + \tfrac{1}{\eta})  \alpha_t^2  \| Z_t - \bar Z_t\|_F^2.
 \end{split}
 \end{equation*}
By setting $\eta = \frac{1-\rho}{2\rho}$, we obtain
 \begin{equation*}
 \|  X_{t+1} - \bar X_{t+1}  \|_F^2 \leq  \frac{1+\rho}{2} \| X_{t} - \bar X_{t} \|_F^2 + \frac{(1+\rho) \alpha_t^2}{1-\rho}  \| Z_t - \bar Z_t\|_F^2.
 \end{equation*}
 Taking expectations on both sides of the above inequality and using Lemma~\ref{lemma:boundedness_multilevel}  that $\| Z_t - \bar Z_t\|_F^2 \leq 4K C_z $ where $C_z = C_f^2 + 2C_f^2 \prod_{m=1}^M  L_{q,m}^2 L_{g,m}^2  $ and $1+\rho \leq 2$, we further have 
  \begin{equation*}
  \begin{split}
 \EE[ \|  X_{t+1} - \bar X_{t+1}  \|_F^2] &  \leq  \frac{1+\rho}{2} \EE[ \| X_{t} - \bar X_{t} \|_F^2 ] + \frac{2 \alpha_t^2}{1-\rho}\EE[ \| Z_t - \bar Z_t\|_F^2 ]
 \\
 & \leq  \frac{1+\rho}{2} \EE[ \| X_{t} - \bar X_{t} \|_F^2 ] + \frac{8 \alpha_t^2 K C_z  }{1-\rho}.
 \end{split}
 \end{equation*}
 We then use an induction argument to prove the result. Suppose $\EE[ \| X_{t} - \bar X_{t} \|_F^2 ]  \leq \hat C K \alpha_{t-1}^2$, then we have 
   \begin{equation*}
  \begin{split}
 \EE[ \|  X_{t+1} - \bar X_{t+1}  \|_F^2] & \leq  \frac{(1+\rho) \hat C K \alpha_{t-1}^2}{2}  + \frac{8 \alpha_t^2 K C_z  }{1-\rho} 
 \\
 & = \alpha_t^2 \Big ( \frac{(1+\rho) \hat C  K \alpha_{t-1}^2 }{2 \alpha_t^2}   +  \frac{8  K C_z }{1-\rho}  \Big ) .
 \end{split}
 \end{equation*}
 We observe that $ \frac{(1+\rho)   \alpha_{t-1}^2 }{2 \alpha_t^2} =  \frac{1+\rho}{2}$ under condition (i). Under condition (ii) where $\lim_{t\to\infty} \frac{\alpha_{t-1}}{\alpha_t} = 1$, we can see that  $ \frac{(1+\rho)   \alpha_{t-1}^2 }{2 \alpha_t^2} \leq \frac{3+ \rho}{4}$  for $t$ sufficiently large.  Combining both scenarios, we observe that $ \Big ( \frac{(1+\rho) \hat C   \alpha_{t-1}^2 }{2 \alpha_t^2}   +  \frac{8   C_z }{1-\rho}  \Big ) \leq \hat C$ for $\hat C =   \frac{32  C_z }{(1-\rho)^2}   $. We then obtain 
 \begin{equation*}
     \sum_{ k \in \cK }  \EE [ \| x_t^k - \bar x_t    \|^2  ] \leq \cO \left ( \frac{ K \alpha_t^2}{(1-\rho)^2} \right ). 
 \end{equation*}
 The analysis for   $  \sum_{ k \in \cK }\EE[ \| y_{m,t}^k - \bar y_{m,t} \|^2] $ is similar. To quantify  $  \sum_{ k \in \cK }\EE[ \| s_{t}^k - \bar s_t\|^2] $, we observe that a weight $1-\beta_t \leq 1$ is assigned to the prior value $s_t^k$, yielding that 
   \begin{equation*}
  \begin{split}
\sum_{ k \in \cK }\EE[ \| s_{t+1}^k - \bar s_{t+1}\|^2]
 & \leq  \frac{(1+\rho)(1-\beta_t)^2}{2 }  \sum_{ k \in \cK }\EE[ \| s_{t}^k - \bar s_t\|^2] + \frac{4 \beta_t^2 K C_f^2}{1-\rho}.
 \end{split}
 \end{equation*}
We acquire the desired result by following the analysis of quantifying $\sum_{ k \in \cK }  \EE [ \| x_t^k - \bar x_t    \|^2  ] $. 
\end{proof}

 \begin{lemma}\label{lemma:nonconvex_main_multilevel}
 Suppose Assumptions~\ref{assumption:SO}, \ref{assumption:W}, \ref{assumption:1}, and \ref{assumption:2} hold.  Then 
   \begin{equation}\label{eq:nonconvex_main_multilevel}
\begin{split}  
& \EE[ \| \nabla F(\bar x_t)\|^2] 
\\
& \leq  \frac{2}{\alpha_t } \Big ( \EE[ F(\bar{x}_t)] -  \E[F(\bar{x}_{t+1}) ] \Big ) 
  -  (1 - \alpha_t L_F ) \EE[ \|  \bar z_t  
\|^2  ]  
 + 4 \EE [ \|  \nabla_1 f( \bar x_t ,   y_{M,t}^\star   ) - \bar s_t \|^2 ]  
 \\
 & \quad +  C_{2} \EE[  \| \nabla_2 f(  \bar x_t  ,   y^\star_{M,t}    ) -  \bar h_t \|^2 ]    + \sum_{m=1}^M  C_{m,3} \EE[ \|  \nabla_{12}^2 g_m( y_{m-1,t}^\star, y_{m,t}^\star    ) - \bar u_{m,t} \|_F^2 ]   
 \\
 & \quad 
+    \sum_{m=1}^M C_{m,4,j} \EE[ \| \nabla_{22}^2 g_m(y_{m-1,t}^\star, y_{m,t}^\star  ) -  \bar v_{m,t,i} \|_F^2]   + \cO \left ( \frac{M \beta_t^2}{(1-  \rho  )^2}\right ),
  \end{split}
\end{equation}
where 
\begin{equation}\label{def:constant_C}
    \begin{split}
         & C_2  = 4(2M+1)\prod_{j=1}^M C_{g,j}^2 L_{q,j}^2,    C_{m,3}  = \frac{ 4(2M+1)C_f^2}{C_{g,m}^2} \prod_{j=1}^M  C_{g,j}^2 L_{q,j}^2, 
    \\
\text{ and } & C_{m,4,j}  = \frac{ 4(2M+1)C_f^2 }{L_{q,m}^2 } \frac{ (1-\kappa_{g,m})^{b - j } }{L_{g,m}^4 \kappa_{g,m}  }     \prod_{j=1}^M  C_{g,j}^2 L_{q,j}^2, \forall 1\leq j \leq b, 1 \leq m \leq M.
    \end{split}
\end{equation}
\end{lemma}
\begin{proof}
We start from the $L_F$-smoothness of $F(x)$ provided by Lemma~\ref{lemma:Lip_y_multilevel}:
\begin{equation*}
\begin{aligned}
F(\bar{x}_{t+1})-F(\bar{x}_t) & \leq  \lv \nabla F(\bar{x}_t),  \bar x_{t+1} - \bar x_t \rv +    \frac{\alpha_t^2L_F}{2} \|\bar{z}_t\|^2 
\\
& \leq    -\alpha_t \lv \nabla F(\bar{x}_t), \bar{z}_t\rv +    \frac{\alpha_t^2L_F}{2} \|\bar{z}_t\|^2 .
\end{aligned}
\end{equation*}
By using the fact that $- 2\lv a, b \rv = - \|a\|^2 - \| b\|^2 + \| a-b\|^2$, we further obtain 
\begin{equation}\label{eq:nonconvex_1}
\begin{aligned}
&  F(\bar{x}_{t+1})  -F(\bar{x}_{t})\\
& \leq   - \frac{ \alpha_t}{2} \| \nabla F(\bar x_t)\|^2 - \frac{\alpha_t}{2}\|  \bar z_t  
\|^2  +  \frac{\alpha_t }{2} \| \nabla F(\bar x_t) - \bar z_t    \|^2  + \frac{\alpha_t^2 L_F}{2} \|\bar{z}_t\|^2.
\end{aligned}
\end{equation}
Dividing both sides by $\alpha_t/2$ and rearranging the terms, we observe that 
\begin{equation}\label{eq:multilevel_eq0}
\begin{aligned}
 \| \nabla F(\bar x_t)\|^2 \leq  \frac{2}{\alpha_t} \Big ( F(\bar{x}_{t})  -  F(\bar{x}_{t+1})  \Big )  - (1 - \alpha_t L_F)  \|  \bar z_t  
\|^2  +   \| \nabla F(\bar x_t) - \bar z_t    \|^2 .
\end{aligned}
\end{equation}
We write $y_{m,t}^\star = y_{m}^\star( \bar x_t )$ for $m=1,\cdots, M$.  By recalling the definition of $\nabla F(\bar x_t)$ in \eqref{eq:gradient_multilevel}, we obtain 
\begin{equation*}
\begin{split}
& \| \nabla F(\bar x_t)  - \bar z_t \|^2   = \|  \nabla f_1(\bar x_t ,y_M^\star( \bar x_t )) + \nabla y_{M}^\star( \bar x_t ) \nabla_2 f(\bar x_t , y_M^\star(\bar x_t )) - \bar z_t \|^2
\\
& 
\leq \frac{1}{K} \sum_{k \in \cK} \|  \nabla f_1(\bar x_t ,y_M^\star( \bar x_t )) + \nabla y_{M}^\star( \bar x_t ) \nabla_2 f(\bar x_t , y_M^\star(\bar x_t )) - z_t^k  \|^2
\\
& \leq \frac{2}{K} \sum_{k\in \cK} \Big ( \|  \nabla f_1(\bar x_t ,y_{M,t}^\star)  -  s_t^k \|^2 + \| \nabla y_{M}^\star(\bar x_t) \nabla_y f(\bar x_t , y_{M,t}^\star ) -  (-1)^M  u_{1,t}^k  q_{1,t}^k  \cdots   u_{M,t}^k  q_{M,t}^k   h_t^k   \|^2 \Big ) 
\\
& \leq \tfrac{2}{K} \sum_{k\in \cK}  \|  \nabla f_1(\bar x_t ,y_{M,t}^\star)  -  s_t^k \|^2   +  \tfrac{2(2M+1) }{K} \sum_{k\in \cK} \nabla y_{M}^\star(\bar x_t) \| \nabla_2 f(\bar x_t, y_{M,t}^\star)- h_t^k\|^2
\\
& \quad +  \tfrac{2(2M+1) }{K} \sum_{k\in \cK} \nabla y_{m-1}^\star(\bar x_t) \| \nabla_{12}^2 g_m(y_{m-1,t}^\star, y_{m,t}^\star ) - u_t^k \|^2  \| q_{m,t}^k u_{m+1}^k q_{m+1,t}^k\cdots u_{M,t}^k q_{M,t}^k h_t^k\|^2
\\
& \quad +  \tfrac{2(2M+1) }{K} \sum_{k\in \cK} \nabla y_{m-1}^\star(\bar x_t) \| \nabla_{12}^2 g_m(y_{m-1,t}^\star, y_{m,t}^\star )  \|^2  \| [\nabla_{22}^2 g_m(y_{m-1,t}^\star, y_{m,t}^\star ) ]^{-1}  - q_{m,t}^k \|^2 \|  u_{m+1}^k \cdots q_{M,t}^k h_t^k\|^2
. 
\end{split}
\end{equation*}
By applying Lemma~\ref{lemma:inverse_Hessian}
 with $(1-\kappa_{g,m})^{b+1} \leq \cO(\frac{1}{T^3})$ for $b = 3 \lceil  \log_{\frac{1}{1-\kappa_{g,m}}}(T) \rceil $, we have 
\begin{equation*}
\begin{split}
\EE[  \| [\nabla_{22}^2 g_m(y_{m-1,t}^\star, y_{m,t}^\star ) ]^{-1}  - q_{m,t}^k \|^2  ] \leq    \sum_{1 \leq j \leq b } \frac{ \EE[ \| \nabla_{22}^2 g_m(y_{m-1,t}^\star, y_{m,t}^\star  ) -   v_{m,t,i}^k \|_F^2]  }{L_{g,m}^4 \kappa_{g,m}  } + \cO\Big (\frac{1}{T^3} \Big ).
\end{split}
\end{equation*}
By combining the above two inequalities, taking expectation on both sides, 
using the boundedness of the stochastic first- and second-order samples in Assumption~\ref{assumption:2}, using the fact that $\| A \|_2 \leq \| A\|_F$ for any matrix $A$, 
and using the consensus errors provided by Lemma~\ref{lemma:consensus_multilevel}, we have that 
\begin{equation}\label{eq:z_error}
\begin{split}
& \EE[ \| \nabla F(\bar x_t)  - \bar z_t \|^2  ]  
\\
& \leq 2 \EE[ \|  \nabla f_1(\bar x_t ,y_{M,t}^\star)  -   s_t^k \|^2 ] +    \frac{C_{2}}{2}  \EE[  \| \nabla_2 f(  \bar x_t  ,   y^\star_{M,t}    ) -  h_t^k \|^2 ]   
\\
& \quad +  \sum_{m=1}^M \frac{  C_{m,3}   }{2}\EE[ \|  \nabla_{12}^2 g_m( y_{m-1,t}^\star, y_{m,t}^\star    ) -  u_{m,t}^k  \|_F^2 ]    +  \sum_{m=1}^M \sum_{j=1}^b \frac{ C_{m,4,j} }{2}\EE[ \| \nabla_{22}^2 g_m(y_{m-1,t}^\star, y_{m,t}^\star  ) -  \ v_{m,t,i}^k \|_F^2]
\\
& \leq 4 \EE[ \|  \nabla f_1(\bar x_t ,y_{M,t}^\star)  -  \bar s_t \|^2 ] +    C_{2}  \EE[  \| \nabla_2 f(  \bar x_t  ,   y^\star_{M,t}    ) -  \bar h_t \|^2 ]   + \cO( \EE[\Delta_t] )
\\
& \quad +  \sum_{m=1}^M  C_{m,3}  \EE[ \|  \nabla_{12}^2 g_m( y_{m-1,t}^\star, y_{m,t}^\star    ) - \bar u_{m,t} \|_F^2 ]    +  \sum_{m=1}^M \sum_{j=1}^b C_{m,4,j} \EE[ \| \nabla_{22}^2 g_m(y_{m-1,t}^\star, y_{m,t}^\star  ) -  \bar v_{m,t,i} \|_F^2] ,
\end{split}
\end{equation}
where $C_2,C_{m,3}, C_{m,4} >0 $ are constants defined within \eqref{def:constant_C} and  \begin{equation*}
  \Delta_t  = \frac{1}{K}\sum_{k \in \cK} \EE \Big ( \| s_t^k - \bar s_t\|^2 + \| h_t^k - \bar  h_t\|^2  + \sum_{m=1}^M[ \| u_{m,t}^k - \bar  u_{m,t}\|_F^2 + \| v_{m,t,j}^k - \bar v_{m,t,j}\|_F^2 ] \Big ).
\end{equation*}
Substituting the above inequality into \eqref{eq:multilevel_eq0}, we conclude that 
       \begin{equation*}
\begin{split}  \EE[ \| \nabla F(\bar x_t)\|^2] 
& \leq  \frac{2}{\alpha_t } \Big ( \EE[ F(\bar{x}_t)] -  \E[F(\bar{x}_{t+1}) ] \Big ) 
  -  (1 - \alpha_t L_F ) \EE[ \|  \bar z_t  
\|^2  ]  
 + 4  \EE [ \|  \nabla_1 f( \bar x_t ,   y_{M,t}^\star   ) - \bar s_t \|^2 ]  
 \\
 & \quad +  C_{2} \EE[  \| \nabla_2 f(  \bar x_t  ,   y^\star_{M,t}    ) -  \bar h_t \|^2 ]    + \sum_{m=1}^M  C_{m,3} \EE[ \|  \nabla_{12}^2 g_m( y_{m-1,t}^\star, y_{m,t}^\star    ) - \bar u_{m,t} \|_F^2 ]   
 \\
 & \quad 
+    \sum_{m=1}^M \sum_{j=1}^b C_{m,4,j} \EE[ \| \nabla_{22}^2 g_m(y_{m-1,t}^\star, y_{m,t}^\star  ) -  \bar v_{m,t,i} \|_F^2]   + \cO( \EE[\Delta_t] ).
  \end{split}
\end{equation*}

 The desired result can be acquired by applying Lemma~\ref{lemma:consensus_multilevel} that 
\begin{equation*}
\begin{split}
\EE[\Delta_t]  &  \leq 
\frac{1}{K} 
\sum_{k \in \cK} \EE \Big ( \| s_t^k - \bar s_t\|^2 + \| h_t^k - \bar  h_t\|^2  + \sum_{m=1}^M[ \| u_{m,t}^k - \bar  u_{m,t}\|_F^2 + \| v_{m,t,j}^k - \bar v_{m,t,j}\|_F^2 ] \Big ) 
\\
& \leq \cO \Big (\frac{ M\beta_t^2}{(1-\rho)^2} \Big). 
\end{split}
\end{equation*}
This completes the proof. 
\end{proof}
\begin{lemma}\label{lemma:y_multilevel}
Suppose Assumptions \ref{assumption:SO}, \ref{assumption:W}, \ref{assumption:1},  and \ref{assumption:2} hold
and $T$ is sufficiently large, let $y_{m,t}^\star = y_m^\star(\bar x_t)$, for each $m=1,\cdots,M$, we have
\begin{equation}\label{eq:yt_recursion_multilevel}
\begin{split}
& \EE [ \| \bar y_{m,t+1} - y_{m,t+1}^\star \|^2 ] 
\\
&
\leq ( 1 - \gamma_t \mu_{g,m})  \EE [ \| \bar y_{m,t} -  y_{m,t}^\star \|^2 ]  + \cO\left (\frac{ \gamma_t \alpha_t^2  + \gamma_t^3 }{(1- \rho )^2 } \right ) + \frac{ \gamma_{t} L_{g,m}^2 }{ \mu_{g,m}} \EE[ \| \bar y_{m-1,t} - y_{m-1,t}^\star \|^2] 
\\
& \quad + \frac{2\gamma_t^2 \sigma_{g,m}^2 }{K}  + \frac{3L_{y,m}^2\alpha_t^2 }{\gamma_t \mu_{g,m}  } \EE [\| \bar z_t \|^2 ]  .
\end{split}
\end{equation}
\end{lemma}
\begin{proof} Recall Algorithm~\ref{alg:1} Step 9 that $	y_{m, t+1}^k = \sum_{j \in \cN_k} w_{k,j} y_{m, t}^j   - \gamma_t \nabla_2 g_m^k (y_{m-1,t}^k,y_{m,t}^k;\xi_{m,t}^k )$, we denote by $\bar y_{m,t} = \frac{1}{K}\sum_{k \in \cK} y_{m,t}^k$ and have the following
$$
 \bar y_{m,t+1} = \bar y_{m,t} -  \frac{\gamma_t}{K}\sum_{k \in \cK } \nabla_2 g_m^k (y_{m-1,t}^k,y_{m,t}^k;\xi_{m,t}^k ). 
$$
We first decompose the estimation error  $\| \bar y_{m,t+1} - y_m^\star(\bar x_{t+1}) \|^2$ as 
\begin{equation}\label{eq:yk_00}
\| \bar y_{m,t+1} -  y_m^\star(\bar x_{t+1})\|^2 \leq \left ( 1+ \frac{\gamma_t \mu_{g,m}}{2} \right ) \| \bar y_{m,t+1} -  y_m^\star(\bar x_{t})\|^2+ \left (1 + \frac{2}{\gamma_t \mu_{g,m} } \right )  \| y_m^\star(\bar x_{t}) - y_m^\star(\bar x_{t+1}) \|^2.
\end{equation}
Recall that  $\bar y_{m,t+1} = \bar y_{m,t} -  \frac{\gamma_t}{K}\sum_{k \in \cK } \nabla_2 g_m^k (y_{m-1,t}^k,y_{m,t}^k;\xi_{m,t}^k )$. Let $$\delta_t = \nabla_2 g_m(y_{m-1}^\star( \bar x_t), \bar y_{m,t}) 
- \frac{1}{K}\sum_{k \in \cK } \nabla_2 g_m^k (y_{m-1,t}^k,y_{m,t}^k;\xi_{m,t}^k ), $$ we obtain 
\begin{equation}\label{eq:yk_01}
\begin{split}
& \| \bar y_{m,t+1} - y_m^\star(\bar x_{t}) \|^2 =  \|  \bar y_{m,t} -  \gamma_t \nabla_2 g_m( y_{m-1}^\star(\bar x_t), \bar y_t) - y_m^\star(\bar x_{t}) + \gamma_t \delta_t  \|^2 
\\
& = \|  \bar y_{m,t} -  \gamma_t \nabla_2 g_m(y_{m-1,t}^\star, \bar y_{m,t}) - y_{m,t}^\star   \|^2  + \gamma_t  \lv \bar y_{m,t} -  \gamma_t \nabla_2 g_m(y_{m-1,t}^\star, \bar y_{m,t}) - y_{m,t}^\star,  \delta_t  \rv + \gamma_t^2   \| \delta_t \|^2. 
\end{split}
\end{equation}
We then provide bounds for the above terms. First, consider $ \|  \bar y_{m,t} -  \gamma_t \nabla_2 g(y_{m-1}^\star(\bar x_t) , \bar y_{m,t}) - y_m^\star(\bar x_{t})  \|^2$, by using the $\mu_g$-strong convexity of $g(\bar x_t,y)$ in $y$ under Assumption~\ref{assumption:2} (i), 
we have 
\begin{equation}\label{eq:yk_02}
\begin{split}
& \|  \bar y_{m,t} -  \gamma_t \nabla_2 g(y_{m-1,t}^\star , \bar y_{m,t}) - y_{m,t}^\star \|^2
\\
 & \leq \| \bar y_{m,t} - y_{m,t}^\star \|^2  - 2\gamma_t  \lv \bar y_{m,t}   - y_{m,t}^\star, \nabla_2 g( y_{m-1,t}^\star , \bar y_{m,t})  \rv + \gamma_t^2 \| \nabla_2 g( y_{m-1,t}^\star , \bar y_{m,t})   \|^2
 \\
 & \leq (1- 2\gamma_t  \mu_{g,m})\| \bar y_{m,t} - y_{m,t}^\star \|^2 + \gamma_t^2 C_{g,m}^2. 
 \end{split}
\end{equation}
Next, consider $ \lv \bar y_{m,t} -  \gamma_t \nabla_2 g_m(y_{m-1,t}^\star, \bar y_{m,t}) - y_{m,t}^\star,  \delta_t  \rv  $, we can see that 
\begin{equation}\label{eq:yk_03}
\begin{split}
&  \EE \left [ \lv \bar y_{m,t} -  \gamma_t \nabla_2 g_m(y_{m-1,t}^\star, \bar y_{m,t}) - y_{m,t}^\star,  \delta_t  \rv    \right  ]
 =  \EE \left [  \Big (  \bar y_{m,t} -  \gamma_t \nabla_2 g_m(y_{m-1,t}^\star, \bar y_{m,t}) - y_{m,t}^\star  \Big )^\top  \Delta_t \right   ]
\\
& \leq   \frac{ \mu_g}{2} \EE[ \|  \bar y_{m,t} -  \gamma_t \nabla_2 g_m(y_{m-1,t}^\star, \bar y_{m,t})  - y_{m,t}^\star  \|^2 ]
+ \frac{ \Delta_t}{2\mu_{g,m}} ,
 \end{split}
\end{equation}
where $\Delta_t  = \EE[ \| \nabla_2 g_m(y_{m-1,t}^\star, \bar y_{m,t}) - \frac{1}{K}\sum_{k=1}^k \nabla_2 g_m^k (y_{m-1,t}^k,y_{m,t}^k)   \|^2]$ and the last inequality comes from the fact that $\lv a, b \rv \leq \frac{\| a\|^2}{2} + \frac{\| b \|^2}{2}$. 
Further, for  $\| \delta_t^k\|^2$, we have 
\begin{equation}\label{eq:yk_04}
\begin{split}
& \EE[ \|  \delta_t\|^2 ]  = \EE \left [ \Big \| \frac{1}{K} \sum_{k \in \cK} \big  ( \nabla_2 g_m^k(y_{m-1}^\star( \bar x_t), \bar y_{m,t})   - \nabla_2 g_m^k (y_{m-1,t}^k,y_{m,t}^k;\xi_{m,t}^k )  \big )   \Big \|^2 \right ]  \\
& \leq 2 \EE \left [ \Big \| \frac{1}{K} \sum_{k \in \cK} \big  ( \nabla_2 g_m^k (y_{m-1,t}^k,y_{m,t}^k)  - \nabla_2 g_m^k (y_{m-1,t}^k,y_{m,t}^k;\xi_{m,t}^k )   \big )   \Big \|^2 \right ] + 2\Delta_t^k \leq \frac{2 \sigma_{g,m}^2}{K} + 2 \Delta_t^k,
\end{split}
\end{equation}
where the last inequality uses the fact that $\{ \nabla_2 g_m^k (y_{m-1,t}^k,y_{m,t}^k)  - \nabla_2 g_m^k (y_{m-1,t}^k,y_{m,t}^k;\xi_{m,t}^k )   \}$'s are conditionally mean-zero and independent such that  for $k \neq s$, 
$$
\EE \left [ \lv \nabla_2 g_m^k (y_{m-1,t}^k,y_{m,t}^k)  - \nabla_2 g_m^k (y_{m-1,t}^k,y_{m,t}^k;\xi_{m,t}^k ) ,  \nabla_2 g_m^s (y_{m-1,t}^k,y_{m,t}^k)  - \nabla_2 g_m^s (y_{m-1,t}^s,y_{m,t}^s;\xi_{m,t}^s )   \rv \right ] = 0. 
$$ 
Taking expectations on both sides of \eqref{eq:yk_01} and combining with \eqref{eq:yk_02}, \eqref{eq:yk_03}, and \eqref{eq:yk_04}, we have 
\begin{equation}\label{eq:yk_05}
\begin{split}
 & \EE [ \| \bar y_{m,t+1} - y_m^\star(\bar x_{t}) \|^2 ]
 \\
& \leq (1  +  \frac{\gamma_t \mu_{g,m}}{2}) \Big (  (1- 2\gamma_t  \mu_{g,m} )\EE[ \| \bar y_{m,t} - y_{m,t}^\star\|^2 ] + \gamma_t^2 C_{g,m}^2 \Big )
+ \left ( \frac{ \gamma_k}{2\mu_{g,m}} +  2\gamma_t^2 \right ) \Delta_t  + \frac{2\gamma_t^2 \sigma_{g,m}^2}{K}
\\
& \leq  (1- \frac{3\gamma_t \mu_g}{2})\EE[\| \bar y_{m,t} - y_{m,t}^\star\|^2] + (1 + \gamma_t \mu_{g,m})\gamma_t^2 C_{g,m}^2 
\\
& \quad + \left ( \frac{ \gamma_k}{2\mu_{g,m}} +  2\gamma_t^2 \right ) \frac{1}{K} \sum_{k \in \cK} L_g^2 \EE \big ( \| y_{m-1,t}^k - y_{m-1,t}^\star\|^2 + \| y_{m,t}^k - y_{m,t}^\star\|^2 \big ) + \frac{2\gamma_t^2 \sigma_{g,m}^2}{K},
\end{split}
\end{equation}
where the second inequality uses the $L_{g,m}$-smoothness of $\nabla_2 g_m(y_{m-1},y_m)$ in both $y_{m-1}$ and $y_m$ such that
\begin{equation}\label{eq:Delta_t}
    \begin{split}
    \Delta_t  & = \EE \left [ \left \| \nabla_2 g_m(y_{m-1,t}^\star, \bar y_{m,t}) - \frac{1}{K}\sum_{k \in \cK}  \nabla_2 g_m^k (y_{m-1,t}^k,y_{m,t}^k) \right \|^2 \right ] 
    \\
    & \leq \frac{1}{K} \sum_{k \in \cK} \EE[ \| \nabla_2 g_m^k  (y_{m-1,t}^\star, \bar y_{m,t})  - \nabla_2 g_m^k  (y_{m-1,t}^k,y_{m,t}^k) \|^2]
    \\
 & 
 \leq \frac{1}{K} \sum_{k \in \cK}  L_g^2 \big ( \EE[ \| y_{m-1,t}^\star  - y_{m-1,t}^k \|^2] + \EE[ \| y_{m,t}^k - \bar y_{m,t}\|^2 ]\big ).    
    \end{split}
\end{equation}
Moreover, by using the $L_{y,m}$-smoothness of $y_m^\star(\cdot)$ characterized by Lemma~\ref{lemma:Lip_y_multilevel}, we have 
\begin{equation}\label{eq:yk_06}
\EE [ \| y^\star(\bar x_{t}) - y^\star(\bar x_{t+1}) \|^2 ]\leq L_{y,m}^2 \EE [\|  \bar x_{t} - \bar x_{t+1} \|^2 ].
\end{equation}
Finally, by substituting \eqref{eq:yk_05} and \eqref{eq:yk_06} into \eqref{eq:yk_00} and applying the bounds of the consensus errors provided by Lemma~\ref{lemma:consensus_multilevel}, we conclude that 
\begin{equation*}
\begin{split}
& \EE [ \| \bar y_{m,t+1} -  y_m^\star(\bar x_{t+1})\|^2  ]
 \\
& \leq  ( 1 - \gamma_t \mu_{g,m} ) \EE[  \| \bar y_{m,t} -  y_m^\star(\bar x_{t})\|^2 ]
+ 
(1 + \gamma_t\mu_{g,m} /2)  (1 + \gamma_t \mu_{g,m} )\gamma_t^2 C_{g,m} ^2 
\\
& \quad + \left  ( \frac{ \gamma_k}{2\mu_{g,m}} +  2\gamma_t^2 \right ) \frac{1}{K} \sum_{k \in \cK} L_{g,m}^2 \big ( \EE[ \| y_{m-1,t}^k - y_{m-1,t}^\star \|^2] +  \EE [ \| y_{m,t}^k - \bar y_{m,t}\|^2] \big ) 
\\
& \quad + \frac{2\gamma_t^2 C_{g,m}^2}{K}
 + \left (1 + \frac{2}{\gamma_t \mu_{g,m} } \right )  \EE [\|  y_{m,t+1}^\star - y_{m,t}^\star \|^2 ].
 \end{split}
\end{equation*}
Because $\gamma_t\mu_g \leq 1$ for large $T$, the desired result can be obtained by using Lemmas~\ref{lemma:Lip_y_multilevel} and \ref{lemma:consensus_multilevel}  that 
\begin{equation*}
\begin{split}
  &   \sum_{k=1}^K \EE[ \| y_{m-1,t}^k - y_{m-1,t}^\star \|^2] \leq  2K \EE[ \| \bar y_{m-1,t} - y_{m-1,t}^\star \|^2] + \cO \left (\frac{K \gamma_t^2}{(1-\rho^2)} \right ),
\\
&\sum_{k=1}^K  \EE [ \| y_{m,t}^k - \bar y_{m,t}\|^2] \leq  \cO \left (\frac{K \gamma_t^2}{(1-\rho^2)} \right ),  \|  y_{m,t+1}^\star - y_{m,t}^\star \|  \leq L_{y,m}  \|\bar x_t - \bar x_{t+1} \|.
\end{split}
\end{equation*}
and the fact  $\bar x_{t+1} = \bar x_t - \alpha_t \bar z_t$. 
\end{proof}


 \begin{lemma}\label{lemma:error_multilevel}
Suppose Assumptions \ref{assumption:SO}, \ref{assumption:W}, \ref{assumption:1}, , and \ref{assumption:2} hold and $T $ is sufficiently large, let $y_{m,t}^\star = y_m^\star(\bar x_t )$ for $m=1,\cdots,M$, then we have 
\noindent 
\\
(a)\begin{equation}\label{eq:s_multilevel}
    \begin{split}
    &    \EE[ \|  \bar s_{t+1} - \nabla_1 f( \bar x_{t+1}, y_{M,t+1}^\star ) \|^2]  
        \\
              & \leq (1-\beta_t)\EE [ \| \bar s_t - \nabla_1 f (   \bar x_{t}, y_{M}^\star(\bar x_t) ) \|^2] + \frac{2\beta_t^2   C_f^2}{K }    + \frac{4\alpha_t^2  L_f^2 (1 + L_{y,M}^2 )}{\beta_t }  \EE[  \| \bar z_t \|^2  ] 
              \\
              & \quad + \frac{2\beta_t^2  \sigma_f^2}{K } +    \cO \left ( \frac{ \beta_t ( \alpha_t^2 +  \beta_t^2) }{(1- \rho)^2 } \right )  + 6\beta_t L_f^2  \EE[ \| \bar y_{M,t} - y_{M,t}^\star \|^2 ]. 
        \end{split}
\end{equation}
\\
\noindent 
(b)\begin{equation}\label{eq:h_multilevel}
    \begin{split}
    &    \EE[ \|  \bar h_{t+1} - \nabla_2 f( \bar x_{t+1}, y_M^\star(\bar x_{t+1}) ) \|^2]  
        \\
              & \leq (1-\beta_t)\EE [ \| \bar h_t - \nabla_2 f (   \bar x_{t}, y_M^\star(x_{t}) ) \|^2]    + \frac{4\alpha_t^2  L_f^2 (1 + L_{y,M}^2 )}{\beta_t }  \EE[  \| \bar z_t \|^2  ] 
              \\
              & \quad  + \frac{2\beta_t^2  \sigma_f^2}{K }  +  \cO \left ( \frac{  \beta_t ( \alpha_t^2 +  \beta_t^2) }{(1- \rho)^2} \right ) + 6\beta_t L_f^2  \EE[ \| \bar y_{M,t} - y_M^\star(x_{t})\|^2 ].
        \end{split}
\end{equation}
\end{lemma}
\begin{proof}(a)
We denote by $y_{M,t}^\star = y_{M}^\star(\bar x_t)$
for notational convenience. 
Consider the update rule of $\bar s_{t+1}$, we have 
\begin{equation*}
\begin{split}
 \bar s_{t+1} - \nabla_1 f(\bar x_{t+1}, y_{M,t+1}^\star)  
 = (1-\beta_t) [\bar s_t  -\nabla_1 f(\bar x_{t+1}, y_{M,t+1}^\star)   ]  + \beta_t \Delta_{f,t},
\end{split}
\end{equation*}
where 
$$
\Delta_{f,t}   =  \frac{1}{K}\sum_{k \in \cK}  \nabla_1 f^k(x_t^k, y_{M,t}^k ; \zeta_t^k )  - \nabla_1 f(\bar x_{t+1},  y_{M,t+1}^\star ) .
$$
We can see that 
\begin{equation}\label{eq:st_1}
    \begin{split}
       & \EE[ \|   \bar s_{t+1} - \nabla_1 f(\bar x_{t+1}, y_{M,t+1}^\star)   \|^2] 
       \\
 & = (1-\beta_t)^2 \EE [ \| \bar s_t -  \nabla_1 f(\bar x_{t+1}, y_{M,t+1}^\star )  \|^2] + \beta_t^2 \EE [ \|  \Delta_{f,t}  \|^2] \\
       & \quad + 2(1-\beta_t)\beta_t \EE \left  [ (\bar s_t - \nabla_1 f(\bar x_{t+1}, y_{M,t+1}^\star  )  )^\top  \Delta_{f,t}   \right ]
    \\
   &  = (1-\beta_t)^2 \EE [ \| \bar s_t -  \nabla_1 f(\bar x_{t+1}, y_{t+1}^\star )  \|^2] + \beta_t^2 \EE [ \|  \Delta_{f,t}  \|^2] \\
       & \quad + 2(1-\beta_t)\beta_t \EE \left  [ (\bar s_t - \nabla_1 f(\bar x_{t+1}, y_{M,t+1}^\star  )  )^\top  \Big  ( \frac{1}{K}\sum_{k \in \cK}  \nabla_1 f^k(x_t^k, y_t^k )  - \nabla_x f(\bar x_{t+1},  y_{t+1}^\star ) \Big )  \right ]
    \\
    & \leq (1-\beta_t)^2 \EE [ \| \bar s_t -  \nabla_1 f( \bar x_{t+1}, y_{t+1}^\star )  \|^2  ] + \beta_t^2 \EE [ \|  \Delta_{f,t}     \|^2] 
    \\
       & \quad + \frac{(1-\beta_t)\beta_t }{2} \Big ( \EE[ \| \bar s_t - \nabla_1 f(\bar x_{t+1}, y_{M,t+1}^\star  )   \|^2 +4  \EE[ \|    \tfrac{1}{K}\sum_{k \in \cK}  \nabla_1 f^k(x_t^k, y_{M,t}^k)  - \nabla_1 f(\bar x_{t+1}, y_{M,t+1}^\star  )   \|^2 ]  \Big)
       \\
       & \leq (1-\beta_t) \left (1- \frac{\beta_t}{2} \right ) \EE[ \| \bar s_t - \nabla_1 f(\bar x_{t+1}, y_{t+1}^\star  )   \|^2 ] + \beta_t^2  \EE [ \|  \Delta_{f,t}  \|^2]  \\
       & \quad  + \frac{2\beta_t(1-\beta_t)  L_f^2 }{K}  \sum_{k \in \cK } \left ( \EE[ \|   x_{t}^k - \bar x_{t+1} \|^2 ] + \EE[ \|   y_{M,t}^k -  y_{M,t+1}^\star  \|^2 ]  \right ) ,
    \end{split}
\end{equation}
where the first equality uses the conditional independence between $\{ \nabla_1 f^k(x_t^k, y_{M,t}^k ; \zeta_t^k )  - \nabla_1 f^k(x_t^k, y_{M,t}^k ) \} $ and $\bar x_{t+1},\bar y_{t+1}, \bar s_t$,  the first inequality uses the fact that $2\lv a,b\rv \leq \frac{\|a\|^2}{2} + 2 \| b\|^2 $, and the last inequality applies similar analysis as \eqref{eq:Delta_t} under the $L_f$-smoothness of $\nabla_x f^k$. 
We observe that 
\begin{equation*}
\begin{split}
&  \| \Delta_{f,t} \|^2 
\\
&  =  \Big \| 
\frac{1}{K}\sum_{k \in \cK}  \left ( \nabla_1 f^k(x_t^k, y_{M,t}^k ; \zeta_t^k )  - \nabla_1 f^k(x_t^k, y_{M,t}^k)  \right )
 +  \frac{1}{K} \sum_{k\in \cK}( \nabla_1 f^k ( x_t^k, y_{M,t}^k) - \nabla_x f^k(\bar x_{t+1}, y_{M,t+1}^\star ) 
 )  \Big \|^2 
 \\
 & \leq  2 \Big \|  \frac{1}{K}\sum_{k \in \cK}  \left ( \nabla_1 f^k(x_t^k, y_{M,t}^k ; \zeta_t^k )  - \nabla_1 f^k(x_t^k, y_{M,t}^k)  \right )  \Big \|^2 
 \\
 & \quad +  2\Big \|  \frac{1}{K} \sum_{k\in \cK}( \nabla_1 f^k ( x_t^k, y_{M,t}^k) - \nabla_1 f^k(\bar x_{t+1},  y_{M,t+1}^\star ) 
 )  \Big \|^2. 
 \end{split}
\end{equation*}
By noting that $\nabla_x f^k(x_t^k, y_t^k ;  \zeta_t^k )  - \nabla_x f^k(x_t^k, y_t^k)$ is conditionally mean-zero and 
$$\EE \Big [ \big (  \nabla_x f^i(x_t^i, y_t^i ; \zeta_t^i )  - \nabla_x f^i(x_t^i, y_t^i)  \big )^\top \big ( \nabla_x f^j( x_t^j, y_t^j ; \zeta_t^j)  - \nabla_x f ^j(x_t^i, y_t^i) \big ) \Big ] = 0, \text{ for } 1\leq i \neq j \leq K, 
$$
and taking expectations on both sides of the above inequality, we obtain 
\begin{equation}\label{eq:st_2}
\begin{split}
\EE[ \| \Delta_{f,t} \|^2 ]
 & \leq    \frac{2}{K^2}  \sum_{k \in \cK} \EE[ \| \nabla_1 f^k(x_t^k, y_{M,t}^k ; \zeta_t^k )  - \nabla_1 f^k(x_t^k, y_{M,t}^k)  \|^2 ] 
 \\
 & \quad +  \frac{2}{K} \sum_{k \in \cK} \EE[ \|  \nabla_1 f^k ( x_t^k, y_{M,t}^k) - \nabla_1 f^k(\bar x_{t+1},  y_{M,t+1}^\star ) 
 \|^2 ] 
 \\
 & \leq \frac{2\sigma_f^2}{K} +  \frac{2 L_f^2 }{K} \sum_{k\in \cK}  \left ( \EE[ \| x_{t}^k -  \bar x_{t+1}  \|^2 ] 
 + \EE[ \| y_{M,t}^k -  y_{M,t+1}^\star   \|^2 ] \right ),
 \end{split}
\end{equation}
where the last inequality uses the $L_f$-smoothness of $\nabla_1 f^k(x,y)$ in both $x$ and $y$, similar as~\eqref{eq:Delta_t}.

Further, we have 
\begin{equation*}
\begin{split}
 & \| \bar s_t - \nabla_1 f(\bar x_{t+1},  y_{M,t+1}^\star  )   \|^2
 \\
 & \leq \left ( 1 +  \frac{\beta_t}{3} \right )  \| \bar s_t - \nabla_1 f(\bar x_{t}, y_{M,t}^\star )   \|^2  + \left (1 + \frac{3}{\beta_t}  \right ) \| \nabla_1 f(\bar x_{t},  y_{M,t}^\star )   -\nabla_1 f(\bar x_{t+1}, y_{M,t+1}^\star ) \|^2 
     \\
  & \leq  \left ( 1 +  \frac{\beta_t}{3} \right ) \| \bar s_t - \nabla_1 f ( \bar x_{t}, y_{M,t}^\star  )  \|^2 + \left (1 + \frac{3}{\beta_t}  \right ) L_f^2 (  \| \bar x_t - \bar x_{t+1}\|^2  +  \| y_t^\star  -  y_{M,t+1}^\star \|^2 )
  \\
  & =  \left ( 1 +  \frac{\beta_t}{3} \right ) \|\bar s_t - \nabla_1 f ( \bar x_{t}, y_{M,t}^\star   )  \|^2 + \left  (1 + \frac{3}{\beta_t}  \right )\alpha_t^2 L_f^2 (1+ L_{y,M}^2)\| \bar z_t \|^2,
    \end{split}
\end{equation*}
where the first inequality uses the $L_f$-smoothness of $\nabla_1 f(x,y_M)$ and the second inequality uses the $L_{y,M}$-Lipschitz continuity of $y_{M}^\star(x)$ such that $\| y_M^\star(\bar x_t) - y_M^\star(\bar x_{t+1}) \|  \leq L_{y,M} \|\bar x_t - \bar x_{t+1} \|$ and the fact that $ \alpha_t \| \bar z_t\| = \| \bar x_t - \bar x_{t+1} \|$.

By substituting  the above inequality and \eqref{eq:st_2} into \eqref{eq:st_1}, we obtain 
\begin{equation*}
    \begin{split}
       &  \EE[ \|  \bar s_{t+1} - \nabla_1 f( \bar x_{t+1}, y_{M,t+1}^\star ) \|^2]  
       \\
              & \leq (1-\beta_t)\EE [ \| \bar s_t - \nabla_1 f (   \bar x_{t}, y_{M,t}^\star ) \|^2] + \frac{2\beta_t^2   \sigma_f^2}{K }    + \frac{7\alpha_t^2  L_f^2 (1 + L_{y,M}^2 )}{2\beta_t }  \EE[  \| \bar z_t \|^2  ] 
              \\
              & \quad + \frac{2\beta_t L_f^2 }{K} \sum_{k \in \cK} \left ( \| x_{t}^k -  \bar x_{t+1}  \|^2   + \| y_{M,t}^k -  y_{M,t+1}^\star   \|^2  \right )  ,
        \end{split}
\end{equation*}
where the last inequality uses the fact that $(1- \frac{\beta_t}{2}) ( 1+ \frac{\beta_t}{3})\leq 1 $ and $(1-\beta_t)(1- \frac{\beta_t}{2}) ( 1 +  \frac{3}{\beta_t } ) \leq \frac{7}{2\beta_t}$ when $T$ is sufficiently large.

Further, we can see 
\begin{equation}\label{eq:x_y_consensus}
\begin{split}
& 
\sum_{k \in \cK} \left ( \EE[ \| x_{t}^k -  \bar x_{t+1}  \|^2 ]  + \EE[ \| y_{M,t}^k -  y_{M,t+1}^\star   \|^2 ] \right ) 
\\
& \leq \sum_{k \in \cK}  \EE \left [ 2\| x_{t}^k -  \bar x_t\|^2 +2 \|  \bar x_t- \bar x_{t+1}  \|^2  +3 \| y_{M,t}^k -  \bar y_t \|^2 + 3\| \bar y_{M,t}  - y_{M,t}^\star \|^2  +  3\| y_{M,t}^\star - y_{M,t+1}^\star   \|^2   \right ]
\\
&  \leq \cO \left ( \frac{ K( \alpha_t^2 +  \beta_t^2) }{(1- \rho)^2 } \right ) + 2K(1+ L_{y,M}^2)\EE[ \|  \bar x_t- \bar x_{t+1}  \|^2]  + 3\sum_{k\in \cK} \EE[ \| \bar y_{M,t} - y_{M,t}^\star\|^2 ]
\\
& \leq  \cO \left ( \frac{ K( \alpha_t^2 +  \beta_t^2) }{(1- \rho)^2 } \right )  + 2K (1 + L_{y,M}^2)  \alpha_t^2 \EE[ \|\bar z_t  \|^2] +  3 K \EE[ \| \bar y_{M,t} - y_{M,t}^\star\|^2 ],
\end{split}
\end{equation}
implying that 
\begin{equation*}
    \begin{split}
       &  \EE[ \|  \bar s_{t+1} - \nabla_1 f( \bar x_{t+1}, y_{M,t+1}^\star ) \|^2]  
       \\
              & \leq (1-\beta_t)\EE [ \| \bar s_t - \nabla_1 f (   \bar x_{t}, y_{M,t}^\star ) \|^2] + \frac{2\beta_t^2   \sigma_f^2}{K }    + \frac{4\alpha_t^2  L_f^2 (1 + L_{y,M}^2 )}{\beta_t }  \EE[  \| \bar z_t \|^2  ] 
              \\
              & \quad +    \cO \left ( \frac{ \beta_t ( \alpha_t^2 +  \beta_t^2) }{(1- \rho)^2 } \right )  + 6\beta_t L_f^2  \EE[ \| \bar y_{M,t} - y_{M,t}^\star\|^2 ]. 
        \end{split}
\end{equation*}
This completes the proof. 
\\
(b) This part can be derived by following the similar analysis as part (a), we skip the details to avoid repetition. 
\end{proof}

\begin{lemma}\label{lemma:ut}
Suppose Assumptions \ref{assumption:SO}, \ref{assumption:W}, \ref{assumption:1}, , and \ref{assumption:2} hold and $T $ is sufficiently large, let $y_{m,t}^\star = y_m^\star(\bar x_t )$ for $m=1,\cdots,M$, then we have 
\\
\noindent 
(a) For $m=1,\cdots, M$, 
\begin{equation}\label{eq:u_multilevel}
    \begin{split}
    &    \EE[ \|  \bar u_{m,t+1} - \nabla_{12}^2 g_m( y_{m-1}^\star( \bar x_{t+1}), y_m^\star(\bar x_{t+1}) ) \|_F^2]  
        \\
              & \leq (1-\beta_t)\EE [ \| \bar u_{m,t} - \nabla_{12}^2  g_m (  y_{m-1}^\star ( \bar x_{t}), y_m^\star(\bar x_{t}) ) \|_F^2]  + \frac{2\beta_t^2   \sigma_{g,m}^2}{K }  +  \cO \left ( \frac{  \beta_t ( \alpha_t^2 +  \beta_t^2)  }{(1- \rho)^2 } \right ) 
              \\
              & \quad + \frac{4 \alpha_t^2 \widetilde L_{g,m}^2(L_{y,m-1}^2 + L_{y,m}^2) }{\beta_t }  \EE[  \| \bar z_t \|^2  ] 
      + 6\beta_t \widetilde L_{g,m}^2  \EE[ \| \bar y_{m,t} - y_m^\star(\bar x_{t})\|^2 ] 
              \\
              & \quad + 6\beta_t \widetilde L_{g,m-1}^2  \EE[ \| \bar y_{m-1,t} - y_{m-1}^\star(\bar x_{t})\|^2 ].
        \end{split}
\end{equation}
\\
\noindent
(b) For $m=1,\cdots, M$, for $j=1,\cdots, b$,
\begin{equation}\label{eq:v_multilevel}
    \begin{split}
    &    \EE[ \|  \bar v_{m,t+1,j} - \nabla_{22}^2 g_m(  y_{m-1}^\star(\bar x_{t+1}) , y_{m}^\star(\bar x_{t+1}) ) \|_F^2]  
        \\
              & \leq (1-\beta_t)\EE [ \| \bar v_{m,t,j} - \nabla_{22}^2 g_m (  y_{m-1}^\star(\bar x_t) , y_{m}^\star(\bar x_{t}) ) \|_F^2]   + \frac{2\beta_t^2   \sigma_{g,m}^2}{K }  +  \cO \left ( \frac{  \beta_t ( \alpha_t^2 +  \beta_t^2)  }{(1- \rho)^2 } \right ) 
              \\
              & \quad + \frac{4\alpha_t^2  \widetilde L_{g,m}^2 (L_{y,m-1}^2 + L_{y,m}^2 )}{\beta_t }  \EE[  \| \bar z_t \|^2  ] 
      + 6\beta_t \widetilde L_{g,m}^2  \EE[ \| \bar y_{m,t} - y_m^\star(\bar x_{t})\|^2 ]
               \\
              & \quad + 6\beta_t \widetilde L_{g,m-1}^2  \EE[ \| \bar y_{m-1,t} - y_{m-1}^\star(\bar x_{t})\|^2 ].
        \end{split}
\end{equation}
\end{lemma}
\begin{proof} 
(a) By recalling Algorithm \ref{alg:1} Step 11 
and letting $\bar u_{m,t} = \frac{1}{K} \sum_{k \in \cK} u_{m,t}^k$, we have 
\begin{equation*}
\bar u_{m,t+1} =(1-\beta_t) \bar u_{m,t} + \frac{\beta_t}{K} \sum_{k \in \cK}  \nabla_{12}^2 g_m^k(y_{m-1,t}^k , y_{m,t}^k ; \xi_{m,t}^k ).
\end{equation*} 
Equivalently, we have 
\begin{equation*}
\begin{split}
\bar u_{m,t+1} - \nabla_{12}^2 g_m( y_{m-1,t+1}^\star, y_{m,t+1}^\star)  
 = (1-\beta_t) [\bar u_{m,t}  -\nabla_{12}^2 g_m(y_{m-1,t+1}^\star, y_{m,t+1}^\star )   ]  + \beta_t \Delta_{g,t},
\end{split}
\end{equation*}
where 
$$
\Delta_{g,t}   =  \Big ( \frac{1}{K}\sum_{k \in \cK}    \nabla_{12}^2 g_m^k(y_{m-1,t}^k, y_{m,t}^k ; \xi_{m,t}^k )   \Big ) - \nabla_{12}^2 g_m(y_{m-1,t+1}^\star, y_{m,t+1}^\star )   .
$$
This implies 
\begin{equation*}
    \begin{split}
       & \EE[ \|   \bar u_{m,t+1} - \nabla_{12}^2 g_m(y_{m-1,t+1}^\star, y_{m,t+1}^\star )     \|_F^2] 
       \\
   &  = (1-\beta_t)^2 \EE [ \| \bar u_{m,t}  -\nabla_{12}^2 g_m(y_{m-1,t+1}^\star, y_{m,t+1}^\star )  \|_F^2] + \beta_t^2 \EE [ \|  \Delta_{g,t}  \|_F^2] 
   \\
       & \quad + 2(1-\beta_t)\beta_t \EE \left  [ \lv \bar u_{m,t}  -\nabla_{12}^2 g_m(y_{m-1,t+1}^\star, y_{m,t+1}^\star )  ,   \Delta_{g,t}\rv  \right ]
       \\
          &  = (1-\beta_t)^2 \EE [ \| \bar u_{m,t}  -\nabla_{12}^2 g_m(y_{m-1,t+1}^\star, y_{m,t+1}^\star )  \|_F^2] + \beta_t^2 \EE [ \|  \Delta_{g,t}  \|_F^2] 
   \\
       & \quad + 2(1-\beta_t)\beta_t \EE \left  [ \lv \bar u_{m,t}  -\nabla_{12}^2 g_m(y_{m-1,t+1}^\star, y_{m,t+1}^\star )  ,  \tilde  \Delta_{g,t}\rv  \right ]
    \end{split}
    \end{equation*}
where 
$$
\tilde \Delta_{g,t} = \Big ( \frac{1}{K}\sum_{k \in \cK}    \nabla_{12}^2 g_m^k(y_{m-1,t}^k, y_{m,t}^k  )   \Big ) - \nabla_{12}^2 g_m(y_{m-1,t+1}^\star, y_{m,t+1}^\star ). 
$$
The $L_{g,m}$-Lipschitz continuity of $\nabla_{12}^2 g_m$ implies 
$$
\| \tilde \Delta_{g,t} \|_F^2 \leq L_{g,m}^2 \sum_{k \in \cK} \Big ( \| y_{m-1,t}^k - y_{m-1,t+1}^\star \|^2 + \| y_{m,t}^k  - y_{m,t+1}^\star\|^2 \Big ). 
$$
Combining the above inequalities and following \eqref{eq:st_1}, 
we  can see that 
\begin{equation}\label{eq:ut_1}
    \begin{split}
       & \EE[ \|   \bar u_{m,t+1} - \nabla_{12}^2 g_m(y_{m-1,t+1}^\star, y_{m,t+1}^\star )     \|_F^2] 
       \\
       &  = (1-\beta_t)^2 \EE [ \| \bar u_{m,t}  -\nabla_{12}^2 g_m(y_{m-1,t+1}^\star, y_{m,t+1}^\star )  \|_F^2] + \beta_t^2 \EE [ \|  \Delta_{g,t}  \|_F^2] 
   \\
       & \quad + 2(1-\beta_t)\beta_t \EE \left  [ \lv \bar u_{m,t}  -\nabla_{12}^2 g_m(y_{m-1,t+1}^\star, y_{m,t+1}^\star )  ,  \tilde  \Delta_{g,t}\rv  \right ]
    \\
 & \leq    (1-\beta_t)^2 \EE [ \| \bar u_{m,t}  -\nabla_{12}^2 g_m(y_{m-1,t+1}^\star, y_{m,t+1}^\star )  \|_F^2] + \beta_t^2 \EE [ \|  \Delta_{g,t}  \|_F^2] 
 \\
       & \quad + \frac{ (1-\beta_t)\beta_t }{2} \left  [ \EE  \| \bar u_{m,t}  -\nabla_{12}^2 g_m(y_{m-1,t+1}^\star, y_{m,t+1}^\star )  \|_F^2 + 4  \EE[ \| \tilde  \Delta_{g,t}\|_F^2 \right ]
    \\
    & \leq (1-\beta_t) \left (1- \frac{\beta_t}{2} \right ) \EE [ \| \bar u_{m,t}  -\nabla_{12}^2 g_m(y_{m-1,t+1}^\star, y_{m,t+1}^\star )    \|_F^2  ] + \beta_t^2 \EE [ \|  \Delta_{g,t}     \|_F^2] 
\\
& \quad + 2(1-\beta_t)\beta_t   \EE[ \|    \tilde \Delta_{g,t}  \|_F^2 ] 
       \\
       & \leq (1-\beta_t) \left (1- \frac{\beta_t}{2} \right ) \EE[ \| \bar u_{m,t}  -\nabla_{12}^2 g_m(y_{m-1,t+1}^\star, y_{m,t+1}^\star )  \|_F^2  ] + \beta_t^2  \EE [ \|  \Delta_{g,t}  \|_F^2]  \\
       & \quad  + \frac{2\beta_t(1-\beta_t)  L_{g,m}^2 }{K}  \sum_{k \in \cK } \left ( \EE[ \| y_{m-1,t}^k - y_{m-1,t+1}^\star \|^2  ] + \EE[ \|   y_{m,t}^k -  y_{m,t+1}^\star  \|^2 ]  \right ) ,
    \end{split}
\end{equation}
where $\lv \cdot , \cdot  \rv$ is the inner product.
By following \eqref{eq:st_2}, we obtain 
\begin{equation}\label{eq:ut_2}
\begin{split}
\EE[ \| \Delta_{g,t} \|_F^2 ]
 & \leq    \frac{2}{K^2}  \sum_{k \in \cK} \EE[ \| \nabla_{12}^2 g_m^k(y_{m-1,t}^k, y_{m,t}^k ; \xi_{m,t}^k )  - \nabla_{12}^2 g_m^k(y_{m-1,t}^k, y_{m,t}^k  )  \|_F^2 ] 
 \\
 & \quad +  \frac{2}{K} \sum_{k \in \cK} \EE[ \|  \nabla_{12}^2 g_m^k ( y_{m-1,t}^k, y_{m,t}^k  ) - \nabla_{12}^2 g_m^k( y_{m-1,t+1}^\star,  y_{m,t+1}^\star ) 
 \|_F^2 ] 
 \\
 & \leq \frac{2\sigma_{g,m}^2}{K} +  \frac{2 \widetilde L_{g,m}^2 }{K} \sum_{k\in \cK}  \left ( \EE[ \| y_{m-1,t}^k -  y_{m-1,t+1}^\star  \|^2 ] 
 + \EE[ \| y_{m,t}^k -  y_{m,t+1}^\star   \|^2 ] \right ).
 \end{split}
\end{equation}
Using the fact that $\| A+B\|_F^2 \leq (1+ \frac{\beta_t}{3}) \|A \|_F^2 + (1+ \frac{3}{\beta_t}) \|B \|_F^2$ and Lemma~\ref{lemma:Lip_y_multilevel} that $\| y_{m,t+1}^\star - y_{m,t}^\star \| \leq L_{y,m} \| \bar x_{t+1} - \bar x_t\| = L_{y,m} \| \bar z_t\| $, we have 
\begin{equation*}
\begin{split}
 & \| \bar u_{m,t} - \nabla_{12}^2 g_m( y_{m-1,t+1}^\star, y_{m,t+1}^\star)  \|^2
 \\
 & \leq \left ( 1 +  \frac{\beta_t}{3} \right )  \| \bar u_{m,t}  -\nabla_{12}^2 g_m( y_{m-1,t}^\star, y_{m,t}^\star )    \|^2  \\
 & \quad + \left (1 + \frac{3}{\beta_t}  \right ) \| \nabla_{12}^2 g_m( y_{m-1,t}^\star,  y_{m,t}^\star )   -\nabla_{12}^2 g_m(y_{m-1,t+1}^\star,  y_{m,t+1}^\star ) \|_F^2 
     \\
  & \leq  \left ( 1 +  \frac{\beta_t}{3} \right ) \| \bar u_{m,t}  -\nabla_{12}^2 g_m( y_{m-1,t}^\star, y_{m,t}^\star)   \|^2 
  \\
 &  \quad + \left (1 + \frac{3}{\beta_t}  \right ) \widetilde L_{g,m}^2  \Big (  \| y_{m-1,t}^\star -  y_{m-1,t+1}^\star\|^2  +  \| y_{m,t}^\star  -  y_{m,t+1}^\star \|^2 \Big  )
  \\
  & =  \left ( 1 +  \frac{\beta_t}{3} \right ) \|\bar u_{m,t}  -\nabla_{12}^2 g_m ( y_{m-1,t}^\star, y_{m,t}^\star )    \|^2 + \left  (1 + \frac{3}{\beta_t}  \right )\alpha_t^2 \widetilde L_{g,m}^2 (L_{y,m-1}^2 + L_{y,m}^2)\| \bar z_t \|^2. 
    \end{split}
\end{equation*}
By substituting  the above inequality and \eqref{eq:ut_2} into \eqref{eq:ut_1}, we obtain 
\begin{equation}\label{eq:ut_3}
    \begin{split}
       &  \EE[ \|   \bar u_{m,t+1} - \nabla_{12}^2 g_m(y_{m-1,t+1}^\star, y_{m,t+1}^\star)     \|_F^2 ]  
       \\
              & \leq (1-\beta_t)\EE [ \| \bar u_{m,t}  -\nabla_{12}^2 g_m(y_{m-1,t}^\star, y_{m,t}^\star)    \|_F^2  ] + \frac{2\beta_t^2   \sigma_{g,m}^2}{K }  
              \\
              & \quad + \frac{7  \widetilde L_{g,m}^2 }{2\beta_t }  \EE \Big [ \| y_{m-1,t}^\star - y_{m-1,t+1}^\star\|^2 +  \| y_{m,t}^\star - y_{m,t+1}^\star\|^2 \Big ] 
              \\
              & \quad + \frac{2\beta_t \widetilde L_{g,m}^2 }{K} \sum_{k \in \cK} \EE \left ( \| y_{m-1,t}^k -  y_{m-1,t+1}^\star  \|^2   + \| y_{m,t}^k -  y_{m,t+1}^\star   \|^2  \right ) .
        \end{split}
\end{equation}

We then observe that 
\begin{equation*}
\begin{split}
& 
\sum_{k \in \cK} \left ( \EE[ \| y_{m-1,t}^k -  y_{m-1,t+1}^\star  \|^2 ]  + \EE [ \| y_{m,t}^k -  y_{m,t+1}^\star   \|^2   ] \right ) 
\\
& \leq \sum_{k \in \cK} 2 \EE \left [ \|  y_{m-1,t}^k -  y_{m-1,t}^\star\|^2 + \|  y_{m-1,t}^\star - y_{m-1,t+1}^\star  \|^2  \right ] 
\\
& \quad +  \sum_{k \in \cK} 3 \EE \left [ \|  y_{m,t}^k - \bar y_{m,t} \|^2 + \| \bar y_{m,t}   -  y_{m,t}^\star  \|^2  +  \|  y_{m,t}^\star - y_{m,t+1}^\star   \|^2   \right ]
\\
&  \leq \cO \left ( \frac{ K( \alpha_t^2 +  \beta_t^2) }{(1- \rho)^2 } \right ) + 2K(L_{y,m-1}^2 + L_{y,m}^2)\EE[ \|  \bar x_t- \bar x_{t+1}  \|^2] 
\\
& \quad + 2K \EE[ \| \bar y_{m-1,t} - y_{m-1,t}^\star\|^2 ] +  3K \EE[ \| \bar y_{m,t} - y_{m,t}^\star\|^2 ] 
\\
& \leq  \cO \left ( \frac{ K( \alpha_t^2 +  \beta_t^2) }{(1- \rho)^2 } \right )  + 2K (L_{y,m-1}^2 + L_{y,m}^2)  \alpha_t^2 \EE[ \|\bar z_t  \|^2] 
\\
& \quad + 2K \EE[ \| \bar y_{m-1,t} - y_{m-1,t}^\star\|^2 ] +  3K \EE[ \| \bar y_{m,t} - y_{m,t}^\star\|^2 ], 
\end{split}
\end{equation*}
where the second inequality comes from Lemma~\ref{lemma:Lip_y_multilevel} that $\|  y_{j,t}^\star - y_{j,t+1}^\star   \|  \leq L_{y,j} \| \bar x_t - \bar x_{t+1}\| \leq \alpha_t \| \bar z_t \|$ for $j=1,\cdots, M$. By combining the above inequality with \eqref{eq:ut_3}, we conclude that 
\begin{equation*}
    \begin{split}
       &  \EE[ \|   \bar u_{m,t+1} - \nabla_{12}^2 g_m(y_{m-1,t+1}^\star, y_{m,t+1}^\star)     \|_F^2 ]  
       \\
              & \leq (1-\beta_t)\EE [ \| \bar u_{m,t}  -\nabla_{12}^2 g_m(y_{m-1,t}^\star, y_{m,t}^\star)    \|_F^2  ] + \frac{2\beta_t^2   \sigma_{g,m}^2}{K }  
              \\
              & \quad + \Big (  \frac{7 \alpha_t^2 \widetilde L_{g,m}^2 (L_{y,m-1}^2 + L_{y,m}^2)  }{2\beta_t }   + 4 (L_{y,m-1}^2 + L_{y,m}^2)  \alpha_t^2 \beta_t \widetilde L_{g,m}^2   \Big ) \EE[ \| \bar z_t \|^2]
              \\
              & \quad + 6\beta_t \widetilde L_{g,m}^2  \EE[ \| \bar y_{m,t} - y_m^\star(\bar x_{t})\|^2 ]  + 6\beta_t \widetilde L_{g,m-1}^2  \EE[ \| \bar y_{m-1,t} - y_{m-1}^\star(\bar x_{t})\|^2 ].
        \end{split}
\end{equation*}
Recall that $\alpha_t, \beta_t \to 0$ when $T\to \infty$. 
The desired result can be acquired by using the fact that $\frac{7 \alpha_t^2 \widetilde L_{g,m}^2 (L_{y,m-1}^2 + L_{y,m}^2)  }{2\beta_t }   + 4 (L_{y,m-1}^2 + L_{y,m}^2)  \alpha_t^2 \beta_t \widetilde L_{g,m}^2   
\leq \frac{4 \alpha_t^2 \widetilde L_{g,m}^2(L_{y,m-1}^2 + L_{y,m}^2) }{\beta_t }  $ when $T $ is large. 
\\
\noindent (b) This part can be derived following the similar analysis as part (a), we skip the details to avoid repetition. 
\end{proof}

\subsection{Proof of Theorem \ref{thm:nonconvex_multilevel}}\label{app:proof_nonconvex_multilevel}
To characterize the convergence properties of nonconvex multilevel problems, we assume Assumptions~\ref{assumption:W},  \ref{assumption:SO}, \ref{assumption:1}, and \ref{assumption:2} hold and the step-sizes follow~\eqref{eq:stepsize_multilevel} that
\begin{equation*}
\alpha_t = C_0 \sqrt{\tfrac{K}{T}}, \   \beta_t = \gamma_t = \sqrt{\tfrac{K}{T}}, \ \text{ and } b = 3 \lceil \log_{\frac{1}{1-\kappa_g}} T \rceil ,    \text{ for all }t=0,1,\cdots, T,
\end{equation*} 
where $C_0>0$ is a small constant and the number of iterations $T $ is large   such that $\beta_t ,\gamma_t \leq 1$ and 
\begin{equation*}
  \Upsilon(C_0,T) =   1 - \frac{C_0 \sqrt{K}L_F}{\sqrt{T}}   -  C_0^2 \Big ( \tilde C    + \sum_{m=1}^M \frac{3r_{5,m}L_{y,m}^2} {\mu_{g,m}^2 }\Big )  \geq 0 , 
\end{equation*}
with $C_2,\{ C_{3,m} \}_{m=1}^M , \{ C_{4,m,j} \}_{m=1,\cdots, M}^{b =1,\cdots,b} $ are constants defined in \eqref{def:constant_C}, 
\begin{equation}\label{eq:r_multilevel}
\begin{split}
 \tilde C & = 4  \big (L_f^2(r_1 + r_2)  + \sum_{m=1}^M \widetilde L_{g,m}^2 (r_{3,m} + \sum_{j=1}^b r_{4,m,j} ) (L_{g,m-1}^2  + L_{g,m}^2 ) \big )
 \\
r_1 & = 4, r_2 = C_2  ,r_{m,3} =  C_{m,3} , r_{m,4,j} = C_{m,4,j},
\text{ for } m \in [M], 
\\
r_{5,M} & = 6 \big  ( L_f^2(r_1 + r_2) +  \widetilde L_{g,M}^2(r_{M,3} + r_{M,4})  \big ),
\end{split}
\end{equation}
and for $m =M-1,\cdots,1$,
\begin{equation*}
\begin{split}
r_{5,m} & = \frac{r_{5,m+1}L_{g,m}^2  }{ \mu_{g,m}^2} + 6  \Big ( \tilde L_{g,m}^2 (r_{m,3} + \sum_{j=1}^b r_{m,4,j})  + \tilde L_{g,m+1}^2 (r_{m+1,3} + \sum_{j=1}^b r_{m+1,4,j})  \Big ).    
 \end{split}
\end{equation*}
Note that $C_0$ adjusts the stepsize $\alpha_t$. To satisfy the above condition, we can set $C_0$ as a small constant such that $C_0^2 \Big ( \tilde C    + \sum_{m=1}^M \frac{3r_{5,m}L_{y,m}^2}{\mu_{g,m}^2 }\Big ) \leq 1/2$ and the set total number of iterations $T \geq 4 C_0^2 K L_F^2$.
\begin{proof}
We write $y_{m,t}^\star = y_m^\star(\bar x_t)$ and start our analysis by considering the term $ \| \bar y_{m,t} -  y_{m}^\star(\bar x_t)\|^2$ for each level $m =1,\cdots,M$. By rearranging  \eqref{eq:yt_recursion_multilevel}, we have
\begin{equation}\label{eq:nonconvex_yt_multilevel}
\begin{split}
  \EE [ \| \bar y_{m,t} -  y_{m}^\star (\bar x_t) \|^2 ]& \leq \frac{\EE \left [  \| \bar y_{m,t} -  y_{m}^\star(\bar x_t)  \|^2 -   \| \bar y_{m,t+1} - y_m^\star(\bar x_{t+1})\|^2  \right ] }{\gamma_t \mu_{g,m}}  + \cO \left (\frac{ \alpha_t^2 + \gamma_t^2}{(1- \rho)^2} \right )
   \\
   & \quad  + \frac{2\gamma_t  C_{g,m}^2}{\mu_{g,m}   K}  + \frac{3L_{y,m}^2 \alpha_t^2 }{\gamma_t^2 \mu_{g,m}^2 } \EE [\| \bar z_t \|^2 ] + \frac{L_{g,m}^2 \EE[ \| \bar y_{m-1,t} - y_{m-1}^\star (\bar x_t)   \|^2] }{ \mu_{g,m}^2} .
   \end{split}
\end{equation}
Letting $r_1 = 4, r_2 = C_2  ,r_{m,3} =  C_{m,3} , r_{m,4,j} = C_{m,4,j} $, and $r_{5,M} = 6 \big  ( L_f^2(r_1 + r_2) +  \widetilde L_{g,M}^2(r_{M,3} + r_{M,4})  \big )  $ be the constants defined within \eqref{eq:r_multilevel},
we define a random variable
\begin{equation*}
\begin{split}
    P_t  & =   \tfrac{2}{\alpha_t}F(\bar x_t)  + \tfrac{r_1}{\beta_t}\| \bar s_t - \nabla_1 f (   \bar x_{t}, y_{M}^\star (\bar x_t)  ) \|^2 
    +  \tfrac{r_2}{\beta_t}\| \bar h_t - \nabla_2 f (   \bar x_{t}, y_{M}^\star  (\bar x_t)  )  \|^2
    \\
    & \quad + \sum_{m=1}^M \tfrac{r_{m,3}}{\beta_t}\| \bar u_{m,t} - \nabla_{12}^2  g_m (  y_{m-1,t}^\star , y_{m}^\star (\bar x_t)   ) \|_F^2  + \sum_{m=1}^M \sum_{j=1}^b \tfrac{r_{m,4,j}}{\beta_t} 
    \| \bar v_{m,j,t} - \nabla_{22}^2 g_m (  y_{m-1,t}^\star , y_{m}^\star (\bar x_t)   ) \|_F^2 
    .
    \end{split}
\end{equation*}
Here we observe that $P_0 \leq \cO \left (\sqrt{ \frac{T}{K}}\right)$ and $P_t \geq \frac{2}{\alpha_t} F(x^*) = \frac{2}{C_0} \sqrt{ \frac{T}{K}} F(x^*) $.
By multiplying $r_1 \beta_t^{-1}$, $r_2\beta_t^{-1} $, $r_{3,m} \beta_t^{-1}$, and $r_{4,j}\beta_t^{-1}$ to both sides of \eqref{eq:s_multilevel}, \eqref{eq:h_multilevel}, \eqref{eq:u_multilevel}, and \eqref{eq:v_multilevel}, respectively, and  combining with \eqref{eq:nonconvex_main_multilevel}, we obtain 
   \begin{equation*}
\begin{split}  
& \EE[ \| \nabla F(\bar x_t)\|^2]  + \EE [ P_{t+1} ]
\\
& \leq \EE[P_t]
 + \cO \left ( \frac{ \alpha_t^2 + \beta_t^2 }{(1- \rho)^2}\right ) +  r_{5,M} \EE[ \| \bar y_{M,t} - y_{M}^\star (\bar x_t)  \|^2 ] 
 \\
 & \quad +\sum_{m=1}^{M-1} 6  \Big ( \tilde L_{g,m}^2 (r_{m,3} + \sum_{j=1}^b r_{m,4,j})  + \tilde L_{g,m+1}^2 (r_{m+1,3} + \sum_{j=1}^b r_{m+1,4,j})  \Big ) \EE[ \| \bar y_{m,t} -y_m^\star(\bar x_t)\|^2] 
 \\
 & \quad + \cO \Big (  \frac{M \beta_t   }{K }   \Big )
- \left ( 1 - \alpha_t L_F   - \frac{\alpha_t^2  \tilde C }{\beta_t^2 }  \right ) \EE[  \|\bar z_t \|^2],
  \end{split}
\end{equation*}
where 
\begin{equation*}
    \begin{split}
        & r_{5,M} = 6 \big  ( L_f^2(r_1 + r_2) +  L_{g,M}^2(r_{3,M} + \sum_{j=1}^b r_{4,M,j})  \big ) 
        \\
    \text{ and }    &  \tilde C = 4  \big (L_f^2(r_1 + r_2)  + \sum_{m=1}^M \widetilde L_{g,m}^2 (r_{3,m} + \sum_{j=1}^b r_{4,m,j} ) (L_{g,m-1}^2  + L_{g,m}^2 ) \big )
    \end{split}
\end{equation*}
are constants.
By multiplying $r_{5,M}$ to both sides of \eqref{eq:nonconvex_yt_multilevel} with $m=M$, we have 
   \begin{equation*}
\begin{split}  
& \EE[ \| \nabla F(\bar x_t)\|^2]  + \EE [ P_{t+1} ] +  \frac{r_{5,M}}{\gamma_t \mu_{g,M}} \EE [ \| \bar y_{M,t+1} -  y_M^\star(\bar x_{t+1})\|^2 ]
\\
& \leq \EE[P_t] +  \frac{r_{5,M}}{\gamma_t  \mu_{g,M}} \EE [ \| \bar y_{M,t} -  y_M^\star(\bar x_{t})\|^2 ]
 + \cO \left ( \frac{ \alpha_t^2 + \beta_t^2 }{(1- \rho)^2}\right ) 
 +  \frac{r_{5,M}L_{g,M-1}^2  }{ \mu_{g,M-1}^2}  \EE[ \| \bar y_{M-1,t} - y_{M-1}^\star (\bar x_t)  \|^2]
 \\
 & \quad +\sum_{m=1}^{M-1} 6  \Big ( \tilde L_{g,m}^2 (r_{3,m} + \sum_{j=1}^b r_{4,m,j})  + \tilde L_{g,m+1}^2 (r_{3,m+1} + \sum_{j=1}^b r_{4,m+1,j})  \Big ) \EE[ \| \bar y_{m,t} -y_m^\star(\bar x_t)\|^2] 
 \\
 & \quad + \cO \Big (  \frac{M \beta_t   }{K }   \Big )
- \left ( 1 - \alpha_t L_F   - \frac{\alpha_t^2  \tilde C }{\beta_t^2 }   - \frac{3r_{5,M}L_{y,M}^2 \alpha_t^2 }{\gamma_t^2 \mu_{g,M}^2 } \right ) \EE[  \|\bar z_t \|^2]. 
  \end{split}
\end{equation*}
We then recursively apply the above process for $m=M-1,\cdots,1$ as follows.
(i) set $r_{5,m} = \frac{r_{5,m+1}L_{g,m}^2  }{ \mu_{g,m}^2} + 6  \Big ( \tilde L_{g,m}^2 (r_{m,3} + \sum_{j=1}^b r_{m,4,j})  + \tilde L_{g,m+1}^2 (r_{m+1,3} + \sum_{j=1}^b r_{m+1,4,j})  \Big )$; (ii) multiply $r_{5,m}$ to both sides of \eqref{eq:nonconvex_yt_multilevel} and combine with the above inequality. This process leads to 

  \begin{equation*}
\begin{split}  
& \EE[ \| \nabla F(\bar x_t)\|^2]  + \EE [ P_{t+1} ] + \sum_{m=1}^M \frac{r_{5,m}}{\gamma_t \mu_{g,m}} \EE [ \| \bar y_{m,t+1} -  y_m^\star(\bar x_{t+1})\|^2 ]
\\
& \leq \EE[P_t] +  \sum_{m=1}^M \frac{r_{5,m}}{\gamma_t \mu_{g,m}}\EE [ \| \bar y_{m,t} -  y_m^\star(\bar x_{t})\|^2 ]
 + \cO \left ( \frac{ M(\alpha_t^2 + \beta_t^2) }{(1- \rho)^2}\right ) 
 \\
 & \quad + \cO \Big (  \frac{M \beta_t   }{K }   \Big )
- \left ( 1 - \alpha_t L_F    - \frac{\alpha_t^2  \tilde C }{\beta_t^2 }   - \sum_{m=1}^M \frac{3r_{5,m}L_{y,m}^2 \alpha_t^2 }{\gamma_t^2 \mu_{g,m}^2 } \right ) \EE[  \|\bar z_t \|^2]. 
  \end{split}
\end{equation*}
Recall that $\alpha_t = C_0\sqrt { \frac{K}{T} } $ and $\beta_t = \gamma_t = \sqrt { \frac{K}{T} } $, by substituting the step-sizes into the above inequality, we further obtain 
  \begin{equation*}
\begin{split}  
& \EE[ \| \nabla F(\bar x_t)\|^2]  + \EE [ P_{t+1} ] + \sum_{m=1}^M \frac{r_{5,m}}{\mu_{g,m}}\sqrt{\frac{T}{K}} \EE [ \| \bar y_{m,t+1} -  y_m^\star(\bar x_{t+1})\|^2 ]
\\
& \leq \EE[P_t] +  \sum_{m=1}^M \frac{r_{5,m}}{\mu_{g,m}}\sqrt{\frac{T}{K}} \EE [ \| \bar y_{m,t} -  y_m^\star(\bar x_{t})\|^2 ]
 + \cO \left ( \frac{ M(\alpha_t^2 + \beta_t^2) }{(1- \rho)^2}\right ) 
 \\
 & \quad + \cO \Big (  \frac{M \beta_t   }{K }   \Big )
- \underbrace{ \left ( 1 - \frac{C_0 \sqrt{K}L_F}{\sqrt{T}}   -  C_0^2 \Big ( \tilde C    + \sum_{m=1}^M \frac{3r_{5,m}L_{y,m}^2} {\mu_{g,m}^2 }\Big )   \right ) }_{\Upsilon(C_0,T) }\EE[  \|\bar z_t \|^2]. 
  \end{split}
\end{equation*}
We then observe that for large $T$, there exists a small constant $\tilde C_0 >0$ such that $\Upsilon(C_0,T) \geq0$ for all $C_0 \leq \tilde C_0$.
 In such scenario, we sum the above inequality over $t=0,1,\cdots, T-1$ and conclude that 
  \begin{equation*}
\begin{split}
& \frac{1}{T}\sum_{t=0}^{T-1} \EE[ \| \nabla F(\bar x_t)\|^2] 
\\
& \leq \frac{P_0 + \sum_{m=1}^M \frac{r_{5,m}}{\mu_{g,m}}\sqrt{\frac{T}{K}}  \| \bar y_{m,0} -  y_m^\star(\bar x_{0})\|^2  - \EE[P_{T}] }{T}  +  \cO \Big ( \frac{M }{\sqrt{ TK} } \Big )   + \cO \left ( \frac{KM}{T (1- \rho)^2 }\right ) 
\\
 & \leq  \cO \Big ( \frac{M}{\sqrt{ TK} } \Big )   + \cO \left ( \frac{KM}{T (1- \rho)^2 }\right ),
  \end{split}
\end{equation*}
where the last inequality applies the facts $P_0 +  \sum_{m=1}^M \frac{r_{5,m}}{\mu_{g,m}}\sqrt{\frac{T}{K}}  \| \bar y_{m,0} -  y_m^\star(\bar x_{0})\|^2 \leq \cO(M \sqrt{ \frac{T}{K} })$ and $P_T \geq \frac{2}{C_0} \sqrt{ \frac{T}{K} }  F(x^*)$. 
This completes the proof. 
\end{proof}


\section{Proof of Results for $\mu$-PL Objectives}\label{app:proof_PL_multilevel}
To characterize the convergence properties of $\mu$-PL multilevel problems, we assume Assumptions~\ref{assumption:SO}, \ref{assumption:W},   \ref{assumption:1}, \ref{assumption:2}, and \ref{assumption:PL}  hold. We set $b = 3 \lceil \log_{\frac{1}{1-\kappa_g}} T \rceil $, consider the scenario where step-sizes follow \eqref{eq:stepsize_PL} such that 
\begin{equation*}
\alpha_t = \frac{2}{\mu(C_1 + t)},    \text{ and } \beta_t =  \gamma_t = \frac{C_1}{ C_1 + t},  \ \ \text{ for } 1\leq t \leq T,
\end{equation*}
where $C_1 >0$ is  a large constant making 
$$
\Psi(C_1) =  \frac{1}{2} -  \frac{ 2L_F}{\mu C_1 }
  - \frac{2 \hat C }{ \mu C_1  }  -  \sum_{m=1}^{M}\frac{6z_{5,m} L_{y,m}^2 }{ \mu_{g,m} \mu C_1 }  > 0 , 
$$
where $C_2,\{ C_{3,m} \}_{m=1}^M , \{ C_{4,m,j} \}_{m=1,\cdots, M}^{b =1,\cdots,b} $ are constants defined in \eqref{def:constant_C}, 
\begin{equation}\label{def:z_multilevel}
    \begin{split}
    \hat C & = 4 \big (L_f^2(z_1 + z_3)  + L_g^2 (z_{3,M} + \sum_{1\leq j \leq b}z_{4,M,j} )  \big )(1 + L_{y,M}^2 ), 
   z_1  = \frac{4}{\mu( C_1 - 2 ) } , 
   \\  z_2  & = \frac{2 C_2  }{\mu( C_1 -2 )} ,   z_{3,m}   = \frac{2C_{3,m}}{ \mu( C_1   - 2 ) }, 
 z_{4,m,j}  =\frac{2C_{4,m,j}}{ \mu( C_1   - 2) }  ,  \forall 1\leq m \leq M,
    \end{split}
\end{equation}
and 
\begin{equation}\label{def:Z_multilevel}
    \begin{split}
       Z_{M} & =  6\big  ( L_f^2(z_1 + z_2) +  L_{g,M}^2(z_{3,M} + \sum_{j=1}^b z_{4,M,j})  \big ),   z_{5,M}  =   \frac{ Z_{M}}{\mu_{g,M} (C_1 - 2/\mu)},
   \\ Z_{m} & = 6  \Big ( \tilde L_{g,m}^2 (z_{3,m} + \sum_{j=1}^b z_{4,m,j})  + \tilde L_{g,m+1}^2 (z_{3,m+1} + \sum_{j=1}^b z_{4,m+1,j})  \Big ), 
   \\
   z_{5,m} & = (Z_m + \frac{z_{5,m+1} L_{g,m+1}^2 }{ \mu_{g,m+1}})(\mu_{g,m}(C_1 - 2/\mu) )^{-1}, \forall 1\leq m \leq M-1. 
    \end{split}
\end{equation}


\subsection{Lemma \ref{lemma:PL_01} and Its Proof} \label{app:main_PL}

\begin{lemma}\label{lemma:PL_01}
Suppose Assumptions \ref{assumption:SO}, \ref{assumption:W}, \ref{assumption:1}, \ref{assumption:2}, and \ref{assumption:PL}  hold. We have 
 \begin{equation}\label{eq:PL_multilevel_01}
\begin{split} 
	  & \EE[F(\bar{x}_{t+1})]  - F^* 
	  \\
	  &  \leq \big (1   - \alpha_t \mu  \big)\E[ F(\bar x_t) -F^*]  
  -  \frac{\alpha_t}{2}(1 - \alpha_t L_F ) \EE[ \|  \bar z_t  
\|^2  ]   + 2\alpha_t  \EE [ \|  \nabla_1 f( \bar x_t ,   y_{M,t}^\star    ) - \bar s_t \|^2 ]  
\\
 & \quad +  \frac{C_{2}\alpha_t }{2} \EE[  \| \nabla_2 f(  \bar x_t  ,   y_{M,t}^\star  ) -  \bar h_t \|^2 ]    + \sum_{m=1}^M  \frac{C_{3,m}\alpha_t }{2} \EE[ \|  \nabla_{12}^2 g_m( y_{m-1,t}^\star  , y_{m,t}^\star   ) - \bar u_{m,t} \|_F^2 ]   
 \\
 & \quad 
+    \sum_{m=1}^M \sum_{j=1}^b \frac{C_{4,m,j}\alpha_t }{2} \EE[ \| \nabla_{22}^2 g_m(y_{m-1,t}^\star, y_{m,t}^\star  ) -  \bar v_{m,t,j} \|_F^2]  + \cO \left ( \frac{M\alpha_t \beta_t^2 }{ (1- \rho)^2 }\right )
. 
  \end{split}
\end{equation}
where $C_{2} , C_{3,m} ,C_{4,m,j}$ are constants defined in \eqref{def:constant_C}. 
\end{lemma}
\begin{proof}
Suppose the objective function $F$ satisfies the $\mu$-PL condition \eqref{def:PL} that $2 \mu (F(\bar x_t ) - F^*) \leq \| \nabla F(\bar x_t)\|^2$, by combining it with \eqref{eq:multilevel_eq0}, we have 
\begin{equation*}
\begin{split}
& \EE[ F(\bar{x}_{t+1}) ]  - \EE [ F(\bar{x}_{t}) ]
\\
& \leq   - \alpha_t \mu \EE  [ F(\bar x_t ) - F^*]   - \frac{\alpha_t}{2}(1-\alpha_t L_F) \EE [ \|  \bar z_t  
\|^2 ]   + \frac{\alpha_t}{2} \EE[ \| \nabla F(\bar x_t) - \bar z_t    \|^2  ].
\end{split}
\end{equation*}
Subtracting $F^*$ on both sides, we obtain 
\begin{equation*}
\begin{split}
& \EE[ F(\bar{x}_{t+1}) ]  -  F^*
\\
& \leq  ( 1 - \alpha_t \mu ) \EE  [ F(\bar x_t ) - F^*]   - \frac{\alpha_t}{2}(1-\alpha_t L_F) \EE [ \|  \bar z_t  
\|^2 ]   + \frac{\alpha_t}{2} \EE[ \| \nabla F(\bar x_t) - \bar z_t    \|^2  ].
\end{split}
\end{equation*}
The desired result can be acquired by applying the bound of $\EE[\| \nabla F(\bar x_t)\ - \bar z_t\|^2]$ in~\eqref{eq:z_error} and applying the convergence of consensus errors in Lemma~\ref{lemma:consensus_multilevel}. 
\end{proof}
\subsection{Proof of Theorem \ref{thm:PL}}\label{app:proof_of_thm_PL}
Before establishing the convergence rate for $\mu$-PL function, we provide a result \citep[Lemma 1]{ghadimi2016accelerated}   to characterize the convergence behavior for a random sequence satisfying a  special form of stochastic recursion as follows.

\begin{lemma}\label{lemma:sequence_alpha}
 Letting  $b_k = \frac{2}{k+1}$  and $\Gamma_k = \frac{2}{k(k+1)}$ for $k \geq 1$ be two nonnegative sequences.
For any nonnegative sequences $\{ A_k \}$ and $\{ B_k \}$  satisfying
$$
A_k \leq (1-b_k) A_{k-1} + B_k , \ \ \text{ for } k \geq 1,
$$
we have $\Gamma_k =  \Gamma_{s} \prod_{j=s+1}^k (1-b_j)$ and 
\begin{equation*}
    A_k \leq  \frac{\Gamma_k}{\Gamma_s} A_{s} + \sum_{i=s+1}^k \frac{   \Gamma_k B_i }{\Gamma_i}.
\end{equation*}
\end{lemma}
\noindent We then derive the convergence rate of $\{ \bar x_t\}$ for $\mu$-PL objectives.

\begin{proof}
We write $y_{m,t}^\star = y_m^\star(\bar x_t)$ for $m=1,\cdots,M$. 
First, under the choice of step-sizes that  $\alpha_t  = \frac{2}{\mu(C_1 + t)}$ and $\beta_t = \gamma_t = \frac{C_1}{ C_1 + t}$, we have  $\lim_{t\to \infty} \alpha_t  = 0$, $\lim_{t\to \infty} \beta_t = 0$, and  $\lim_{t\to \infty} \gamma_t = 0$. By following the analysis of Lemmas~\ref{lemma:y_multilevel}, \ref{lemma:error_multilevel}, and \ref{lemma:ut}, and applying the convergence rates of consensus errors in Lemma~\ref{lemma:consensus_multilevel},  we obtain that 
\eqref{eq:yt_recursion_multilevel}, \eqref{eq:s_multilevel}, \eqref{eq:h_multilevel}, \eqref{eq:u_multilevel}, \eqref{eq:v_multilevel} still hold under this choice of step-size. 

Next, we define a random variable 
\begin{align*}
J_k & =  F(\bar x_t)  -F^* +  z_1   \|  \nabla_1 f( \bar x_t , y_{M,t}^\star    ) - \bar s_t \|^2   + z_2  \| \nabla_2 f(  \bar x_t  ,   y_{M,t}^\star    ) -  \bar h_t \|^2   
\\
& \quad + \sum_{m=1}^M \Big ( z_{3,m}   \|  \nabla_{12}^2 g_m( y_{m-1,t}^\star, y_{m,t}^\star    ) - \bar u_{m,t}  \|_F^2      + \sum_{1\leq j \leq b}z_{4,m,j} \| \nabla_{22}^2 g_m(y_{m-1,t}^\star, y_{m,t}^\star  ) -  \bar v_{m,t,j} )  \|_F^2\Big  ) ,
\end{align*}
where 
\begin{align*}
 z_1 = \frac{2\alpha_t}{\beta_t  - \alpha_t \mu } = \frac{4}{\mu( C_1 - 2 ) } ,  & z_2 = \frac{\alpha_t C_2 }{\beta_t  - \alpha_t \mu } = \frac{2 C_2  }{\mu( C_1 -2 )}, 
 \\
 z_{3,m}  = \frac{\alpha_t C_{3,m} }{\beta_t  - \alpha_t \mu } = \frac{2C_{3,m}}{ \mu( C_1   - 2 ) }, 
 \text{ and } & z_{4,m,j} = \frac{\alpha_t C_{4,m,j} }{\beta_t  - \alpha_t \mu }    =\frac{2C_{4,m,j}}{ \mu( C_1   - 2) },
\end{align*}
are all constants defined in \eqref{def:z_multilevel}. By multiplying $z_1, z_2 , z_{3,m}$, and $z_{4,m,j}$ to  both sides of \eqref{eq:s_multilevel}, \eqref{eq:h_multilevel}, \eqref{eq:u_multilevel}, and \eqref{eq:v_multilevel}, respectively,  and combining them with \eqref{eq:PL_multilevel_01}, we obtain 
\begin{equation*}
	\begin{split}
	 \EE[ J_{t+1}	]	
	& \leq  (1   - \alpha_t \mu ) \EE[ J_t] +  \cO \Big (  \frac{\beta_t^2 }{K }   \Big )    +  \cO \left ( \frac{1}{t^3 (1-\rho)^2}\right ) 
  - \alpha_t \left ( \frac{1-\alpha_t L_F}{2}
  - \frac{4\hat C}{\beta_t }  \right ) \EE[  \|\bar z_t \|^2]
  \\
  & \quad +  Z_{M} \beta_t \EE[ \| \bar y_{M,t} - y_{M,t}^\star\|^2 ]  +\sum_{m=1}^{M-1} Z_{m} \beta_t \EE[ \| \bar y_{m,t} -y_m^\star(\bar x_t)\|^2] 
 .
\end{split}
\end{equation*}
where $\hat C, Z_M$, and $\{ Z_m \}_{m=1}^{M-1}$ are constants defined in \eqref{def:Z_multilevel}. 
By recalling that $\alpha_t = \frac{2}{\mu(C_1 + t)}$ and $\beta_t = \gamma_t =  \frac{C_1}{C_1 + t}$, we have $\alpha_t/\beta_t =  \alpha_t/\gamma_t  =  \frac{2}{\mu C_1}$.
Further, letting 
$$z_{5,M} = \frac{  Z_{M} \beta_t }{\mu_{g,M} (\gamma_t - \alpha_t)} =  \frac{ Z_{M}}{\mu_{g,M} (C_1 - 2/\mu)},$$
by multiplying $z_{5,M}  $ to both sides of \eqref{eq:yt_recursion_multilevel} with $m =M$ and combining with the above inequality, we have 
\begin{equation*}
	\begin{split}
	 & \EE[ J_{t+1}] 	 + z_{5,M} \EE[ \| \bar y_{M,t+1} - y_{M,t+1}^\star\|^2]
	 \\
	& \leq  \left (1   - \tfrac{2}{C_1 + t+1}  \right ) \Big ( \EE[ J_t ]  +z_{5,M} \EE[ \| \bar y_{M,t} - y_{M,t}^\star\|^2] \Big )  +  \cO \Big (  \frac{\beta_t^2 }{K }   \Big )    +  \cO \left ( \frac{1}{t^3 (1-\rho)^2 }\right ) 
  \\
  & \quad   - \alpha_t \left ( \frac{1}{2} -  \frac{ 2L_F}{\mu C_1 }
  - \frac{2\hat C }{ \mu_g \mu C_1 }  - \frac{6z_{5,M} L_{y,M}^2}{\mu_{g,M} \mu C_1}\right )\EE[  \|\bar z_t \|^2]
\\
& \quad + 
\frac{ z_{5,M} \gamma_{t} L_{g,M}^2 }{ \mu_{g,M}} \EE[ \| \bar y_{M-1,t} - y_{M-1,t}^\star \|^2]   +\sum_{m=1}^{M-1} Z_{m} \beta_t \EE[ \| \bar y_{m,t} -y_m^\star(\bar x_t)\|^2] 
 .
\end{split}
\end{equation*}
Recall that $\beta_t = \gamma_t$,
by recursively multiplying $z_{5,m} = \frac{\beta_t   }{\mu_g (\gamma_t - \alpha_t)}(Z_m + \frac{z_{5,m+1} L_{g,m+1}^2 }{ \mu_{g,m+1}})  = (Z_m + \frac{z_{5,m+1} L_{g,m+1}^2 }{ \mu_{g,m+1}})(\mu_{g,m}(C_1 - 2/\mu) )^{-1}$ for $m =M-1,\cdots, 1$ to both sides of \eqref{eq:yt_recursion_multilevel} ,  and combining with the above inequality, we obtain that  
\begin{equation*}
	\begin{split}
	 & \EE[ J_{t+1}] 	 + \sum_{m=1}^M z_{5,m} \EE[ \| \bar y_{m,t+1} - y_{m,t+1}^\star\|^2]
	 \\
	& \leq  \left (1   - \tfrac{2}{C_1 + t+1}  \right ) \Big ( \EE[ J_t ]  +\sum_{m=1}^M z_{5,m} \EE[ \| \bar y_{m,t} - y_{m,t}^\star\|^2]\Big )  +  \cO \Big (  \frac{M \beta_t^2 }{K }   \Big )    +  \cO \left ( \frac{1}{t^3 (1-\rho)^2 }\right ) 
  \\
  & \quad   - \alpha_t  \underbrace{ \left ( \frac{1}{2} -  \frac{ 2L_F}{\mu C_1 }
  - \frac{2 \hat C }{ \mu C_1  }  -  \sum_{m=1}^{M}\frac{6z_{5,m} L_{y,m}^2 }{ \mu_{g,m} \mu C_1 } \right )}_{\Psi(C_1)}\EE[  \|\bar z_t \|^2]
 .
\end{split}
\end{equation*}
Clearly, there exists a constant $\tilde C_1$ such that for all $\Psi(C_1) \geq 0$ for all $C_1 \geq \tilde C_1$, which further leads to 
\begin{equation*}
	\begin{split}
	 & \EE[ J_{t+1}] 	 + \sum_{m=1}^M z_{5,m} \EE[ \| \bar y_{m,t+1} - y_{m,t+1}^\star\|^2]
	 \\
	& \leq  \left (1   - \tfrac{2}{C_1 + t+1}  \right ) \Big ( \EE[ J_t ]  +\sum_{m=1}^M z_{5,m} \EE[ \| \bar y_{m,t} - y_{m,t}^\star\|^2]\Big )  +  \cO \Big (  \frac{M \beta_t^2 }{K }   \Big )    +  \cO \left ( \frac{1}{t^3 (1-\rho)^2 }\right ) 
	\\
		& \leq \frac{\Gamma_{C_1 + t}}{\Gamma_{C_1 }}[ J_0  + \sum_{m=1}^M z_{5,m}  \| \bar y_{m,0} - y_{m,0}^\star\|^2]  + \sum_{j=C_1}^{C_1 + t} \cO \Big ( \frac{M \beta_j^2 \Gamma_{C_1 + t}}{K \Gamma_j} \Big)   + \sum_{j = C_1 }^{t+ C_1 } \cO \left (\frac{\Gamma_{t + C_1}}{j^3(1-\rho)^2 \Gamma_j} \right )
		\\
& \leq \cO \Big ( \frac{M}{(t+1) K} \Big ) + \cO \left ( \frac{M \ln t}{t^2 (1-\rho)^2 } \right ),
\end{split}
\end{equation*}
completing the proof. 
\end{proof}


\section{Additional Numerical Experiments}\label{app:numerics}

\subsection{Hyper-parameter Optimization} \label{app:hyper_opt}

The baseline algorithm DBSA conducts the followings. 
At the outer solution  $x_t^k$, DBSA obtains an estimator $y_t^k$ of $y^\star(x_k)$ via conducting $t$ gossip stochastic gradient descent steps
$\tilde y_{t,i+1}^k = \sum_{j \in \cN_k}w_{k,j}\tilde y_{t,i}^{j} - \eta_{t,i} \nabla_y g(x_t^k, \tilde y_{t,i}^k;\xi_i^k) $ for $i=0,1,\cdots, t$ with $y_t^k = \tilde y_{t,t}^k$, and then update the main solution $x_t^k$ by one stochastic gradient descent step that $x_{t+1}^k = \sum_{j \in \cN_k}w_{k,j}x_t^{j} - \alpha_t \nabla_x  f(x_t^k,y_t^k;\zeta_t^k)$.  We summarize the details in Algorithm~\ref{alg:DBSA}. 

We test our algorithm over various networks to further investigate its performance. Specifically, we test our algorithm over fully connected networks with $K=5,10,20,100$ and  randomly generated connected networks with $K=5,10,20$. We summarize the details in Figures~\ref{fig:aus_fullyconnected} and \ref{fig:aus_random}, respectively. These numerical results indicate the efficiency and robustness of our algorithm over large networks  with  various topologies.

\begin{figure}[t]
\begin{minipage}{0.5\textwidth}
    \centering
    \includegraphics[width=0.8\linewidth]{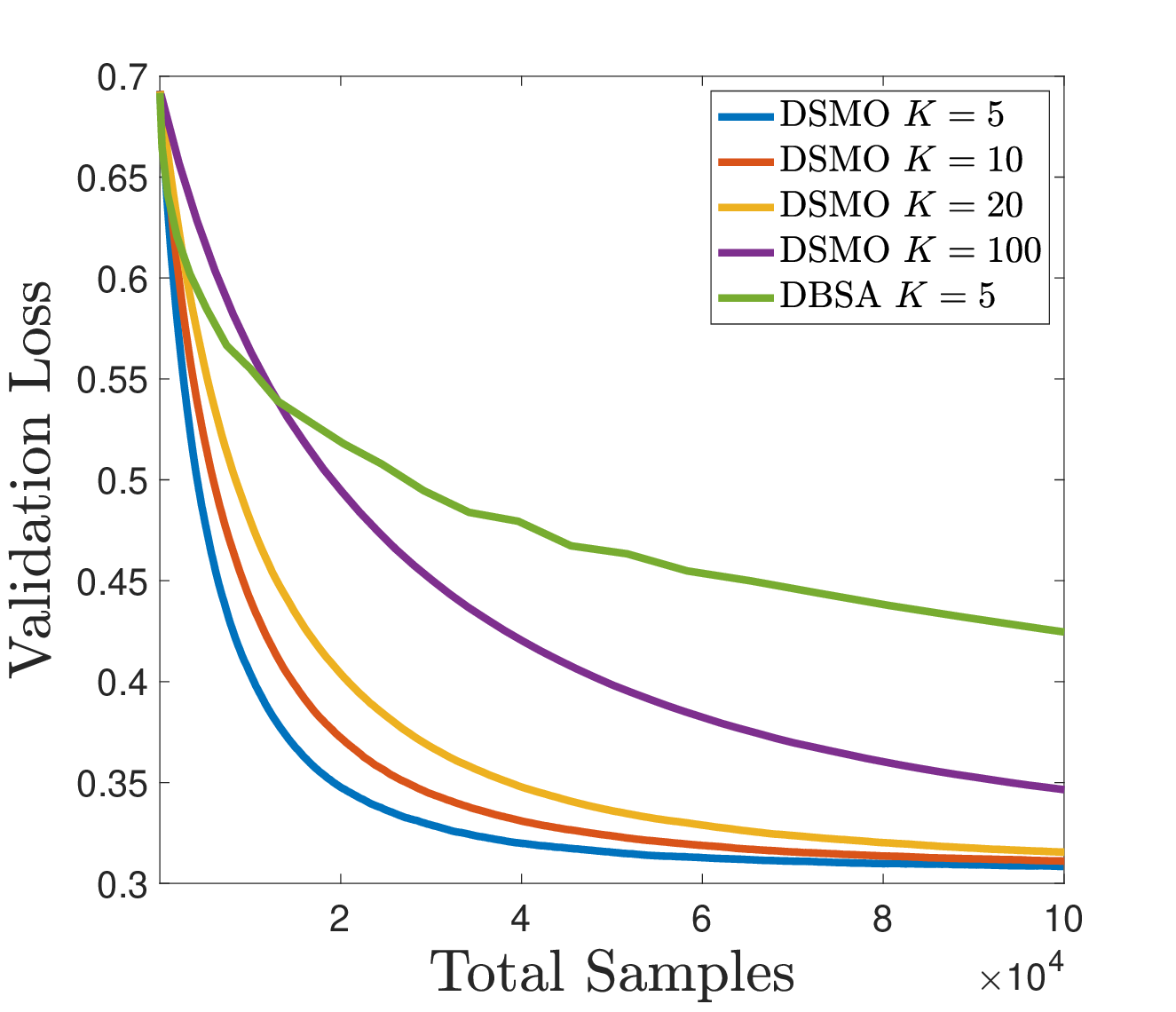}\\
    (a)
\end{minipage}
\begin{minipage}{0.5\textwidth}
    \centering
    \includegraphics[width=0.8\linewidth]{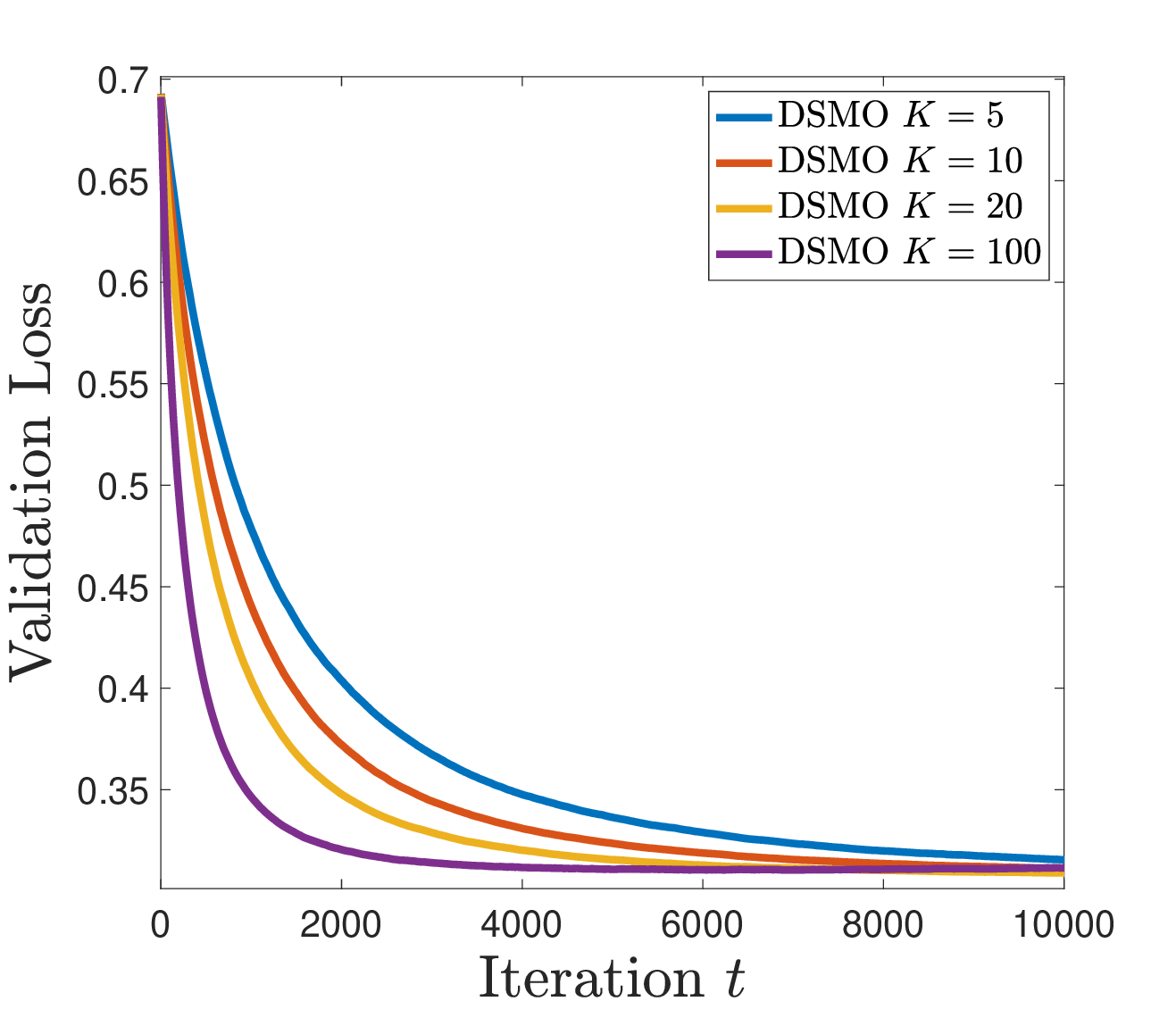}\\
    (b)
\end{minipage}
\caption{Performances of Algorithm~\ref{alg:1} and DSBA over a uniform fully connected network: (a) Empirical averaged training loss against total samples  for DSMO $K=5,10,20,100$ and DSGD $K = 5$ 
(b) Empirical averaged validation loss against iteration for DSMO $K=5,10,20,100$. All figures are generated through 10 independent simulations over the Australia handwriting dataset.
}
\label{fig:aus_fullyconnected}
\end{figure}

\begin{figure}[t]
\begin{minipage}{0.5\textwidth}
    \centering
    \includegraphics[width=0.8\linewidth]{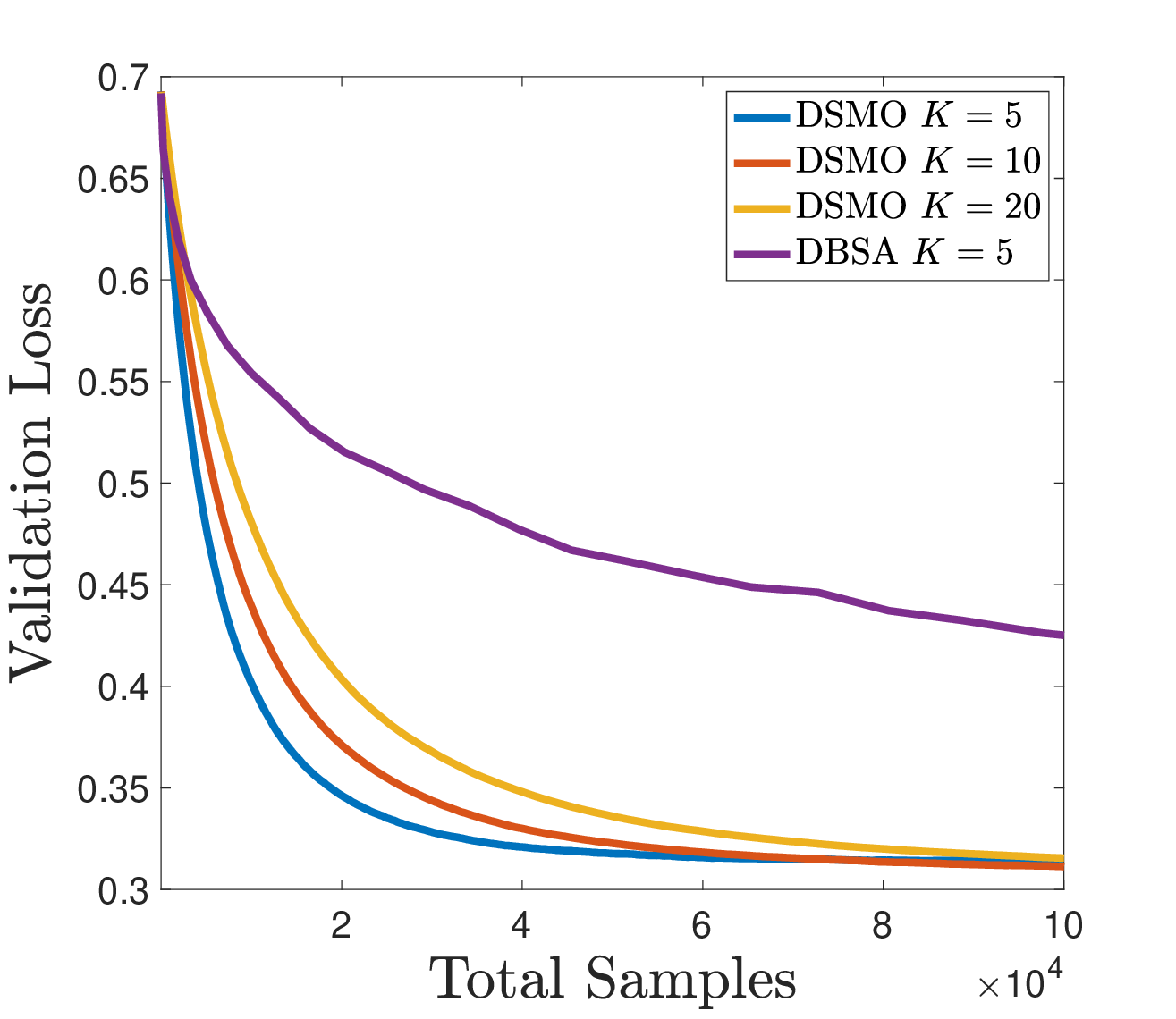}
    \\
    (a)
\end{minipage}
\begin{minipage}{0.5\textwidth}
    \centering
    \includegraphics[width=0.8\linewidth]{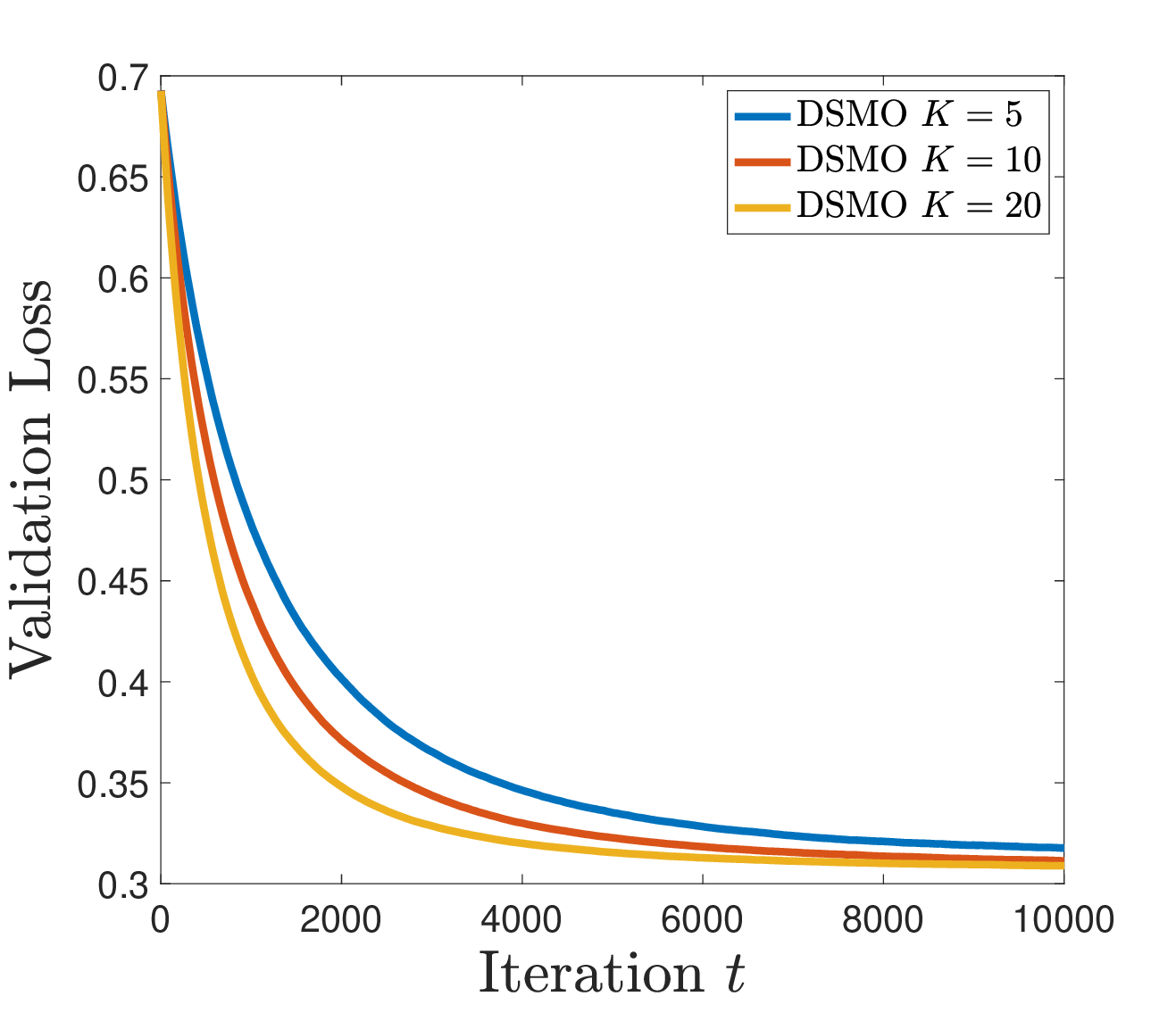}\\
    (b)
\end{minipage}
\caption{Performances of Algorithm~\ref{alg:1} and DSBA over a  connected random network: (a) Empirical averaged training loss against total samples  for DSMO $K=5,10,20$ and DBSA $K = 5$. 
(b) Empirical averaged validation loss against iteration for DSMO $K=5,10,20$. All figures are generated through 10 independent simulations over the Australia handwriting dataset.
}
\label{fig:aus_random}
\end{figure}

\begin{algorithm}[t!]
\small 
\caption{Decentralized Bilevel Stochastic Approximation}\label{alg:DBSA}
\begin{algorithmic}[1]
\REQUIRE Step-sizes $\{ \alpha_t \}$,  $\{ \eta_{t,i} \}$, number of total iterations $T$.
\\
$x_0^k = \bf{0}$, $y_0^k = \bf{0}$
\FOR{$t=0, 1, \cdots, T-1$}
	\STATE 	\textcolor{blue}{Inner loop update}:  
	\FOR{ $i = 0,1,\cdots, t$}
	\FOR{$k = 1,2,\cdots, K$}
		\STATE \textcolor{blue}{Local sampling}: Query $\cS\cO$ at $(x_t^k, \tilde y_{t,i}^k)$ to obtain $\nabla_y g(x_t^k, \tilde y_{t,i}^k;\xi_{t,i}^k)  $.
	\STATE \textcolor{blue}{Estimate:} $\tilde y_{t,i+1}^k = \sum_{j \in \cN_k}w_{k,j}\tilde y_{t,i}^{j} - \eta_{t,i} \nabla_y g(x_t^k, \tilde y_{t,i}^k;\xi_{t,i}^k) $.
 	\ENDFOR
 	\ENDFOR 
 
	\STATE \textcolor{blue}{Outer loop update}: $x_{t+1}^k = \sum_{j \in \cN_k}w_{k,j}x_t^{j} - \alpha_t \nabla_x  f(x_t^k,y_t^k;\zeta_t^k)$.
 \ENDFOR
 \ENSURE $ \bar x_{T} = \frac{1}{K} \sum_{k\in \cK} x_T^k$.
\end{algorithmic}
\end{algorithm}

\subsection{Distributed Policy Evaluation for Reinforcement Learning} \label{app:MDP}

\textbf{Simulation environment:} In our experiments, for each state $s \in \cS$, we generate its feature $\phi_s \sim \text{Unif}[0,1]^m$;
we uniformly generate the transition probabilities $p_{s,s'}$ and standardize them such that $\sum_{s' \in \cS} p_{s,s'} = 1$; we sample the mean of rewards $\bar r_{s,s'}^k \sim \text{Unif}[0,1]$ for all $s \in \cS$ and each agent $k \in [K]$. We set the regularizer parameter $\lambda = 1$.

In each simulation, we set $|\cS| = 100$ and update the solution $(x^k,y^k)$ for each agent in a parallel manner as follows:
At iteration $t$, for each state $s \in \cS$, we simulate a random transition to another state $s' \in \cS$ using the transition probability $p_{s,s'}$'s, generate a random reward $r_{s,s'}^k \sim \cN(\bar r_{s,s'}^k, 1)$, and update $x_t^k$ using step-sizes $\alpha_t = \min \{ 0.01,\frac{2}{\lambda t} \}$ and $\beta_t = \gamma_t = \min \{ 0.5, \frac{50}{t}\}$.

\begin{algorithm}[t!]
\small 
\caption{Decentralized Stochastic Gradient Descent}\label{alg:DSGD}
\begin{algorithmic}[1]
\REQUIRE Step-sizes $\{ \alpha_t \}$, weights  $ \{ \eta_{t,i}\}$, number of total iterations $T$.
\\
$x_0^k = \bf{0}$, $y_0^k = \bf{0}$
\FOR{$t=0, 1, \cdots, T-1$}
	\STATE 	\textcolor{blue}{Inner value update}:   Set $\tilde y_{t,0}^k = 0$.
	\FOR{ $i = 0,1,\cdots, t-1$}
	\FOR{$k = 1,\cdots, K$}
		\STATE \textcolor{blue}{Local sampling}: Query $\cS\cO$ at $x_t^k$ to obtain $  g^k(x_t^k;\xi_{t,i}^k)  $.
	\STATE \textcolor{blue}{Estimate:} $ \tilde y_{t,i+1}^k = (1-\eta_{t,i})\sum_{j \in \cN_k}w_{k,j}\tilde y_{t,i}^{j}  +
 \eta_{t,i} g^k(x_t^k;\xi_{t,i}^k) $.
 	\ENDFOR
 	\ENDFOR 
  \STATE 
 	Set $y_{t}^k = \tilde y_{t,t}^k$. 
	\STATE \textcolor{blue}{Outer loop update}: $x_{t+1}^k = \sum_{j \in \cN_k}w_{k,j}x_t^{j} - \alpha_t \nabla g(x_t^k;\xi_t^k) \nabla  f(y_t^k;\zeta_t^k)$.
 \ENDFOR
 \ENSURE $ \bar x_{T} = \frac{1}{K} \sum_{k\in \cK} x_T^k$.
\end{algorithmic}
\end{algorithm}

\begin{figure}[t]
\begin{minipage}{0.32\textwidth}
    \centering
   \includegraphics[width=1\linewidth]{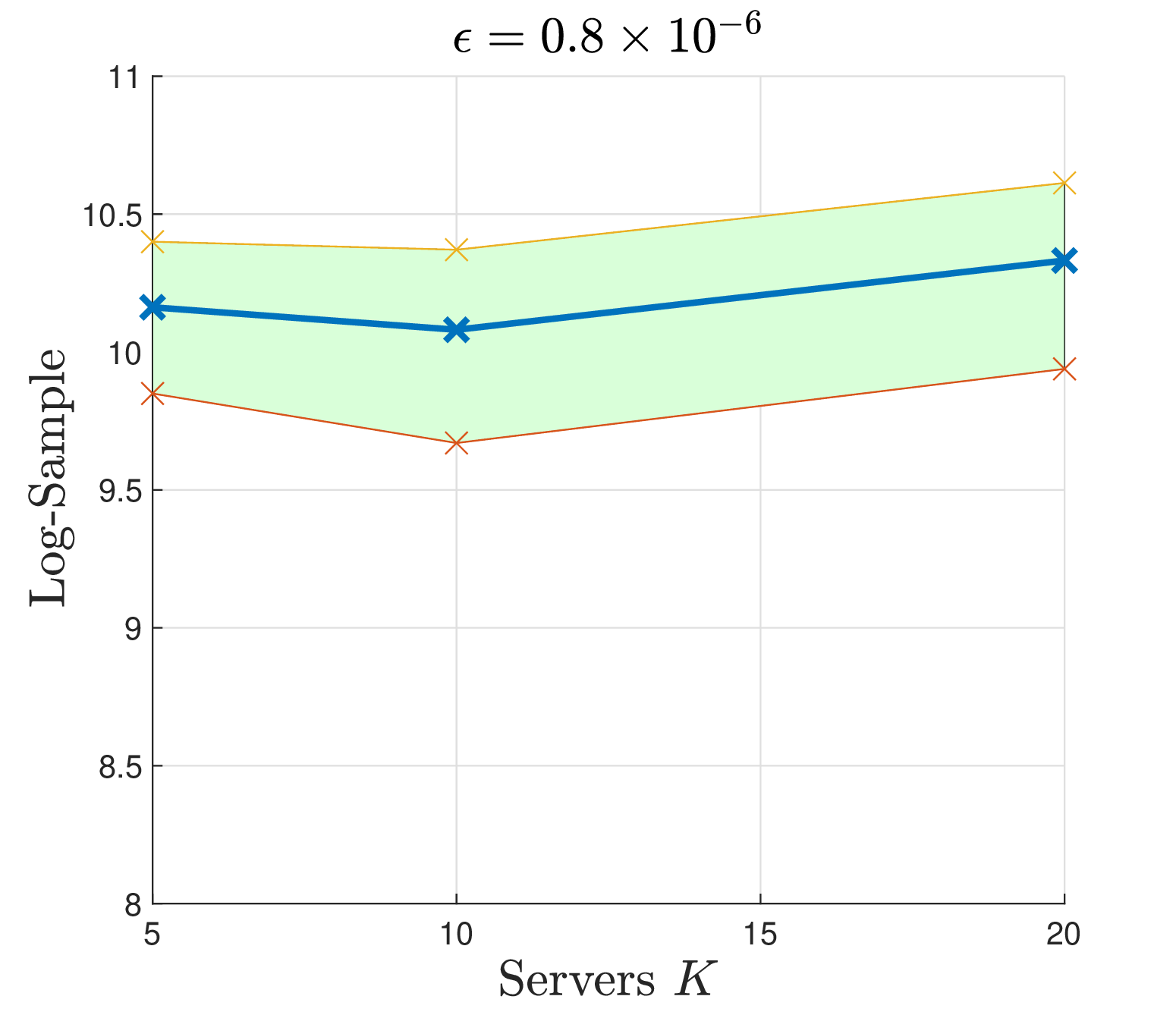}
        (a)
        \end{minipage}  
        \begin{minipage}{0.32\textwidth}
        \centering
   \includegraphics[width=1\linewidth]{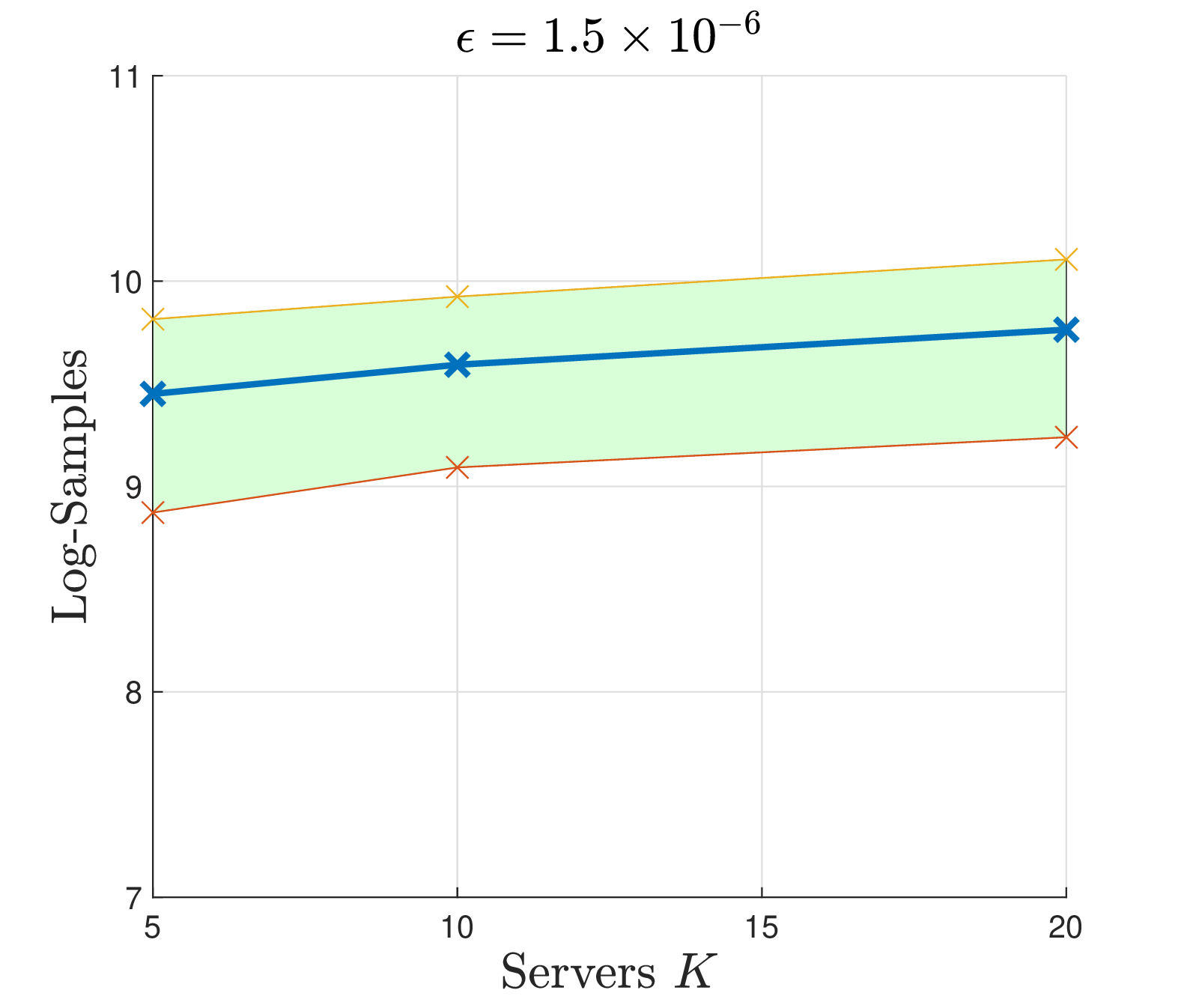}
   (b)
   \end{minipage}
 \begin{minipage}{0.32\textwidth}
        \centering
   \includegraphics[width=1\linewidth]{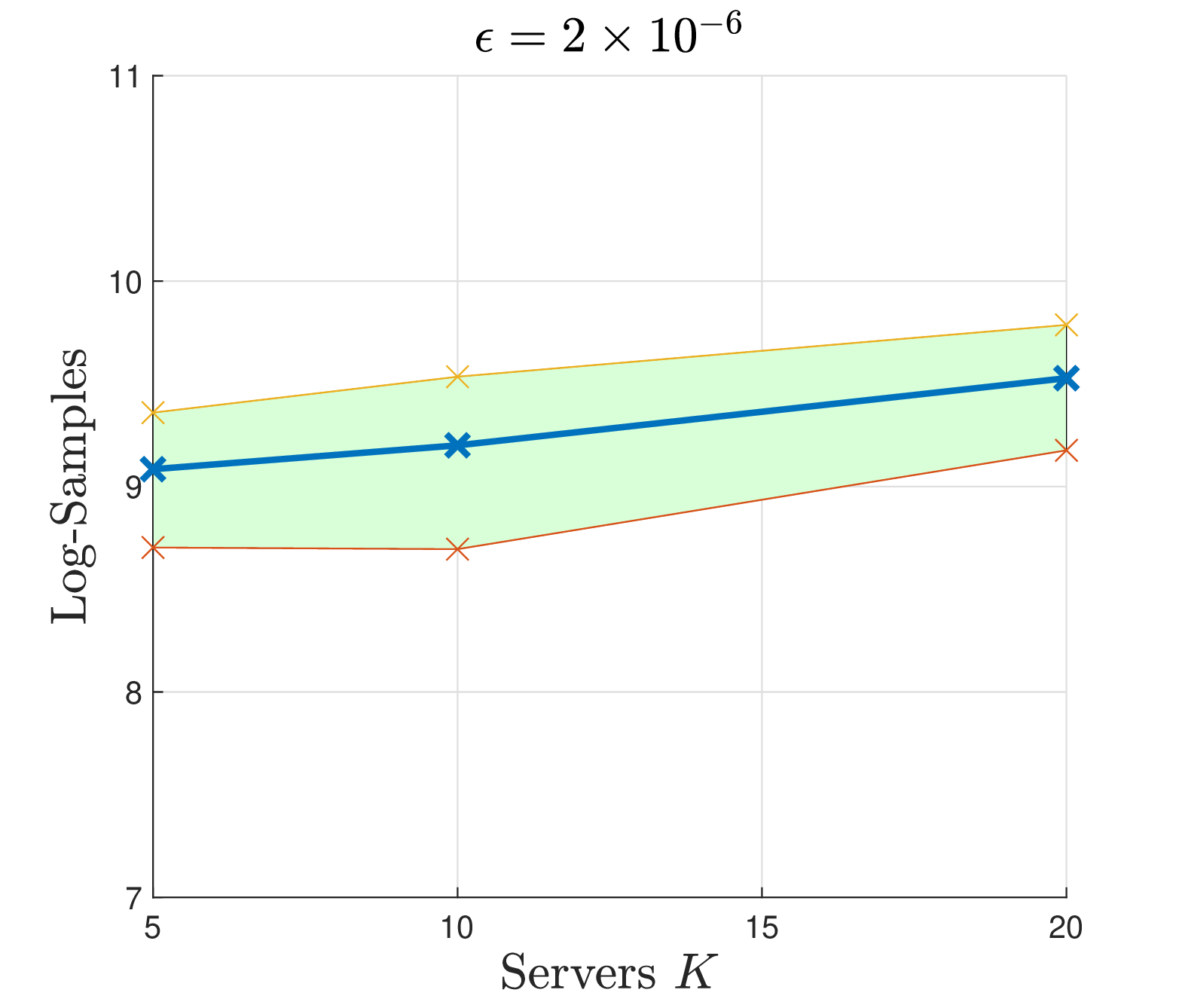}
   (c)
   \end{minipage}
    \caption{75\% confidence region of log- total samples for achieving $\|\bar x_t - x^* \|^2 \leq \epsilon$ with varying network sizes $K = 5,10,20$ for $\epsilon = 0.8\times 10^{-6}, 1.5\times 10^{-6}, 2 \times 10^{-6}$. 
    }
    \label{fig:MDP2}
\end{figure}

We provide the details of the baseline algorithm DSGD in Algorithm~\ref{alg:DSGD}.

To further study the linear speedup effect under various accuracy levels, we  compute the total generated samples for finding an $\epsilon$-optimal solution $\|\bar x_t - x^* \|^2 \leq \epsilon $ and plot the 75\% confidence region of log-sample against the number of agents $K = 5,10,20$ for various $\epsilon$'s  in Figure~\ref{fig:MDP2}. Similar as in Figure~\ref{fig:1}, we observe that for all accuracy levels $\epsilon = 0.8\times 10^{-6}, 1.5\times 10^{-6}, 2 \times 10^{-6}$, the required samples for finding an $\epsilon$-optimal solution by $K$ agents are roughly the same,
 further demonstrating the linear speedup effect.



\end{document}